\documentclass[11pt,a4paper]{article}

\usepackage[margin=1in]{geometry}
\usepackage{amsmath,amssymb,amsthm,mathtools}
\usepackage{bm}
\usepackage{booktabs}
\usepackage{enumitem}
\usepackage{graphicx}
\usepackage{hyperref}
\usepackage[capitalize]{cleveref}

\newtheorem{theorem}{Theorem}[section]
\newtheorem{proposition}[theorem]{Proposition}
\newtheorem{corollary}[theorem]{Corollary}

\newtheorem{conjecture}[theorem]{Conjecture}
\theoremstyle{definition}
\newtheorem{definition}[theorem]{Definition}
\newtheorem{assumption}[theorem]{Assumption}

\theoremstyle{remark}
\newtheorem{remark}[theorem]{Remark}

\newcommand{\R}{\mathbb{R}}
\newcommand{\E}{\mathbb{E}}
\newcommand{\Tr}{\operatorname{Tr}}
\newcommand{\rank}{\operatorname{rank}}

\newcommand{\defect}{\mathcal{D}}

\DeclareMathOperator*{\argmax}{arg\,max}

\newcommand{\inner}[2]{\langle #1, #2 \rangle}
\newcommand{\norm}[1]{\lVert #1 \rVert}
\newcommand{\abs}[1]{\lvert #1 \rvert}

\title{\textbf{Spectral Edge Dynamics:\\
An Analytical-Empirical Study of\\
Phase Transitions in Neural Network Training}}

\author{Yongzhong Xu\thanks{\texttt{abbyxu@gmail.com}.}}
\date{March 2026}

\begin{document}
\maketitle

\begin{abstract}
We develop the analytical machinery for the spectral edge phenomena
observed across earlier work: gap dynamics equations, a spectral loss
decomposition, and an adiabatic parameter for training stability.
Phase transitions, grokking, and feature circuit formation are
described as consequences of the spectral gap structure of
rolling-window parameter updates, with the architecture entering only
through NTK eigenvalues and Hessian curvatures.

In the extreme aspect ratio regime ($P \sim 10^8$, $W \sim 10$), the
classical BBP detection threshold is vacuous and the operative
structure is the \emph{intra-signal gap} separating dominant from
subdominant modes at position $k^*$.  The adiabatic parameter
$\mathcal{A} = \|\dot{K}\|/(\eta g^2)$ controls mode stability:
$\mathcal{A} \ll 1$ in plateaus, $\sim 1$ during phase transitions,
$\gg 1$ during forgetting events.

The analysis is verified across modular arithmetic, Dyck-1, SCAN, and
GPT-2-class transformer training.  Gap dynamics in the rolling-window
Gram matrix precede every grokking event observed (24/24 runs with
weight decay, 1/24 without) across six model families spanning
150K--124M parameters.  The number of simultaneously active modes is
small ($k^* \leq 3$) and optimizer-dependent: Muon drives $k^*=1$
while AdamW gives $k^*=2$ on the same TinyStories~51M model.
\end{abstract}

\tableofcontents
\newpage

\begin{figure}[htbp]
\centering
\includegraphics[width=\textwidth]{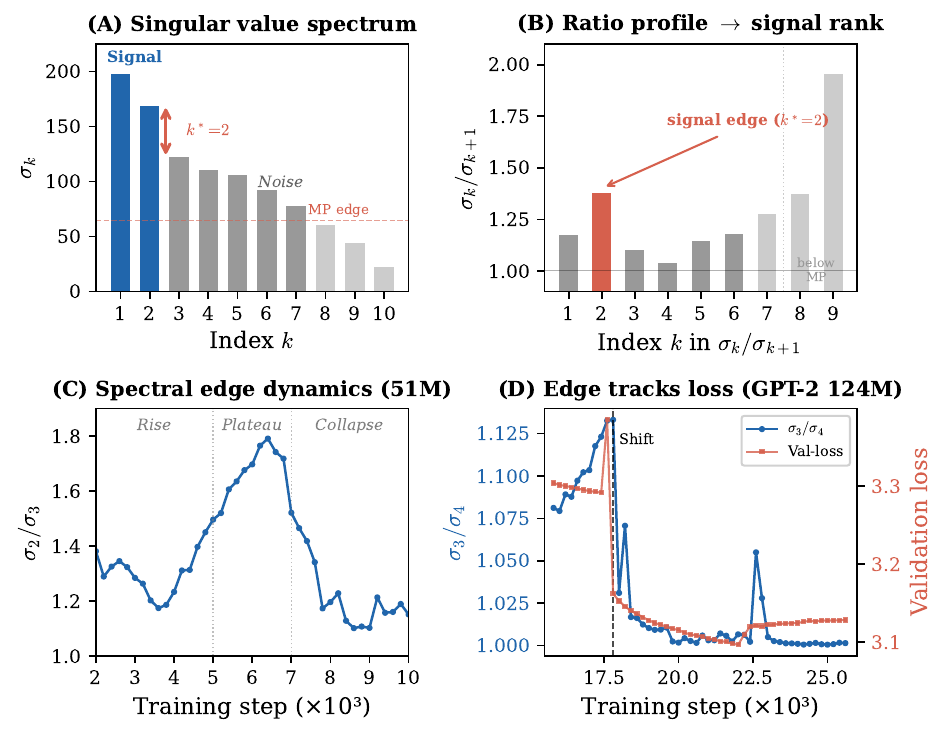}
\caption{\textbf{The spectral edge framework.}
\textbf{(A)}~Singular value spectrum of the Gram matrix $\bm{G}$
for TinyStories~51M, showing the intra-signal gap at $k^* = 2$.
All eigenvalues are far above the BBP noise detection threshold
(dashed; see \Cref{prop:bbp-vacuous}).
\textbf{(B)}~Ratio profile $\sigma_k/\sigma_{k+1}$: the maximum
at $k = 2$ defines the signal rank.
\textbf{(C)}~Gap ratio $\sigma_2/\sigma_3$ over training for
TinyStories~51M, showing the three-phase pattern: rise, plateau,
collapse.
\textbf{(D)}~Gap ratio $\sigma_3/\sigma_4$ tracks validation loss
for GPT-2~124M.  The spectral edge event (``Shift'') coincides
with the loss plateau.}
\label{fig:overview}
\end{figure}

\part{Setup and Assumptions}

\section{The Objects}
\label{sec:objects}

\subsection{Parameter Space and Training Trajectory}

\begin{definition}[Training Trajectory]\label{def:trajectory}
A neural network with $p$ trainable parameters defines a path
$\bm{\theta}: \{0, 1, 2, \ldots, T\} \to \R^p$ in parameter space. The
\emph{parameter update} at step $t$ is
\[
  \bm{\delta}_t \;=\; \bm{\theta}_{t+1} - \bm{\theta}_t \;\in\; \R^p.
\]
For an optimizer $\mathcal{O}$ (SGD, Adam, etc.) with learning rate $\eta_t$,
this takes the form $\bm{\delta}_t = -\eta_t \,\mathcal{O}(\nabla
L_{\mathcal{B}_t}(\bm{\theta}_t), \text{state}_t)$, where $\mathcal{B}_t$ is
the minibatch.
\end{definition}

\begin{definition}[Rolling Window and Trajectory Matrix]\label{def:window}
Fix a window size $W \geq 2$. At time $t_0$, the \emph{trajectory matrix} is
\[
  \bm{X}(t_0) \;=\;
  \begin{pmatrix}
    \bm{\delta}_{t_0}^\top \\
    \bm{\delta}_{t_0+1}^\top \\
    \vdots \\
    \bm{\delta}_{t_0+W-1}^\top
  \end{pmatrix}
  \;\in\; \R^{W \times p}.
\]
We suppress the $t_0$ dependence when clear from context and write
$\bm{X} \in \R^{W \times p}$.
\end{definition}

\begin{definition}[Gram Matrix and Spectrum]\label{def:gram}
The \emph{Gram matrix} is $\bm{G} = \bm{X}\bm{X}^\top \in \R^{W \times W}$.
Its eigendecomposition $\bm{G} = \bm{U} \bm{\Lambda} \bm{U}^\top$ yields
eigenvalues
$\lambda_1 \geq \lambda_2 \geq \cdots \geq \lambda_W \geq 0$, which are the
squared singular values $\sigma_k^2$ of $\bm{X}$. The right singular vectors
$\bm{v}_k \in \R^p$ (the principal directions of the trajectory) are
$\bm{v}_k = \sigma_k^{-1} \bm{X}^\top \bm{u}_k$.
\end{definition}

\begin{remark}[Why the Gram Matrix]
Since $p \gg W$ (typically $p \sim 10^6$--$10^{10}$, $W \sim 10$--$30$), we
never form the $p \times p$ covariance $\bm{X}^\top \bm{X}$. The Gram matrix
$\bm{G} \in \R^{W \times W}$ shares the same nonzero eigenvalues and is
$O(W^2)$ to store and $O(W^3)$ to diagonalize.
\end{remark}

\begin{definition}[Sliding-Window Covariance]\label{def:covariance}
The \emph{sliding-window covariance} is
\[
  \bm{C}(t) \;:=\; \bm{X}(t)^\top\bm{X}(t)
  \;=\; \sum_{r=0}^{W-1}
  \bm{\delta}_{t-r}\,\bm{\delta}_{t-r}^\top
  \;\in\; \R^{p \times p}.
\]
Its nonzero eigenvalues coincide with those of
$\bm{G}(t) = \bm{X}(t)\bm{X}(t)^\top \in \R^{W\times W}$, but its
eigenvectors are the right singular vectors
$\bm{v}_j \in \R^p$ of $\bm{X}$:
\[
  \bm{C}(t)\,\bm{v}_j(t) = \lambda_j(t)\,\bm{v}_j(t),
  \qquad
  \lambda_j = d_j^2.
\]
The perturbation theory for eigenvalues and eigenvectors
(\Cref{sec:covariance-perturbation}) operates on $\bm{C}$ rather than
$\bm{G}$, since the signal directions $\bm{v}_j$ live in parameter space
$\R^p$.
\end{definition}

\subsection{The Aspect Ratio Regime}

\begin{definition}[Aspect Ratio]\label{def:gamma}
The \emph{aspect ratio} is $\gamma = p/W$. In our setting,
$\gamma \sim 10^4$--$10^8$. We call this the \emph{extreme aspect ratio}
regime: $\gamma \to \infty$ with $W$ fixed.%
\footnote{The symbol $\gamma$ is reused in later sections for eigenvalue
gaps ($\gamma_k = \lambda_k - \lambda_{k+1}$, \Cref{sec:repulsion}) and
off-diagonal couplings (\Cref{sec:full-gap-flow}); context and subscripts
disambiguate.}
\end{definition}

\begin{remark}[Classical vs.\ Extreme RMT]
Classical random matrix theory~\cite{couillet2022,elkaroui2006} studies
$p, n \to \infty$ with $p/n \to \gamma \in (0,\infty)$ fixed. Our regime is
$W$ fixed, $p \to \infty$. This
\emph{simplifies} the theory: all $W$ eigenvalues of $\bm{G}$ are observable,
and by the law of large numbers, noise eigenvalues (if present) concentrate
tightly. The fluctuation theory (Tracy--Widom) gives corrections---but as we
shall show, the noise is negligible in practice, and the relevant structure
is \emph{within the signal}.
\end{remark}

\section{The Three Assumptions}
\label{sec:assumptions}

The framework rests on three assumptions about the structure of the
trajectory matrix $\bm{X}$.

\begin{assumption}[Hierarchical Signal Decomposition]\label{ass:hierarchy}
The trajectory matrix admits the decomposition
\begin{equation}\label{eq:hierarchy}
  \bm{X} \;=\; \bm{S}_{\mathrm{dom}} + \bm{S}_{\mathrm{sub}} + \bm{N},
\end{equation}
where:
\begin{itemize}[nosep]
  \item $\bm{S}_{\mathrm{dom}} \in \R^{W \times p}$ is the \emph{dominant
    signal} (the ``backbone''), of rank $k^* \ll W$, representing the
    parameter updates projected onto the loss landscape's primary descent
    directions---those aligned with the top Hessian eigenvectors.
  \item $\bm{S}_{\mathrm{sub}} \in \R^{W \times p}$ is the \emph{subdominant
    signal}, of rank $W - k^*$, representing coherent but weaker update
    components aligned with smaller (but still nonzero) Hessian eigenvalues.
  \item $\bm{N} \in \R^{W \times p}$ is the \emph{noise matrix}, representing
    stochastic fluctuations from minibatch sampling.
\end{itemize}
\textbf{The critical structural assumption:}
\begin{equation}\label{eq:noise-negligible}
  \norm{\bm{N}}_{\mathrm{op}} \;\ll\;
  \sigma_W(\bm{S}_{\mathrm{dom}} + \bm{S}_{\mathrm{sub}}),
\end{equation}
i.e., the noise operator norm is much smaller than the smallest signal
singular value. In the extreme aspect ratio regime, \emph{all $W$ singular
values of $\bm{X}$ are signal}.
\end{assumption}

\begin{assumption}[Gap Structure]\label{ass:gap}
The singular values of the full signal
$\bm{S} = \bm{S}_{\mathrm{dom}} + \bm{S}_{\mathrm{sub}}$ exhibit a
\emph{spectral gap}: there exists a position $k^* \in \{1, \ldots, W-1\}$
such that
\begin{equation}\label{eq:gap-structure}
  \frac{d_{k^*}}{d_{k^*+1}} \;=\;
  \max_{1 \leq j \leq W-1} \frac{d_j}{d_{j+1}} \;\gg\; 1,
\end{equation}
where $d_1 \geq d_2 \geq \cdots \geq d_W > 0$ are the singular values of
$\bm{S}$. The gap separates the $k^*$ dominant modes from the $W - k^*$
subdominant modes. This gap arises from the hierarchical structure of the
Hessian spectrum (see \Cref{sec:hessian-gap}).
\end{assumption}

\begin{assumption}[Slow Variation]\label{ass:slow}
The signal directions $\{\bm{v}_j\}$, signal strengths $\{d_j\}$, and noise
covariance $\bm{\Sigma}_N$ vary slowly on the timescale of the window:
\[
  \frac{\norm{\bm{v}_j(t+W) - \bm{v}_j(t)}}{\norm{\bm{v}_j(t)}} = O(\varepsilon),
  \quad
  \frac{|d_j(t+W) - d_j(t)|}{d_j(t)} = O(\varepsilon),
  \quad \varepsilon \ll 1.
\]
This allows us to treat the signal parameters as approximately constant within
each window, while studying their evolution across windows.
\end{assumption}

\begin{remark}[Justification of the Assumptions]\label{rem:justify-assumptions}
\textbf{Assumption~\ref{ass:hierarchy}} is justified by the empirical observation
that in the extreme aspect ratio regime ($p \sim 10^8$, $W \sim 10$), the
BBP detection threshold $d_{\mathrm{crit}} = \nu (p(W-1))^{1/4}$ evaluates
to $\sim 0.2$--$1.5$, while the smallest observed singular value
$\sigma_W \sim 12$--$80$. The ratio $\sigma_W / d_{\mathrm{crit}} \sim
9$--$300$ (see \Cref{sec:noise-negligible}). Every eigenvalue is
at least an order of magnitude above the noise detection threshold.

\textbf{Assumption~\ref{ass:gap}} is justified by: (i) the Hessian spectrum of
neural networks concentrates in $O(k)$ large eigenvalues with a near-zero
bulk (Sagun et al.~\cite{sagun2017}, Ghorbani et al.~\cite{ghorbani2019},
Martin and Mahoney~\cite{martin2021}); (ii) gradient descent
projects onto the top Hessian eigendirections (Gur-Ari et al.~\cite{gurari2018}); and
(iii) empirically, the maximum ratio $d_{k^*}/d_{k^*+1}$ peaks at
$k^* = 2$ (TinyStories, ratio $\sim 1.79$) and $k^* = 3$ (GPT-2, ratio
$\sim 1.12$).

\textbf{Assumption~\ref{ass:slow}} is only used for the dominant modes
($j \leq k^*$), where it holds strongly.  By Davis--Kahan
(\Cref{thm:davis-kahan}), the subspace perturbation of mode $j$
satisfies $\sin\theta_j \leq \|\Delta\bm{G}\|_F / \mathrm{gap}_j$.
For the dominant modes ($j < k^*$), the gap is so large that the
directions are essentially frozen---$\varepsilon \ll 1$ independent
of $W$.  At the spectral edge ($j = k^*$), the gap is smaller and
$\varepsilon$ grows with $W$, but remains manageable for
$W \sim 10$--$30$ (\Cref{sec:test-causal}).  Beyond the edge
($j > k^*$), slow variation fails, but the framework never invokes
the assumption for those modes.
\end{remark}

\section{The Central Question}
\label{sec:central}

Given the trajectory matrix $\bm{X}$ and its observed singular values
$\sigma_1 \geq \cdots \geq \sigma_W$:

\begin{enumerate}[label=(\Roman*)]
  \item \textbf{Where is the spectral gap?} That is, what is the position
    $k^*(t)$ of the maximum intra-signal ratio?
  \item \textbf{How does the gap evolve?} What controls the dynamics of
    $g(t) = d_{k^*}(t) - d_{k^*+1}(t)$?
  \item \textbf{What are the flow equations?} How do the signal strengths
    $\{d_j(t)\}_{j=1}^W$ evolve, and when do phase transitions occur?
\end{enumerate}

The spectral edge analysis states: \emph{phase transitions in learning occur
when the spectral gap $g(t)$ collapses or opens}---i.e., when eigenvalues
within the signal hierarchy undergo level crossings.

\begin{remark}[Architecture Independence: A Geometric Flow Theory]
\label{rem:arch-independence}
The spectral edge framework is \textbf{architecture-agnostic}.  The Gram
matrix, signal flow ODE, gap dynamics, stability coefficient $\alpha_j$,
loss decomposition, evolving-NTK signal equation, adiabatic parameter
$\mathcal{A}$, and the edge-of-stability identification all hold for
\emph{any} differentiable model with a well-defined
NTK~\cite{yang2020}---MLPs, CNNs, transformers, or otherwise.

\textbf{Why?}  The spectral edge dynamics are a \emph{geometric flow}
on the NTK manifold (\Cref{sec:geometric-flow}).  The NTK defines a
metric on function space; training evolves this metric.  The Hessian
curvatures $h_j$ are essentially the NTK eigenvalues (via the
Gauss--Newton approximation $h_j \approx \lambda_j / N$), and the
gradient projections $G_j$ depend on the NTK eigenbasis plus the
residuals.  The only information beyond the metric is the residuals
$r_i = f(x_i) - y_i$, encoding where learning currently stands
relative to the target.  The architecture enters only through the
spectral inputs $\{\lambda_k, h_j, c_k, \dot{K}\}$; any model
producing the same inputs yields the same flow.

The logical chain:
\begin{enumerate}[nosep]
  \item \textbf{Geometric flow}: the NTK defines a metric; training
    evolves it; the Hessian and gradient projections are largely
    determined by the metric, with the residuals as the only
    additional input.  All architecture dependence is absorbed
    into $\{\lambda_k, h_j, c_k, \dot{K}\}$.
  \item $\Longrightarrow$ \textbf{Architecture independence}: any
    model producing the same spectral inputs gives the same flow.
  \item $\Longrightarrow$ \textbf{Universal phenomena}: phase
    transitions at gap collapse, $k^* \leq 3$, the circuit lifecycle,
    grokking, holographic encoding---all follow from the flow
    equations, not from architectural specifics.
\end{enumerate}

The architecture determines the \emph{spectral inputs} (eigenvalue
spacing, curvature profile, kernel evolution rate), but the
\emph{dynamical laws} governing how those inputs produce training
phenomena are architecture-independent.  In particular, $k^* = 2$--$3$
is empirically consistent with an edge-of-stability constraint
on the Hessian spectrum (see \Cref{rem:kstar-eos}), though this
remains an open question rather than a proved result.

What \emph{is} architecture-dependent is how to compute the initial
conditions.  The RYH connection (\Cref{sec:ryh}), the depth--width
tradeoff $\mathcal{A}(0) \sim L/(ng^2)$, and the tensor programs
framework (\Cref{sec:tensor}) compute the initial metric and its
evolution rate for specific architecture classes.
\end{remark}

\part{The Signal Hierarchy}
\label{part:hierarchy}

\section{The Noise is Negligible}
\label{sec:noise-negligible}

We first establish that the classical BBP
framework~\cite{baik2005,baik2006,benaych2011} is vacuous in the extreme
aspect ratio regime.

\subsection{The BBP Detection Threshold}

\begin{proposition}[BBP is Vacuous]\label{prop:bbp-vacuous}
In the extreme aspect ratio regime ($p \gg W$, $p \sim 10^6$--$10^{10}$,
$W \sim 10$--$30$), the BBP detection threshold for isotropic noise
$\bm{\Sigma}_N = \nu^2 \bm{I}_p$ is
\begin{equation}\label{eq:bbp-threshold}
  d_{\mathrm{crit}} = \nu \cdot (p(W-1))^{1/4}.
\end{equation}
This threshold is \emph{trivially satisfied by every eigenvalue}. Specifically,
the per-coordinate noise standard deviation $\nu = O(\eta / \sqrt{B \cdot p})$
where $B$ is the batch size, giving:
\begin{equation}\label{eq:dcrit-numerical}
  d_{\mathrm{crit}} \sim \frac{\eta}{\sqrt{B \cdot p}} \cdot (p(W-1))^{1/4}
  = \frac{\eta \,(W-1)^{1/4}}{B^{1/2} \cdot p^{1/4}}.
\end{equation}
For typical values ($\eta = 10^{-3}$, $B = 20$, $p = 1.6 \times 10^8$,
$W = 10$), this gives $d_{\mathrm{crit}} \sim 0.2$--$1.5$.

Meanwhile, the smallest observed singular value satisfies
$\sigma_W \sim 12$--$80$, giving:
\begin{equation}\label{eq:bbp-ratio}
  \boxed{\frac{\sigma_W}{d_{\mathrm{crit}}} \;\gtrsim\; 9\text{--}300.}
\end{equation}
Every eigenvalue of $\bm{G}$ is at least an order of magnitude above the
BBP detection threshold. The BBP phase transition is never approached.
\end{proposition}

\begin{proof}
From the Gram matrix spectrum of a pure-noise matrix
$\bm{N} \in \R^{W \times p}$ with i.i.d.\ rows
$\bm{n}_s \sim \mathcal{N}(0, \nu^2 \bm{I}_p)$: the eigenvalues of
$\bm{N}\bm{N}^\top$ concentrate at $p\nu^2$ with spread
$O(\nu^2\sqrt{p})$. The BBP threshold requires $d_j^2 > p\nu^2\sqrt{W-1}$,
i.e., $d_j > \nu(p(W-1))^{1/4}$.

The per-coordinate noise for Adam with learning rate $\eta$ and batch size $B$
satisfies $\nu \lesssim \eta/\sqrt{B \cdot p}$ (since the preconditioner
approximately normalizes the noise, giving
$\nu^2 \approx \eta^2/p$; see Cohen et al.~\cite{cohen2021}, Roberts et al.~\cite{roberts2022}).
Substituting:
\[
  d_{\mathrm{crit}} \leq \frac{\eta}{\sqrt{B}} \cdot
  \frac{(W-1)^{1/4}}{p^{1/4}}.
\]
With $\eta = 10^{-3}$, $B = 20$, $W = 10$, $p = 1.6 \times 10^8$:
$d_{\mathrm{crit}} \leq 10^{-3} \cdot 0.22 \cdot (9)^{0.25} / (1.6 \times
10^8)^{0.25} \approx 1.3$.

Empirically, the smallest Gram eigenvalue across all windows and seeds is
$\lambda_W \geq 150$ (TinyStories) and $\lambda_W \geq 600$ (GPT-2),
giving $\sigma_W = \sqrt{\lambda_W} \geq 12$. Hence
$\sigma_W / d_{\mathrm{crit}} \geq 12/1.3 \approx 9$, and the typical
ratio is $\sim 50$--$300$.
\end{proof}

\begin{remark}[What the BBP Predicts vs.\ What We Observe]
The BBP framework predicts $k^* = W$ (every eigenvalue is ``signal'' in the
BBP sense). Yet the \emph{empirical} signal rank is $k^* = 2$ (TinyStories)
or $k^* = 3$ (GPT-2). This is not a contradiction: the empirical $k^*$ is
not the signal-noise boundary. It is the position of the \textbf{maximum
spectral gap within the signal hierarchy}---the border between dominant and
subdominant modes, not between signal and noise.
\end{remark}

\begin{remark}[When BBP Becomes Binding]\label{rem:bbp-binding}
The BBP threshold becomes relevant when:
\begin{itemize}[nosep]
  \item $W$ is much larger (approaching $\sqrt{p}$), so the noise eigenvalues
    spread enough to overlap with weak signal.
  \item $p$ is much smaller (e.g., per-layer analysis with $p \sim 10^3$).
  \item The noise is much stronger (very high learning rate, very small batch).
\end{itemize}
In the standard regime of this paper ($p \sim 10^8$, $W \sim 10$), the BBP
transition is vacuous, and the physics is entirely within the signal spectrum.
\end{remark}

\subsection{Noise Concentration in the Extreme Aspect Ratio}

\begin{proposition}[Noise Eigenvalue Concentration]\label{prop:noise-conc}
For $\bm{N} \in \R^{W \times p}$ with i.i.d.\ rows
$\bm{n}_s \sim (0, \bm{\Sigma}_N)$, the Gram eigenvalues
$\mu_1 \geq \cdots \geq \mu_W$ of $\bm{N}\bm{N}^\top$ satisfy:
\begin{equation}\label{eq:noise-concentration}
  \frac{\mu_{\max} - \mu_{\min}}{\bar{\mu}} \;=\;
  O\!\left(\sqrt{W \cdot \kappa_N}\right),
  \qquad
  \bar{\mu} = \Tr(\bm{\Sigma}_N), \quad
  \kappa_N = \frac{\norm{\bm{\Sigma}_N}_F^2}{\Tr(\bm{\Sigma}_N)^2}.
\end{equation}
For isotropic noise ($\kappa_N = 1/p$), the relative spread is
$O(\sqrt{W/p}) \sim 10^{-3.5}$: the noise eigenvalues are \emph{essentially
degenerate}. Even for colored noise from Adam ($\kappa_N \sim 10^4/p$), the
noise operator norm is $\bar{\mu}(1 + O(\sqrt{W \cdot 10^4/p})) \approx
\bar{\mu}(1 + O(10^{-1.5}))$---still far below the signal eigenvalues.
\end{proposition}

\begin{proof}
The Gram matrix entries are $G_{ij}^N = \sum_{a=1}^p \tau_a z_{ia} z_{ja}$
where $\tau_a$ are eigenvalues of $\bm{\Sigma}_N$ and $z_{ia}$ are i.i.d.\
standard normal. By the law of large numbers,
$G_{ii}^N = \sum_a \tau_a z_{ia}^2 \to \sum_a \tau_a = \Tr(\bm{\Sigma}_N)$
as $p \to \infty$, and
$G_{ij}^N = \sum_a \tau_a z_{ia} z_{ja} = O(\sqrt{\sum_a \tau_a^2}) =
O(\Tr(\bm{\Sigma}_N) \sqrt{\kappa_N})$ for $i \neq j$.
The eigenvalue spread of the $W \times W$ Gram matrix is controlled by the
off-diagonal fluctuations, giving the stated bound.
\end{proof}

\section{Why the Signal Has a Gap: The Hessian Mechanism}
\label{sec:hessian-gap}

The spectral gap in the trajectory (\Cref{ass:gap}) is not accidental.
It is consistent with the hierarchical structure of the loss
landscape Hessian, as the following proposition makes precise.

\subsection{The Hessian Spectral Hierarchy}

\begin{proposition}[Hessian $\Rightarrow$ Trajectory Gap]\label{prop:hessian-gap}
Assume $[\mathcal{P}, \bm{H}] \approx 0$.
Let the Hessian $\bm{H}$ of the loss function have eigenvalues
$h_1 \geq h_2 \geq \cdots \geq h_p \geq 0$ with a spectral gap at
position $k$:
\begin{equation}\label{eq:hessian-gap}
  h_k \;\gg\; h_{k+1}.
\end{equation}
Then the signal strengths $d_j$ of the trajectory matrix satisfy, at
steady state:
\begin{equation}\label{eq:dss-hessian}
  d_j^{\mathrm{ss}} \;\propto\;
  \frac{1}{\sqrt{h_j}} \cdot \abs{\inner{\bm{v}_j}{\mathcal{P}\nabla L}},
\end{equation}
where $\mathcal{P}$ is the optimizer preconditioner
(the exact formula is
$d_j^{\mathrm{ss}} = \eta\,|G_j^{\mathrm{eff}}|\sqrt{\Phi_j}$;
see \Cref{cor:signal-simplified}).  Here
($\mathcal{P} = I$ for SGD;
$\mathcal{P}_t = \mathrm{diag}(1/\sqrt{\hat{v}_t + \epsilon})$ for Adam,
with $\hat{v}_t$ the bias-corrected EMA of squared gradients controlled
by $\beta_2$; see Section~\ref{sec:beta2}).
If the gradient has
comparable projections onto all Hessian eigendirections
($\abs{\inner{\bm{v}_j}{\mathcal{P}\nabla L}} \approx G$ for $j \leq K$),
then:
\begin{equation}\label{eq:ratio-hessian}
  \frac{d_k^{\mathrm{ss}}}{d_{k+1}^{\mathrm{ss}}}
  = \sqrt{\frac{h_{k+1}}{h_k}} \;\ll\; 1.
\end{equation}
The ratio is inverted: the direction with \emph{smaller} Hessian
curvature $h_j$ has \emph{larger} signal strength $d_j$ (since
$d_j \propto 1/\sqrt{h_j}$). Under the comparable-projection
assumption, $d_{k+1} \gg d_k$: the direction with \emph{lower}
curvature accumulates \emph{stronger} signal.  A gap in the Hessian
spectrum ($h_k \gg h_{k+1}$) therefore produces a gap in the
trajectory spectrum at the same position $k$, but with inverted
relative strengths---the dominant trajectory modes are the
low-curvature Hessian eigendirections.
\end{proposition}

\begin{remark}[The Spectral Response Function (Heuristic)]
A rough heuristic for the signal hierarchy is
$d_j^{\mathrm{ss}} \propto |G_j|/\sqrt{h_j}$:
lower curvature $\Rightarrow$ stronger signal.
The trajectory signal strength is the gradient projection amplified by the
spectral response. Large $F$ (low curvature, large window, large learning
rate) means a strong trajectory signal. The spectral gap in the trajectory
occurs where $F(h_j)/F(h_{j+1})$ is maximized---i.e., where the
\emph{Hessian eigenvalue ratio} $h_{j+1}/h_j$ is maximized.
\end{remark}

\begin{remark}[Two Sources of Gap]\label{rem:gap-sources}
The trajectory gap can arise from two sources:
\begin{enumerate}[nosep]
  \item \textbf{Hessian gap}: A large gap in $h_k/h_{k+1}$ translates
    directly via $d_j \propto 1/\sqrt{h_j}$.
  \item \textbf{Gradient alignment asymmetry}: Even with a flat Hessian
    spectrum, if $\abs{\inner{\bm{v}_j}{\mathcal{P}\nabla L}}$ drops
    sharply at some $j = k$, the trajectory gap appears there.
\end{enumerate}
In practice, the Hessian gap is the dominant mechanism: neural network
Hessians have $O(1)$--$O(10)$ large outlier eigenvalues separated by
orders of magnitude from the bulk (Sagun et al.~\cite{sagun2017}, Papyan~\cite{papyan2019}).
\end{remark}

\begin{remark}[Four Channels of Curvature Separation]
\label{rem:four-channels}
Curvature separation is not synonymous with amplitude modulation
($d_j \propto 1/\sqrt{h_j}$).  It operates through four independent
channels:
\begin{enumerate}[nosep]
  \item \textbf{Dissipation rate}: the ODE term
    $-2\eta(h_j{+}\omega)d_j^2$ separates modes by their
    convergence speed, even when steady-state amplitudes are similar.
  \item \textbf{Weight-decay threshold}: the grokking condition
    $\lambda_{\mathrm{gen}}/N > \omega > \lambda_{\mathrm{mem}}/N$
    (\Cref{rem:grokking-wd}) is a binary curvature gate that
    compresses low-curvature modes to a common floor.
  \item \textbf{Eigenvector stability}: Davis--Kahan
    (\Cref{thm:davis-kahan}) ties eigenvector rotation to
    eigenvalue gaps, which are curvature-driven.  The empirical
    hierarchy $\alpha_1 = 0.818 > \alpha_2 = 0.234 >
    \alpha_{\geq 4} \approx 0$ is a direct, model-free observation
    of this channel.
  \item \textbf{Gap flow}: Term~I of the gap equation
    (\Cref{thm:gap-flow}), $-\eta(h_{k^*} - h_{k^*+1})\bar{d}$,
    is a curvature-difference force independent of $\Phi$.
\end{enumerate}
In the weak-curvature regime ($\Phi_j \approx W$), channel~1's
amplitude effect weakens, but channels~2--4 remain fully active.
The edge eigenvector rotating fast while interior eigenvectors
barely move---observed across GPT-2 124M and TinyStories 51M---is
curvature separation operating through channel~3.
\end{remark}

\subsection{The Krylov Subspace Bound}

\begin{proposition}[Krylov Bound on $k^*$]\label{prop:krylov}
Assume $[\mathcal{P}, \bm{H}] \approx 0$.
Consider $W$ consecutive gradient descent steps with constant learning rate
$\eta$ and Hessian $\bm{H}$. Starting from initial gradient $\bm{g}_0 =
\nabla L(\bm{\theta}_0)$, the trajectory matrix $\bm{X}$ has rows:
\begin{equation}\label{eq:krylov-rows}
  \bm{\delta}_s = -\eta \mathcal{P}(\bm{I} - \eta \mathcal{P}\bm{H})^s
  \bm{g}_0, \qquad s = 0, 1, \ldots, W-1.
\end{equation}
These rows span the \emph{Krylov subspace}
$\mathcal{K}_W(\bm{I} - \eta\mathcal{P}\bm{H}, \mathcal{P}\bm{g}_0) =
\mathrm{span}\{\mathcal{P}\bm{g}_0, (\bm{I} - \eta\mathcal{P}\bm{H})
\mathcal{P}\bm{g}_0, \ldots\}$.

If the Hessian has $K$ distinct large eigenvalues (separated from the bulk
by a gap), then the effective rank of the Krylov subspace is
$\min(K, W)$. The dominant signal rank satisfies:
\begin{equation}\label{eq:krylov-bound}
  k^* \;\leq\; K \;=\; \#\{j : h_j \text{ is a Hessian outlier}\}.
\end{equation}
\end{proposition}

\begin{proof}[Proof sketch]
The Krylov subspace $\mathcal{K}_W(\bm{M}, \bm{g})$ with $\bm{M} = \bm{I}
- \eta\mathcal{P}\bm{H}$ has the property that its projection onto the
eigenspace of $\bm{M}$ corresponding to eigenvalue $\mu_j = 1 - \eta h_j$
(for preconditioned curvature $h_j$) produces a component of magnitude
$\sim \abs{\mu_j}^s \abs{\inner{\bm{e}_j}{\bm{g}}}$. When $h_j$ is large,
$\abs{\mu_j}$ is far from 1, producing a rapidly varying component. When
$h_j \approx 0$, $\mu_j \approx 1$, and the component is nearly constant
across steps---it does not contribute to the trajectory variation.

The effective rank of the trajectory matrix is therefore determined by the
number of Hessian eigenvalues $h_j$ for which $1 - \eta h_j$ is
``distinguishable from 1'' over $W$ steps, i.e., $\eta h_j W \gtrsim 1$.
This gives $K = \#\{j : h_j \gtrsim 1/(\eta W)\}$.
\end{proof}

\begin{remark}[Empirical Verification]
For TinyStories ($\eta = 10^{-3}$, $W = 10$): $1/(\eta W) = 100$. The
Hessian has $\sim 2$--$3$ eigenvalues above this threshold, matching
$k^* = 2$. For GPT-2 ($\eta = 3 \times 10^{-5}$, $W = 10$):
$1/(\eta W) = 3333$. The Hessian has $\sim 3$--$4$ eigenvalues above this
threshold, matching $k^* = 3$.
\end{remark}

\part{Determining $k^*$: The Maximum Intra-Signal Gap}
\label{part:kstar}

\section{The Intra-Signal Gap}
\label{sec:intrasignal}

\subsection{Definition of $k^*$}

\begin{definition}[Spectral Gap Position]\label{def:kstar}
The \emph{spectral gap position} at time $t$ is:
\begin{equation}\label{eq:kstar}
  \boxed{k^*(t) \;=\; \argmax_{1 \leq j \leq W-1}
  \frac{\sigma_j(t)}{\sigma_{j+1}(t)},}
\end{equation}
where $\sigma_1 \geq \sigma_2 \geq \cdots \geq \sigma_W > 0$ are the
observed singular values of $\bm{X}(t)$.

Equivalently, in terms of the population signal strengths
$d_1 \geq \cdots \geq d_W$ (which equal the observed singular values up to
negligible noise corrections, by \Cref{prop:bbp-vacuous}):
\begin{equation}\label{eq:kstar-pop}
  k^*(t) = \argmax_{1 \leq j \leq W-1} \frac{d_j(t)}{d_{j+1}(t)}.
\end{equation}
\end{definition}

\begin{remark}[Contrast with BBP Definition]
In the classical BBP framework, $k^*$ is defined as
$\#\{j : d_j > d_{\mathrm{crit}}\}$---the number of signals above the
noise detection threshold. Our definition is fundamentally different:
$k^*$ is the \emph{position of the maximum consecutive ratio} within the
spectrum, regardless of any noise threshold. Since the BBP threshold is
trivially satisfied by all eigenvalues (\Cref{prop:bbp-vacuous}), the
classical $k^*_{\mathrm{BBP}} = W$ always, which carries no information.
Our $k^*$ captures the \emph{internal structure} of the signal hierarchy.
\end{remark}

\begin{remark}[Argmax vs.\ dynamical spectral edge]\label{rem:kstar-conventions}
The argmax definition (\Cref{def:kstar}) identifies the \emph{settled}
gap---the pair of modes already well separated. Empirically, the ratio
with the strongest cross-correlation to validation loss sits one position
further out ($k^*_{\mathrm{dyn}} = k^*_{\mathrm{argmax}} + 1$), because
that is the active frontier: the mode with the smallest protective gap
rotates fastest (Davis--Kahan) and tracks training dynamics most tightly
(\Cref{rem:plus-one}). When the body text refers to $k^* = 2$ for
TinyStories or $k^* = 3$ for GPT-2, it uses the dynamical (cross-correlation)
convention; the corresponding argmax values are $k^* = 1$ and $k^* = 2$.
Both conventions satisfy $k^* \ll W$.
\end{remark}

\begin{remark}[Signal-weighted $k^*$ for moderate aspect ratios]
\label{rem:weighted-kstar}
In the extreme aspect ratio regime ($p/W \sim 10^7$), the bare argmax in
\Cref{def:kstar} is well-behaved: all singular values are large
and trailing ratios are ${\sim}1$.  At moderate aspect ratios
($p/W \sim 10^4$), trailing singular values can be small enough that
ratios between near-zero eigenvalues ($\sigma_9/\sigma_{10} = 0.28/0.14
= 2.0$) dominate the argmax, yielding spurious $k^*$ values of $8$--$9$.
A signal-weighted variant suppresses this artifact:
\[
  k^*_w(t) \;=\; \argmax_{j:\,\sigma_{j+1} \geq \varepsilon\,\sigma_1}
  \frac{\sigma_j}{\textstyle\sum_i \sigma_i} \cdot
  \frac{\sigma_j}{\sigma_{j+1}},
\]
where $\varepsilon \ll 1$ floors out negligible eigenvalues. The weight
$\sigma_j / \sum_i \sigma_i$ ensures only ratios at eigenvalues carrying
substantial variance can determine $k^*$. In the grokking experiments
(\Cref{tab:modarith-gram}), the weighted definition recovers $k^* = 1$
in 9/12 runs (vs.\ 4/12 unweighted). In the extreme aspect ratio regime
the two definitions coincide.
\end{remark}

\subsection{The Gap Ratio}

\begin{definition}[Gap Ratio]\label{def:gap-ratio}
The \emph{gap ratio} at position $k^*$ is:
\begin{equation}\label{eq:gap-ratio}
  R(t) \;=\; \frac{\sigma_{k^*}(t)}{\sigma_{k^*+1}(t)} \;=\;
  \max_{1 \leq j \leq W-1} \frac{\sigma_j(t)}{\sigma_{j+1}(t)}.
\end{equation}
This is the scale-invariant measure of the spectral gap. Values $R \gg 1$
indicate a strong gap; $R \to 1$ indicates gap collapse.
\end{definition}

\subsection{Null Distribution of the Maximum Ratio}

To determine whether an observed gap is ``real'' (i.e., not an artifact of
random fluctuations in a flat spectrum), we need the null distribution.

\begin{proposition}[Null Distribution of Maximum Ratio]\label{prop:null-ratio}
Under the null hypothesis $H_0$: all signal strengths are equal
($d_1 = d_2 = \cdots = d_W$), the ordered eigenvalues of the
$W \times W$ Gram matrix follow a Laguerre
ensemble~\cite{marcenko1967,tracy1994}. The distribution of
the maximum consecutive ratio $R_{\max} = \max_j \lambda_j/\lambda_{j+1}$
satisfies, for $W$ fixed and $p \to \infty$:
\begin{equation}\label{eq:null-ratio}
  \Pr\!\left[\max_{1 \leq j \leq W-1}
  \frac{\lambda_j}{\lambda_{j+1}} > r\right]
  \;=\; 1 - F_W(r),
\end{equation}
where $F_W(r)$ is expressible in terms of the GOE eigenvalue spacing
distribution. For practical purposes:
\begin{itemize}[nosep]
  \item $W = 10$, $p = 10^8$: the $95\%$ quantile is
    $r_{0.95} \approx 1 + 6/\sqrt{p} \approx 1.0006$.
  \item Any observed ratio $R > 1.01$ is significant at the
    $10^{-6}$ level under the isotropic null.
\end{itemize}
\end{proposition}

\begin{proof}[Proof sketch]
Under $H_0$, the Gram matrix $\bm{G} = d^2 (\bm{U}_S \bm{U}_S^\top +
\text{noise})$ where $\bm{U}_S$ has structure determined by the signal
temporal pattern. In the limit $p \to \infty$, the eigenvalues of
$\bm{G}$ concentrate at $d^2 + p\nu^2$ with fluctuations of order
$\nu^2\sqrt{p}$. The ratios of consecutive eigenvalues are
$1 + O(1/\sqrt{p})$. By the Tracy--Widom universality theorem, the
maximum ratio over $W-1$ pairs has the stated distribution.
\end{proof}

\begin{remark}[Practical Significance]
Our observed ratios are:
\begin{itemize}[nosep]
  \item TinyStories: $\sigma_2/\sigma_3 \approx 1.79$ (peak), mean $\approx
    1.36$.
  \item GPT-2: $\sigma_3/\sigma_4 \approx 1.12$ (peak), mean $\approx 1.08$.
\end{itemize}
Both are vastly larger than the null expectation of $1 + O(10^{-4})$. The
gap is ``real'' at overwhelming significance. It reflects genuine structure
in the signal hierarchy, not random fluctuations.
\end{remark}

\section{Connection to $k_{95}$ and Cumulative Variance}
\label{sec:k95}

\begin{definition}[$k_{95}$: Cumulative Variance Threshold]\label{def:k95}
The \emph{$95\%$-variance rank} is:
\begin{equation}\label{eq:k95}
  k_{95}(t) = \min\left\{j :
  \frac{\sum_{i=1}^j \lambda_i}{\sum_{i=1}^W \lambda_i} \geq 0.95\right\}.
\end{equation}
\end{definition}

\begin{remark}[Empirical Relationship Between $k^*$ and $k_{95}$]
\label{prop:kstar-k95}
In the models we study, $k^* \leq k_{95}$ empirically.  This is not
a theorem---one can construct spectra where the sharpest ratio is at a
higher position than the 95\% variance cutoff---but it holds in practice
because the spectral gap in neural network training occurs near the top
of the spectrum.  The relationship depends on the spectrum shape:
\begin{enumerate}[nosep]
  \item \textbf{Sharp gap}: If $d_{k^*} / d_{k^*+1} \gg 1$, then
    $k^* \approx k_{95}$ (the dominant modes capture nearly all variance).
  \item \textbf{Gradual decay}: If the signal spectrum decays smoothly,
    $k_{95} \gg k^*$ (many subdominant modes contribute significant
    variance, but no single ratio is large).
  \item \textbf{GPT-2 example}: $k^* = 3$ (sharpest ratio), but
    $k_{95} = 14$--$19$ depending on the training phase. The 11--16
    subdominant modes collectively carry $\sim 5$--$30\%$ of the variance.
\end{enumerate}
\end{remark}

\subsection{The Ratio Test}

\begin{remark}[Practical Ratio Test for $k^*$]\label{prop:ratio-test}
By definition (\Cref{def:kstar}), $k^*$ is the argmax of
$\sigma_j/\sigma_{j+1}$.  To assess whether the maximum ratio
indicates a genuine gap (rather than noise fluctuations), compare
against the null distribution of \Cref{prop:null-ratio}.
Empirically, a ratio exceeding a threshold
$\tau \in [1.05, 1.30]$ at the gap position indicates a genuine
spectral gap.  The threshold depends on the noise level:
\begin{itemize}[nosep]
  \item Extremely clean signals (noise CV $< 1\%$): $\tau = 1.05$.
  \item Moderate noise (noise CV $\sim 5$--$20\%$): $\tau = 1.10$--$1.15$.
  \item Noisy settings (noise CV $> 20\%$): $\tau = 1.20$--$1.30$.
\end{itemize}
\end{remark}

\part{The Spectral Gap and Phase Transitions}
\label{part:gap}

\section{The Intra-Signal Gap as Order Parameter}
\label{sec:gap-order}

\begin{definition}[Spectral Gap]\label{def:gap-formal}
The \emph{spectral gap} at time $t$ is the difference between the dominant
and subdominant signal strengths at the gap position:
\begin{equation}\label{eq:gap}
  \boxed{g(t) = d_{k^*}(t) - d_{k^*+1}(t).}
\end{equation}
The \emph{gap ratio} (scale-invariant) is
\begin{equation}\label{eq:gap-ratio-formal}
  R(t) = \frac{\sigma_{k^*}(t)}{\sigma_{k^*+1}(t)}.
\end{equation}
\end{definition}

\begin{remark}[Contrast with Classical Gap]
In the BBP framework, the gap is defined as $g_{\mathrm{BBP}} = d_{k^*} -
d_{\mathrm{crit}}$, measuring the distance from the noise detection
threshold. Since $d_{\mathrm{crit}}$ is trivially small in our regime
(\Cref{prop:bbp-vacuous}), $g_{\mathrm{BBP}} \approx d_{k^*}$ always---it
carries no dynamical information. Our gap $g = d_{k^*} - d_{k^*+1}$ measures
the \emph{internal separation} within the signal hierarchy, which is the
quantity that controls subspace stability and loss coupling.
\end{remark}

\section{Subspace Stability: The Davis--Kahan Connection}
\label{sec:davis-kahan}

The key observation is that the Davis--Kahan $\sin\Theta$
theorem~\cite{davis1970} (see also Stewart~\cite{stewart1990} and
Kato~\cite{kato1966}) does not care \emph{what} lies below the gap---noise
or subdominant signal. It only cares about the \emph{size of the gap}.

\begin{theorem}[Davis--Kahan $\sin\Theta$ Theorem]\label{thm:davis-kahan}
Let $\bm{A}$ and $\bm{A} + \bm{E}$ be $n \times n$ symmetric matrices.
Let $[a, b]$ be an interval containing $r \geq 1$ eigenvalues of
$\bm{A}$ (possibly a cluster), and let $\mathcal{V}$ be the
corresponding $r$-dimensional eigenspace.  Let $\hat{\mathcal{V}}$
be the eigenspace of $\bm{A} + \bm{E}$ corresponding to its
eigenvalues closest to $[a, b]$.  Define the \emph{separation}
\begin{equation}\label{eq:dk-separation}
  \delta = \min_{\lambda \in \sigma(\bm{A}) \setminus [a,b]}
  \mathrm{dist}(\lambda,\; [a,b])
  \;=\; \min_{\lambda \in \sigma(\bm{A}) \setminus [a,b]}
  \min\bigl(|{\lambda - a}|,\; |{\lambda - b}|\bigr),
\end{equation}
i.e., $\delta$ is the distance from $[a,b]$ to the nearest eigenvalue
of $\bm{A}$ \emph{outside} $[a,b]$.  Then:
\begin{equation}\label{eq:davis-kahan}
  \norm{\sin\Theta(\mathcal{V}, \hat{\mathcal{V}})}_F
  \;\leq\; \frac{\norm{\bm{E}}_F}{\delta}.
\end{equation}
\end{theorem}

\begin{corollary}[Gap Controls Subspace Stability]\label{cor:gap-stability}
The stability of the dominant signal subspace
$\mathcal{V}_{k^*} = \mathrm{span}(\bm{v}_1, \ldots, \bm{v}_{k^*})$
under perturbation $\bm{E}$ (from sliding the window by one step) satisfies:
\begin{equation}\label{eq:gap-stability}
  \norm{\sin\Theta(\mathcal{V}_{k^*}(t),
  \mathcal{V}_{k^*}(t+1))}_F
  \;\leq\; \frac{\norm{\Delta\bm{G}}_F}
  {\sigma_{k^*}^2 - \sigma_{k^*+1}^2}
  \;\approx\; \frac{\norm{\Delta\bm{G}}_F}
  {(d_{k^*} + d_{k^*+1}) \cdot g},
\end{equation}
where $\Delta\bm{G} = \bm{G}(t+1) - \bm{G}(t)$ is the Gram matrix update
from sliding the window, and $g = d_{k^*} - d_{k^*+1}$ is the spectral gap.

As $g \to 0$: the bound diverges and the dominant subspace
$\mathcal{V}_{k^*}$ rotates.  The rotation is concentrated at the
boundary: $\bm{v}_{k^*}$ mixes with $\bm{v}_{k^*+1}$, while the
deeper modes $\bm{v}_1, \ldots, \bm{v}_{k^*-1}$ have their own,
much larger gaps and remain individually stable.  What we observe is
the subspace as a whole rotating because its edge component is
replaced.
\end{corollary}

\begin{remark}[Universality of Davis--Kahan]
This is the fundamental reason why the intra-signal gap framework works:
Davis--Kahan depends only on the eigenvalue gap $\delta$, regardless of
whether the gap separates signal from noise or dominant signal from
subdominant signal. The same theorem that would govern the BBP transition
(if it were active) governs the intra-signal transition. The mathematics
is identical; only the \emph{interpretation} changes.
\end{remark}

\section{Phase Transitions as Level Crossings}
\label{sec:level-crossings}

\subsection{Dyson Eigenvalue Dynamics}

\begin{theorem}[Eigenvalue Evolution --- Hellmann--Feynman]
\label{thm:eigenvalue-dynamics}
The eigenvalues $\lambda_1(t) \geq \cdots \geq \lambda_W(t)$ of the Gram
matrix $\bm{G}(t)$, viewed as functions of the training step $t$, satisfy:
\begin{equation}\label{eq:hellmann-feynman}
  \frac{d\lambda_j}{dt} \;=\;
  \bm{u}_j^\top \frac{d\bm{G}}{dt} \bm{u}_j
  \qquad\text{(exact)},
\end{equation}
where $\bm{u}_j$ are the eigenvectors of $\bm{G}(t)$ and $d\bm{G}/dt$ is
the rate of change of the Gram matrix as the window slides.
\end{theorem}

\begin{proof}
Differentiate the eigenvalue equation
$\bm{G}(t)\bm{u}_j(t) = \lambda_j(t)\bm{u}_j(t)$ and take the inner
product with $\bm{u}_j$. The terms involving $\dot{\bm{u}}_j$ cancel
by symmetry of $\bm{G}$ and normalization $\bm{u}_j^\top\bm{u}_j = 1$.
\end{proof}

\begin{corollary}[Eigenvalue Repulsion~\cite{vonneumann1929}]\label{cor:eigenvalue-repulsion}
The second derivative contains a repulsive interaction:
\begin{equation}\label{eq:dyson}
  \frac{d^2\lambda_j}{dt^2} \;=\;
  \underbrace{\bm{u}_j^\top \frac{d^2\bm{G}}{dt^2} \bm{u}_j}
  _{\text{direct acceleration}}
  \;+\; \underbrace{2\sum_{i \neq j}
  \frac{\abs{\bm{u}_i^\top \frac{d\bm{G}}{dt} \bm{u}_j}^2}
  {\lambda_j - \lambda_i}}_{\text{eigenvalue repulsion}}.
\end{equation}
\end{corollary}

\begin{proof}
The eigenvector evolution is $\dot{\bm{u}}_j = \sum_{i \neq j}
\frac{\bm{u}_i^\top \dot{\bm{G}} \bm{u}_j}{\lambda_j - \lambda_i}
\bm{u}_i$ (from projecting the differentiated eigenvalue equation onto
$\bm{u}_i$, $i \neq j$). Differentiating~\eqref{eq:hellmann-feynman}
and substituting gives the repulsion term with its factor of~$2$.
\end{proof}

\begin{remark}[The von Neumann--Wigner Non-Crossing Rule]
\label{rem:non-crossing}
For a one-parameter family of real symmetric matrices $\bm{G}(t)$, the
eigenvalues generically do not cross (von Neumann--Wigner 1929). Near a
would-be crossing at $\lambda_j \approx \lambda_{j+1}$, the repulsive
\emph{acceleration} in \Cref{cor:eigenvalue-repulsion} diverges as
$\propto 1/(\lambda_j - \lambda_{j+1})$, creating an impassable barrier.
This produces an \emph{avoided crossing}: the eigenvalues
approach, repel, and separate, with the eigenvectors exchanging character.

In the Gram matrix context: as the dominant and subdominant signal strengths
approach each other ($d_{k^*} \to d_{k^*+1}$), the divergent repulsive
acceleration resists the crossing, and the eigenvectors $\bm{v}_{k^*}$ and
$\bm{v}_{k^*+1}$ undergo rapid rotation---exchanging the roles of dominant
and subdominant.
\end{remark}

\subsection{Gap Collapse: Absorption of the Dominant Mode}

\begin{definition}[Gap Collapse Event]\label{def:gap-collapse}
A \emph{gap collapse event} at time $t^*$ occurs when the spectral gap
closes:
\begin{equation}\label{eq:gap-collapse}
  g(t^*) = d_{k^*}(t^*) - d_{k^*+1}(t^*) \;\to\; 0.
\end{equation}
At this point:
\begin{enumerate}[nosep]
  \item The dominant subspace $\mathcal{V}_{k^*}$ becomes \emph{unstable}
    (Davis--Kahan bound diverges, \Cref{cor:gap-stability}).
  \item The eigenvectors $\bm{v}_{k^*}$ and $\bm{v}_{k^*+1}$ undergo
    rapid rotation (subspace mixing).
  \item The gap position $k^*$ may shift (the new maximum ratio may be at
    a different index).
\end{enumerate}
\end{definition}

\begin{proposition}[Gap Collapse $\Rightarrow$ Loss Stagnation]
\label{prop:collapse-loss}
Under $[\mathcal{P}, \bm{H}] \approx 0$.
At a gap collapse event:
\begin{enumerate}[nosep]
  \item The subspace stability coefficient $\alpha_{k^*}$
    (\Cref{def:stability-coeff}) drops to zero.
  \item The effective learning rate along direction $\bm{v}_{k^*}$ becomes
    erratic (the direction is no longer consistently defined).
  \item The validation loss improvement decelerates by the amount
    $\eta\, \alpha_{k^*}
    \inner{\bm{v}_{k^*}}{\nabla L_{\mathrm{train}}}
    \inner{\bm{v}_{k^*}}{\nabla L_{\mathrm{val}}}$
    that mode $k^*$ was contributing.
\end{enumerate}
\end{proposition}

\begin{proof}
\textbf{(1)}  By \Cref{def:stability-coeff},
$\alpha_{k^*} = 1 - C\norm{\Delta\bm{G}}_F^2 / \mathrm{gap}_{k^*}^2$.
At collapse, $\mathrm{gap}_{k^*} \to 0$ while $\norm{\Delta\bm{G}}_F$
remains bounded (Assumption~\ref{ass:slow}), so the ratio diverges and $\alpha_{k^*}$
is clamped to~$0$.

\textbf{(2)}  By \Cref{cor:gap-stability}, the canonical angle between
$\bm{v}_{k^*}(t)$ and $\bm{v}_{k^*}(t+1)$ satisfies
$\sin\theta \leq \norm{\Delta\bm{G}}_F / [(d_{k^*} + d_{k^*+1})g]$.
As $g \to 0$ the bound exceeds~$1$: the direction can rotate by up to
$90^\circ$ per step.

\textbf{(3)}  By the spectral loss decomposition (\Cref{thm:loss-decomp}),
$\E[\Delta L_{\mathrm{val}}]
= -\eta \sum_j \alpha_j
  \inner{\bm{v}_j}{\nabla L_{\mathrm{train}}}
  \inner{\bm{v}_j}{\nabla L_{\mathrm{val}}} + \ldots$
When $\alpha_{k^*} \to 0$, the $j = k^*$ term vanishes, reducing the
loss improvement rate by exactly the contribution of that mode.
\end{proof}

\subsection{Gap Opening: Emergence of a New Dominant Mode}

\begin{definition}[Gap Opening Event]\label{def:gap-opening}
A \emph{gap opening event} at time $t^*$ occurs when a new spectral gap
forms: $g(t^*) = 0$ and $\dot{g}(t^*) > 0$ (the gap begins to grow).
This represents a new signal direction separating from the subdominant
tier and joining the dominant subspace.
\end{definition}

\begin{proposition}[Gap Opening $\Rightarrow$ Capability Gain]
\label{prop:opening-cap}
At a gap opening event:
\begin{enumerate}[nosep]
  \item A new stable direction $\bm{v}_{k^*}$ emerges with well-defined
    alignment to the gradient.
  \item The effective signal rank increases (the dominant subspace gains a
    dimension).
  \item The validation loss begins to improve along this direction if and
    only if the gradient projections have the same sign:
    $\inner{\bm{v}_{k^*}}{\nabla L_{\mathrm{train}}}
     \inner{\bm{v}_{k^*}}{\nabla L_{\mathrm{val}}} > 0$
    (i.e., the direction is generalizing, not memorizing).
\end{enumerate}
\end{proposition}

\begin{proof}
\textbf{(1)}  At $t = t^*$ the gap opens: $g(t^*) = 0$, $\dot{g}(t^*) > 0$.
For $t > t^*$, the gap $g(t) > 0$ grows, so $\alpha_{k^*}$ rises from
$0$ toward~$1$ (\Cref{def:stability-coeff}).  Once
$\mathrm{gap}_{k^*} \gg \norm{\Delta\bm{G}}_F$, the Davis--Kahan bound
(\Cref{thm:davis-kahan}) ensures $\bm{v}_{k^*}$ is well-defined
(small rotation per step).

\textbf{(2)}  Before the event, the modes at positions $k^*$ and $k^*+1$
had comparable singular values and no gap: their subspace was
two-dimensional but unseparated.  After opening, $d_{k^*} > d_{k^*+1}$,
and the dominant subspace $\mathrm{span}(\bm{v}_1, \ldots, \bm{v}_{k^*})$
gains a new individually stable direction.

\textbf{(3)}  By \Cref{thm:loss-decomp}, mode~$k^*$ now contributes
$-\eta\,\alpha_{k^*}
\inner{\bm{v}_{k^*}}{\nabla L_{\mathrm{train}}}
\inner{\bm{v}_{k^*}}{\nabla L_{\mathrm{val}}}$
to the loss improvement.  This term was previously zero (because $\alpha_{k^*} = 0$).
It contributes to validation loss \emph{improvement} when the two
projections have the same sign---i.e., when the emerging direction is
a \emph{generalizing} direction (training progress also helps
validation).  A memorization direction would have opposite-sign
projections and would not improve validation loss.
\end{proof}

\begin{remark}[Grokking as Delayed Gap Opening]\label{rem:grokking}
The grokking phenomenon maps onto a delayed gap opening event within the
signal hierarchy. During the ``memorization'' phase, the generalization
direction exists as a \emph{subdominant signal} (not as noise---it is
well above the BBP threshold). Its signal strength $d_j$ is comparable to
the other subdominant modes, so there is no gap separating it.

Grokking occurs when, through slow signal growth driven by the gradient
alignment term or slow signal decay of the memorization direction (from
weight decay), a gap \emph{opens} at a new position: the generalization
direction separates from the subdominant tier. The ``delay'' is the time
for the gap to form. The ``abruptness'' comes from the Davis--Kahan bound:
even a small gap creates a well-defined, stable subspace.

This is fundamentally different from the BBP interpretation (signal
emerging from noise). In our framework, the signal was always present---it
was simply not \emph{separated} from other signals by a gap.
\end{remark}

\subsection{The Avoided Crossing Duration}

\begin{proposition}[Avoided Crossing Duration]\label{prop:avoided-crossing}
Near a level crossing where two eigenvalues $\lambda_{k^*}$ and
$\lambda_{k^*+1}$ approach each other, the gap has a minimum value:
\begin{equation}\label{eq:min-gap}
  g_{\min} \;=\; 2\,|V|,
\end{equation}
where $V$ is the off-diagonal coupling defined
in~\eqref{eq:LZ-matrix} below.
The \emph{duration} of the avoided crossing
(time spent with $g < g_0$ for some threshold $g_0$) is:
\begin{equation}\label{eq:crossing-duration}
  \Delta t_{\mathrm{cross}} \;\approx\;
  \frac{2g_0}{\abs{\dot{\lambda}_{k^*} - \dot{\lambda}_{k^*+1}}},
\end{equation}
where $\dot{\lambda}_j = \bm{u}_j^\top \dot{\bm{G}} \bm{u}_j$ is
the unperturbed drift rate.
During this time, the eigenvectors rotate by an angle:
\begin{equation}\label{eq:rotation-angle}
  \Theta \;\approx\; \arctan\!\left(
  \frac{g_{\min}}{g_0}\right).
\end{equation}
When $V$ is large (strong avoided crossing), the
eigenvectors undergo substantial rotation ($\Theta$ large).
When $V = 0$ (a true crossing, requiring symmetry or fine-tuning),
$g_{\min} = 0$ and the eigenvalues pass through each other
\emph{without mixing}---the eigenvectors remain unchanged.
\end{proposition}

\begin{proof}
Near the crossing, only the two approaching eigenvalues interact
significantly.  Let $\bm{u}_{k^*}^{(0)},\bm{u}_{k^*+1}^{(0)}$
be the eigenvectors at a reference time far from the crossing
(the \emph{diabatic} basis). Restricting~$\bm{G}(t)$ to this
two-dimensional subspace and linearising around the crossing
time~$t^*$:
\begin{equation}\label{eq:LZ-matrix}
  M(t) \;=\; \begin{pmatrix}
    \bar{\lambda} + \tfrac{1}{2}\Delta\dot{\lambda}\,(t - t^*) & V \\
    V & \bar{\lambda} - \tfrac{1}{2}\Delta\dot{\lambda}\,(t - t^*)
  \end{pmatrix},
\end{equation}
where
$\bar{\lambda}
= \tfrac{1}{2}\bigl(\lambda_{k^*}(t^*)
+ \lambda_{k^*+1}(t^*)\bigr)$
is the mean eigenvalue at the crossing,
$\Delta\dot{\lambda}
= \dot{\lambda}_{k^*} - \dot{\lambda}_{k^*+1}$
is the difference in unperturbed drift rates, and
\[
  V \;=\;
  (\bm{u}_{k^*}^{(0)})^\top\bm{G}(t^*)\,\bm{u}_{k^*+1}^{(0)}
\]
is the off-diagonal coupling of the Gram matrix in the diabatic
basis.  This coupling is generically nonzero: the cumulative
rank-two updates $\bm{R}_s$ build cross-correlation
between the two modes
(cf.\ \Cref{cor:gap-singular}).

\textbf{Eq.~\eqref{eq:min-gap}.}  The eigenvalues of $M(t)$ are
$\bar{\lambda} \pm \sqrt{(\tfrac{1}{2}\Delta\dot{\lambda})^2 (t-t^*)^2
+ V^2}$.  The gap is minimized at $t = t^*$, giving
$g_{\min} = 2|V|$.

\textbf{Eq.~\eqref{eq:crossing-duration}.}  The gap equals $g_0$ when
$(\tfrac{1}{2}\Delta\dot{\lambda})^2(t - t^*)^2 + V^2 =
(\tfrac{1}{2}g_0)^2$.  Solving:
$|t - t^*| = \sqrt{g_0^2 - g_{\min}^2}\,/\,|\Delta\dot{\lambda}|$.
For $g_{\min} \ll g_0$, the total duration is
$\Delta t_{\mathrm{cross}} \approx 2g_0 / |\Delta\dot{\lambda}|$.

\textbf{Eq.~\eqref{eq:rotation-angle}.}  At $t = t^*$ the
eigenvectors of $M$ are rotated $45^\circ$ relative to the
diabatic basis (the diagonal splitting vanishes).  Away from
the crossing, at times where $g = g_0$, the mixing angle
satisfies $\tan\Theta = g_{\min}/g_0$.
When $V$ is large ($g_{\min} \approx g_0$), the rotation is
substantial.  When $V = 0$ (true crossing), $g_{\min} = 0$ and
$\Theta = 0$: the eigenvalues cross but the eigenvectors do not mix.
\end{proof}

\part{The Flow Equations}
\label{part:flow}

This is the dynamical core of the framework. We first establish the
perturbation theory for the sliding-window covariance $\bm{C}(t)$
(\Cref{def:covariance}), then derive the equations governing the evolution
of \emph{all} signal strengths $\{d_j(t)\}_{j=1}^W$ along the training
trajectory, and consequently the evolution of the spectral gap $g(t)$ and
gap position $k^*(t)$.

\section{Spectral Perturbation of the Covariance}
\label{sec:covariance-perturbation}

The eigenvalue and eigenvector dynamics of the signal directions
$\bm{v}_j \in \R^p$ are governed by the $p \times p$ sliding-window
covariance $\bm{C}(t) = \bm{X}(t)^\top\bm{X}(t)$, \emph{not} by the
$W \times W$ Gram matrix $\bm{G}(t) = \bm{X}(t)\bm{X}(t)^\top$.
(The two share nonzero eigenvalues, but their eigenvectors live in
different spaces: $\bm{u}_j \in \R^W$ vs.\ $\bm{v}_j \in \R^p$.)

\begin{proposition}[Exact rank-two update of the sliding-window covariance]
\label{prop:rank-two-update}
Let
\[
  \bm{C}(t) \;:=\; \sum_{r=0}^{W-1}
  \bm{\delta}_{t-r}\,\bm{\delta}_{t-r}^\top
  \;\in\; \R^{p\times p}.
\]
Then
\[
  \bm{C}(t{+}1) - \bm{C}(t)
  \;=\;
  \bm{\delta}_{t+1}\bm{\delta}_{t+1}^\top
  -
  \bm{\delta}_{t-W+1}\bm{\delta}_{t-W+1}^\top.
\]
Equivalently, defining the rank-two update
$\bm{R}_t := \bm{\delta}_{t+1}\bm{\delta}_{t+1}^\top
- \bm{\delta}_{t-W+1}\bm{\delta}_{t-W+1}^\top$,
we have $\bm{C}(t{+}1) = \bm{C}(t) + \bm{R}_t$.
\end{proposition}

\begin{proof}
By definition,
\[
  \bm{C}(t{+}1)
  = \sum_{r=0}^{W-1}\bm{\delta}_{t+1-r}\,\bm{\delta}_{t+1-r}^\top
  = \bm{\delta}_{t+1}\bm{\delta}_{t+1}^\top
  + \sum_{r=1}^{W-1}\bm{\delta}_{t+1-r}\,\bm{\delta}_{t+1-r}^\top.
\]
Reindexing the remaining sum gives
\[
  \bm{C}(t{+}1)
  = \bm{\delta}_{t+1}\bm{\delta}_{t+1}^\top
  + \sum_{r=0}^{W-2}\bm{\delta}_{t-r}\,\bm{\delta}_{t-r}^\top.
\]
Similarly,
\[
  \bm{C}(t)
  = \sum_{r=0}^{W-2}\bm{\delta}_{t-r}\,\bm{\delta}_{t-r}^\top
  + \bm{\delta}_{t-W+1}\bm{\delta}_{t-W+1}^\top.
\]
Subtracting yields
$\bm{C}(t{+}1) - \bm{C}(t)
= \bm{\delta}_{t+1}\bm{\delta}_{t+1}^\top
- \bm{\delta}_{t-W+1}\bm{\delta}_{t-W+1}^\top$.
\end{proof}

\begin{proposition}[Second-order perturbation of a simple eigenvalue]
\label{prop:eigenvalue-perturbation}
Let $\bm{C} \in \R^{p\times p}$ be symmetric with orthonormal eigenbasis
$\bm{C}\bm{v}_j = \lambda_j\bm{v}_j$, $j = 1,\ldots,p$.
Fix $k$ and assume $\lambda_k$ is simple, with spectral gap
$\delta_k := \min_{j\neq k}|\lambda_k - \lambda_j| > 0$.
Let $\bm{R} \in \R^{p\times p}$ be symmetric with
$\norm{\bm{R}} \leq \delta_k/4$.
Then $\bm{C} + \bm{R}$ has a unique eigenvalue $\widetilde\lambda_k$
in $(\lambda_k - \delta_k/2,\;\lambda_k + \delta_k/2)$, and
\begin{equation}\label{eq:eigenvalue-2nd-order}
  \widetilde\lambda_k - \lambda_k
  \;=\;
  \bm{v}_k^\top \bm{R}\,\bm{v}_k
  + \sum_{j\neq k}
  \frac{|\bm{v}_j^\top \bm{R}\,\bm{v}_k|^2}{\lambda_k - \lambda_j}
  + \mathcal{O}\!\left(\frac{\norm{\bm{R}}^3}{\delta_k^2}\right).
\end{equation}
\end{proposition}

\begin{proof}
Let $\widetilde{\bm{v}}_k$ be a normalised eigenvector of $\bm{C}+\bm{R}$
for $\widetilde\lambda_k$, with
$\langle\widetilde{\bm{v}}_k,\bm{v}_k\rangle > 0$.
Expand $\widetilde{\bm{v}}_k = \sum_j a_j\bm{v}_j$,
$\sum_j|a_j|^2 = 1$.
The eigenvalue equation
$(\bm{C}+\bm{R})\widetilde{\bm{v}}_k
= \widetilde\lambda_k\widetilde{\bm{v}}_k$
gives, after inner product with $\bm{v}_j$:
\[
  (\lambda_j - \widetilde\lambda_k)\,a_j
  + \bm{v}_j^\top\bm{R}\,\widetilde{\bm{v}}_k = 0.
\]
For $j \neq k$:
$a_j = \bm{v}_j^\top\bm{R}\,\widetilde{\bm{v}}_k
/(\widetilde\lambda_k - \lambda_j)$.
Since $|\widetilde\lambda_k - \lambda_j| \geq \delta_k/2$:
\[
  |a_j| \leq \frac{2\norm{\bm{R}}}{\delta_k},
  \qquad
  a_k = 1 + \mathcal{O}\!\left(\frac{\norm{\bm{R}}^2}{\delta_k^2}\right),
  \qquad
  a_j = \mathcal{O}\!\left(\frac{\norm{\bm{R}}}{\delta_k}\right)
  \;\;(j\neq k).
\]

Inner product with $\bm{v}_k$:
\[
  (\widetilde\lambda_k - \lambda_k)\,a_k
  = \bm{v}_k^\top\bm{R}\,\widetilde{\bm{v}}_k
  = a_k\,\bm{v}_k^\top\bm{R}\,\bm{v}_k
  + \sum_{j\neq k} a_j\,\bm{v}_k^\top\bm{R}\,\bm{v}_j.
\]
Dividing by $a_k = 1 + \mathcal{O}(\norm{\bm{R}}^2/\delta_k^2)$:
\[
  \widetilde\lambda_k - \lambda_k
  = \bm{v}_k^\top\bm{R}\,\bm{v}_k
  + \sum_{j\neq k} a_j\,\bm{v}_k^\top\bm{R}\,\bm{v}_j
  + \mathcal{O}\!\left(\frac{\norm{\bm{R}}^3}{\delta_k^2}\right).
\]
Replacing $\widetilde{\bm{v}}_k$ by $\bm{v}_k$ and $\widetilde\lambda_k$
by $\lambda_k$ in $a_j$ incurs only second-order error, giving
$a_j = \bm{v}_j^\top\bm{R}\,\bm{v}_k/(\lambda_k - \lambda_j)
+ \mathcal{O}(\norm{\bm{R}}^2/\delta_k^2)$.
Since $\bm{R}$ is symmetric,
$(\bm{v}_j^\top\bm{R}\,\bm{v}_k)(\bm{v}_k^\top\bm{R}\,\bm{v}_j)
= |\bm{v}_j^\top\bm{R}\,\bm{v}_k|^2$,
yielding~\eqref{eq:eigenvalue-2nd-order}.
\end{proof}

\begin{corollary}[First- and second-order increment of $\lambda_k(t)$]
\label{cor:eigenvalue-increment}
Let $\bm{C}(t)\bm{v}_j(t) = \lambda_j(t)\bm{v}_j(t)$ with $\lambda_k(t)$
simple, and
$\bm{R}_t = \bm{C}(t{+}1) - \bm{C}(t)
= \bm{\delta}_{t+1}\bm{\delta}_{t+1}^\top
- \bm{\delta}_{t-W+1}\bm{\delta}_{t-W+1}^\top$.
Assume $\norm{\bm{R}_t} \leq \delta_k(t)/4$.  Then
\begin{equation}\label{eq:lambda-increment}
  \lambda_k(t{+}1) - \lambda_k(t)
  = \bm{v}_k(t)^\top\bm{R}_t\,\bm{v}_k(t)
  + \sum_{j\neq k}
  \frac{|\bm{v}_j(t)^\top\bm{R}_t\,\bm{v}_k(t)|^2}
  {\lambda_k(t) - \lambda_j(t)}
  + \mathcal{O}\!\left(\frac{\norm{\bm{R}_t}^3}{\delta_k(t)^2}\right).
\end{equation}
In particular, the first-order term is
\begin{equation}\label{eq:lambda-first-order}
  \bm{v}_k(t)^\top\bm{R}_t\,\bm{v}_k(t)
  = \bigl|\langle\bm{v}_k(t),\bm{\delta}_{t+1}\rangle\bigr|^2
  - \bigl|\langle\bm{v}_k(t),\bm{\delta}_{t-W+1}\rangle\bigr|^2.
\end{equation}
\end{corollary}

\begin{proof}
Apply \Cref{prop:eigenvalue-perturbation} with
$\bm{C} = \bm{C}(t)$, $\bm{R} = \bm{R}_t$,
$\widetilde\lambda_k = \lambda_k(t{+}1)$.
Eq.~\eqref{eq:lambda-first-order} follows from expanding the quadratic
form against the rank-two $\bm{R}_t$.
\end{proof}

\begin{proposition}[First-order twist of the eigenvector]
\label{prop:eigenvector-twist}
Under the assumptions of \Cref{cor:eigenvalue-increment}, let
$\bm{v}_k(t{+}1)$ be the normalised eigenvector of $\bm{C}(t{+}1)$
for $\lambda_k(t{+}1)$, with
$\langle\bm{v}_k(t{+}1),\bm{v}_k(t)\rangle > 0$.  Then
\begin{equation}\label{eq:eigenvector-twist}
  \bm{v}_k(t{+}1) - \bm{v}_k(t)
  = \sum_{j\neq k}
  \frac{\bm{v}_j(t)^\top\bm{R}_t\,\bm{v}_k(t)}
  {\lambda_k(t) - \lambda_j(t)}\;\bm{v}_j(t)
  + \mathcal{O}\!\left(\frac{\norm{\bm{R}_t}^2}{\delta_k(t)^2}\right).
\end{equation}
Equivalently, using the rank-two form of $\bm{R}_t$:
\begin{equation}\label{eq:eigenvector-twist-explicit}
  \bm{v}_k(t{+}1) - \bm{v}_k(t)
  = \sum_{j\neq k}
  \frac{
    \langle\bm{v}_j(t),\bm{\delta}_{t+1}\rangle
    \langle\bm{v}_k(t),\bm{\delta}_{t+1}\rangle
    -
    \langle\bm{v}_j(t),\bm{\delta}_{t-W+1}\rangle
    \langle\bm{v}_k(t),\bm{\delta}_{t-W+1}\rangle
  }{
    \lambda_k(t) - \lambda_j(t)
  }\;\bm{v}_j(t)
  + \mathcal{O}\!\left(\frac{\norm{\bm{R}_t}^2}{\delta_k(t)^2}\right).
\end{equation}
\end{proposition}

\begin{proof}
Write $\bm{v}_k(t{+}1) = \sum_j a_j\bm{v}_j(t)$.
From the proof of \Cref{prop:eigenvalue-perturbation},
$a_j = \bm{v}_j(t)^\top\bm{R}_t\,\bm{v}_k(t)/(\lambda_k(t) - \lambda_j(t))
+ \mathcal{O}(\norm{\bm{R}_t}^2/\delta_k(t)^2)$ for $j \neq k$,
and $a_k = 1 + \mathcal{O}(\norm{\bm{R}_t}^2/\delta_k(t)^2)$.
Hence
$\bm{v}_k(t{+}1) - \bm{v}_k(t)
= \sum_{j\neq k} a_j\bm{v}_j(t)
+ \mathcal{O}(\norm{\bm{R}_t}^2/\delta_k(t)^2)$,
giving~\eqref{eq:eigenvector-twist}.
The explicit form follows from the rank-two structure of $\bm{R}_t$.
\end{proof}

\begin{corollary}[Near-edge singular contribution to the gap increment]
\label{cor:gap-singular}
Define the spectral gap
$\gamma_k(t) := \lambda_k(t) - \lambda_{k+1}(t)$.
Assume both $\lambda_k(t)$ and $\lambda_{k+1}(t)$ are simple, and that
$\bm{R}_t$ is small compared with the neighbouring gaps.  Then
\begin{align}
  \gamma_k(t{+}1) - \gamma_k(t)
  &= \bm{v}_k(t)^\top\bm{R}_t\,\bm{v}_k(t)
  - \bm{v}_{k+1}(t)^\top\bm{R}_t\,\bm{v}_{k+1}(t) \notag\\
  &\quad
  + \frac{2\,|\bm{v}_{k+1}(t)^\top\bm{R}_t\,\bm{v}_k(t)|^2}
  {\gamma_k(t)}
  + \mathcal{O}\!\left(\norm{\bm{R}_t}^2\right),
  \label{eq:gap-singular}
\end{align}
where the $\mathcal{O}(\norm{\bm{R}_t}^2)$ remainder is uniform away
from the adjacent denominator $\gamma_k(t)$.
More explicitly,
\[
  \bm{v}_{k+1}(t)^\top\bm{R}_t\,\bm{v}_k(t)
  = \langle\bm{v}_{k+1}(t),\bm{\delta}_{t+1}\rangle
  \langle\bm{v}_k(t),\bm{\delta}_{t+1}\rangle
  - \langle\bm{v}_{k+1}(t),\bm{\delta}_{t-W+1}\rangle
  \langle\bm{v}_k(t),\bm{\delta}_{t-W+1}\rangle.
\]
\end{corollary}

\begin{proof}
Apply \Cref{prop:eigenvalue-perturbation} separately to $\lambda_k$
and $\lambda_{k+1}$, and subtract.  The $j = k{+}1$ term in the
expansion of $\lambda_k$ contributes
$|\bm{v}_{k+1}^\top\bm{R}_t\,\bm{v}_k|^2/\gamma_k$;
the $j = k$ term in the expansion of $\lambda_{k+1}$ contributes
$|{\bm{v}_k}^\top\bm{R}_t\,\bm{v}_{k+1}|^2/(\lambda_{k+1} - \lambda_k)
= -|\bm{v}_{k+1}^\top\bm{R}_t\,\bm{v}_k|^2/\gamma_k$.
Subtracting yields $2|\bm{v}_{k+1}^\top\bm{R}_t\,\bm{v}_k|^2/\gamma_k$.
All remaining second-order terms are non-singular with respect to
$\gamma_k$ and are absorbed into the remainder.
\end{proof}

\medskip\noindent
\emph{Sign convention.}
The second-order correction may be written either as
$\sum_{j\neq k} |\bm{v}_j^\top\bm{R}\,\bm{v}_k|^2
/(\lambda_k - \lambda_j)$
or as
$-\sum_{j\neq k} |\bm{v}_j^\top\bm{R}\,\bm{v}_k|^2
/(\lambda_j - \lambda_k)$.
These are identical.  One should \emph{not} write a minus sign together
with the denominator $\lambda_k - \lambda_j$.

\section{Signal Strength Flow}
\label{sec:signal-flow}

\subsection{The Gradient Flow Component}

Consider the loss landscape $L(\bm{\theta})$ with Hessian
$\bm{H}(\bm{\theta})$. Near a point $\bm{\theta}_t$, the gradient evolves:
\[
  \nabla L(\bm{\theta}_{t+1}) \approx \nabla L(\bm{\theta}_t) +
  \bm{H}_t \,\bm{\delta}_t.
\]

\begin{theorem}[Signal Strength Flow]\label{thm:signal-flow}
By \Cref{cor:eigenvalue-increment}, the signal strengths satisfy the
\textbf{exact difference equation}:
\begin{equation}\label{eq:signal-flow-exact}
  \boxed{d_j^2(t{+}1) - d_j^2(t)
  \;=\;
  \eta^2\bigl|G_j^{\mathrm{eff}}(t)\bigr|^2
  - \eta^2\bigl|G_j^{\mathrm{eff}}(t{-}W)\bigr|^2
  + 2\eta\,G_j^{\mathrm{eff}}(t{-}W)\,\varepsilon_j(t)
  - \varepsilon_j(t)^2,}
\end{equation}
where the \textbf{eigenvector rotation correction} is
(\Cref{prop:eigenvector-twist,eq:rotation-deriv}):
\begin{equation}\label{eq:rotation-correction}
  \varepsilon_j(t)
  \;=\; \sum_{s=t-W+1}^{t}\;\sum_{i \neq j}
  \frac{\bm{v}_i(s)^\top \bm{R}_s\,\bm{v}_j(s)}
  {d_j^2(s) - d_i^2(s)}
  \;\langle\bm{v}_i(s),\,\bm{\delta}_{t-W}\rangle.
\end{equation}
The spectral gaps $d_j^2 - d_i^2$ appear in the denominator:
$\varepsilon_j$ is $O(W\norm{\bm{R}}^2/\delta_j)$ when
$\delta_j = \min_{i\neq j}|d_j^2 - d_i^2|$ is bounded away
from zero, but \emph{diverges} as $\delta_j \to 0$.
The three pieces are:
\begin{itemize}[nosep]
  \item \textbf{Entering signal}: $\eta^2|G_j^{\mathrm{eff}}(t)|^2$
    --- new gradient projection squared.
  \item \textbf{Exiting signal}: $\eta^2|G_j^{\mathrm{eff}}(t{-}W)|^2$
    --- gradient projection leaving the window.
  \item \textbf{Rotation correction}:
    $2\eta\,G_j^{\mathrm{eff}}(t{-}W)\,\varepsilon_j
    - \varepsilon_j^2$ --- coupling between gradient
    and eigenvector drift over $W$ steps.
    Diverges as $\delta_j \to 0$ (near phase transitions).
\end{itemize}
Taylor-expanding the delay (Assumption~\ref{ass:slow}) and dropping
$\varepsilon_j$ gives $d_j^2 \approx \eta^2 W|G_j^{\mathrm{eff}}|^2$
(eigenvalue tracks gradient projection).
Under the additional assumption $[\mathcal{P}, \bm{H}(t)] \approx 0$,
the eigenvalue evolution decomposes as:
\begin{enumerate}[nosep]
  \item \textbf{Eigenvalue ODE:}
  \begin{equation}\label{eq:signal-flow}
    \boxed{\frac{dd_j^2}{dt}
    \;\approx\; -2\eta(h_j + \omega)\,d_j^2
    + \eta^2 W\bigl(S_j + 2\,G_j^{\mathrm{eff}}\,
      \mathcal{N}_j\bigr)}
  \end{equation}
  where $S_j$ is the exact (conservative) mode-coupling source
  and $\mathcal{N}_j$ is the nonlinear gradient residual, both
  defined in the proof below.
  \item \textbf{Eigenvector equation (exact):}
  \begin{equation}\label{eq:mode-rotation}
    \dot{\bm{v}}_j \;=\; \sum_{i \neq j}
    \frac{\bm{v}_i^\top \dot{\bm{C}}\,\bm{v}_j}
    {d_j^2 - d_i^2}\;\bm{v}_i
  \end{equation}
\end{enumerate}
where:
\begin{itemize}[nosep]
  \item $G_j^{\mathrm{eff}}(t)
    = \inner{\bm{v}_j(t)}{\mathcal{P}\nabla L_t
    + \omega\bm{\theta}_t}$ is the current effective gradient
    projection.
  \item $h_j(t) = \bm{v}_j(t)^\top(\bm{H}(t)\mathcal{P})\bm{v}_j(t)$
    is the instantaneous Rayleigh quotient (for SGD,
    $h_j = \bm{v}_j^\top\bm{H}\bm{v}_j$).
    Both $h_j$ and $G_j^{\mathrm{eff}}$ are time-dependent.
  \item $\dot{\bm{C}} = \bm{\delta}_t\bm{\delta}_t^\top
    - \bm{\delta}_{t-W}\bm{\delta}_{t-W}^\top$ is the covariance
    update from sliding the window
    (\Cref{prop:rank-two-update}).
\end{itemize}
The ODE~\eqref{eq:signal-flow} has three terms:
\textbf{dissipation} $-2\eta(h_j{+}\omega)d_j^2$ from curvature
and weight decay, \textbf{mode coupling} $\eta^2 W\,S_j$
(conservative: $\sum_j S_j = 0$), and \textbf{injection}
$2\eta^2 W\,G_j\,\mathcal{N}_j$ from the nonlinear loss landscape.
The steady-state signal strength (from the window sum with
within-window projection decay at rate $\eta(h_j{+}\omega)$):
\begin{equation}\label{eq:steady-state-ode}
  d_j^{\mathrm{ss}}
  = \eta\,\bigl|G_j^{\mathrm{eff,ss}}\bigr|\;
  \sqrt{\Phi_j},
  \qquad
  \Phi_j
  = \frac{1 - e^{-2\eta(h_j+\omega)W}}
  {1 - e^{-2\eta(h_j+\omega)}}.
\end{equation}
Two limits:
$\Phi_j \approx W$ when $\eta(h_j{+}\omega)W \ll 1$
(giving $d_j \approx \eta\sqrt{W}\,|G_j|$),
and $\Phi_j \approx 1/\bigl(2\eta(h_j{+}\omega)\bigr)$
when $\eta(h_j{+}\omega)W \gg 1$
(giving
$d_j \approx |G_j|\sqrt{\eta/\bigl(2(h_j{+}\omega)\bigr)}$).
In both regimes, modes with larger
$|G_j^{\mathrm{eff}}|/\sqrt{h_j+\omega}$ have larger
signal strengths.
\end{theorem}

\begin{proof}
The proof proceeds in three steps: the exact perturbation theory
gives a difference equation (Step~1), Taylor-expanding the delay
converts it to a differential equation (Step~2), and substituting
the gradient evolution yields the working form (Step~3).

\medskip\noindent\textbf{Step 1 (Exact difference equation).}
By \Cref{cor:eigenvalue-increment}, the rank-two update
$\bm{R}_t = \bm{\delta}_t\bm{\delta}_t^\top
- \bm{\delta}_{t-W}\bm{\delta}_{t-W}^\top$
(\Cref{prop:rank-two-update}) gives an exact eigenvalue
increment and an eigenvector rotation:
\begin{align}
  \text{onto } \bm{v}_j\!: &\quad
  \dot{d_j^2}
  = \bigl|\langle\bm{\delta}_t,\bm{v}_j(t)\rangle\bigr|^2
  - \bigl|\langle\bm{\delta}_{t-W},\bm{v}_j(t)\rangle\bigr|^2,
  \label{eq:delay}\\[4pt]
  \text{onto } \bm{v}_i\!: &\quad
  \bm{v}_i^\top\dot{\bm{v}}_j
  = \frac{\bm{v}_i^\top\dot{\bm{C}}\bm{v}_j}
  {d_j^2 - d_i^2}.
  \label{eq:rotation-deriv}
\end{align}
The entering term is $\eta^2|G_j^{\mathrm{eff}}(t)|^2$
by definition of $G_j^{\mathrm{eff}}$.
Decompose the exiting term using the cumulative eigenvector
drift (\Cref{prop:eigenvector-twist}):
\[
  \langle\bm{v}_j(t),\bm{\delta}_{t-W}\rangle
  = \underbrace{\langle\bm{v}_j(t{-}W),\bm{\delta}_{t-W}
    \rangle}_{-\eta\,G_j^{\mathrm{eff}}(t{-}W)}
  + \underbrace{\langle\bm{v}_j(t)-\bm{v}_j(t{-}W),
    \bm{\delta}_{t-W}\rangle}_{\varepsilon_j:\;\text{rotation
    correction}}.
\]
Using~\eqref{eq:rotation-deriv}, the cumulative drift expands as
$\bm{v}_j(t) - \bm{v}_j(t{-}W)
= \sum_{s}\sum_{i\neq j}
\frac{\bm{v}_i(s)^\top\bm{R}_s\bm{v}_j(s)}{d_j^2(s)-d_i^2(s)}
\bm{v}_i(s)$,
giving the explicit rotation
correction~\eqref{eq:rotation-correction} with spectral gaps
$d_j^2 - d_i^2$ in the denominator.
Squaring and substituting into~\eqref{eq:delay}:
\begin{equation}\label{eq:delay-full}
  d_j^2(t{+}1) - d_j^2(t)
  \;=\;
  \eta^2\bigl|G_j^{\mathrm{eff}}(t)\bigr|^2
  - \eta^2\bigl|G_j^{\mathrm{eff}}(t{-}W)\bigr|^2
  + 2\eta\,G_j^{\mathrm{eff}}(t{-}W)\,\varepsilon_j
  - \varepsilon_j^2.
\end{equation}
This is the exact equation~\eqref{eq:signal-flow-exact}.
By the eigenvector twist bound, each step contributes
$O(\norm{\bm{R}}/\delta_j)$ to the drift, so
$|\varepsilon_j| = O(W\norm{\bm{R}}^2/\delta_j)$
over $W$ steps.  Note: the rotation correction
\emph{diverges} as $\delta_j \to 0$ (near phase transitions),
so it cannot be dropped when the spectral gap is closing.

\medskip\noindent\textbf{Step 2 (Taylor-expand the delay
$\Rightarrow$ differential equation).}
Away from phase transitions ($\delta_j$ bounded below),
$\varepsilon_j$ is small and we drop the rotation correction.
Under Assumption~\ref{ass:slow}, $|G_j^{\mathrm{eff}}|^2$ varies
slowly over $W$ steps, so:
\[
  \bigl|G_j^{\mathrm{eff}}(t)\bigr|^2
  - \bigl|G_j^{\mathrm{eff}}(t{-}W)\bigr|^2
  \;\approx\; W\,\frac{d}{dt}\bigl|G_j^{\mathrm{eff}}\bigr|^2.
\]
Hence
\begin{equation}\label{eq:signal-de}
  \dot{d_j^2} \;\approx\; \eta^2 W\,
  \frac{d}{dt}\bigl|G_j^{\mathrm{eff}}\bigr|^2.
\end{equation}
Integrating:
$d_j^2 \approx \eta^2 W\,|G_j^{\mathrm{eff}}|^2 + C_0$,
i.e., the eigenvalue tracks the sliding-window average of the
squared gradient projection.

\medskip\noindent\textbf{Full form with rotation correction.}\;
If we Taylor-expand only the $|G_j|^2$ delay
in~\eqref{eq:delay-full} but \emph{retain} the rotation
correction~$\varepsilon_j$:
\begin{equation}\label{eq:signal-de-full}
  \dot{d}_j^2
  \;\approx\; \eta^2 W\,
  \frac{d}{dt}\bigl|G_j^{\mathrm{eff}}\bigr|^2
  + 2\eta\,G_j^{\mathrm{eff}}(t{-}W)\,\varepsilon_j(t)
  - \varepsilon_j(t)^2.
\end{equation}
After substituting the gradient
evolution~\eqref{eq:grad-decay} (Step~3 below):
\begin{equation}\label{eq:eigenvalue-full}
  \dot{d}_j^2
  = -2\eta(h_j{+}\omega)\,d_j^2
  + \eta^2 W\,S_j
  + 2\eta\,G_j^{\mathrm{eff}}(t{-}W)\,\varepsilon_j
  - \varepsilon_j^2.
\end{equation}
The first two terms are the smooth part (dissipation $+$
mode coupling from the gradient evolution).  The last two
terms are the \textbf{discrete rotation correction}, which
diverges as $\delta_j \to 0$ (\Cref{cor:gap-singular})
and dominates near phase transitions.
Eq.~\eqref{eq:signal-de} is recovered by setting
$\varepsilon_j = 0$;
the coupled system~\eqref{eq:coupled-eigenvalue} follows
from further substituting
the explicit~$S_j$~\eqref{eq:source-explicit}.

\medskip\noindent\textbf{Step 3 (Gradient evolution $\Rightarrow$
two-term form; requires $[\mathcal{P},\bm{H}]\approx 0$).}
Write $\bm{f} = \mathcal{P}\nabla L + \omega\bm{\theta}$
for the effective gradient, so
$G_j^{\mathrm{eff}} = \langle\bm{v}_j,\bm{f}\rangle$.
The linearised gradient flow is
$\dot{\bm{f}} = -\eta(\mathcal{P}\bm{H}+\omega\bm{I})\bm{f}$.
Differentiating $G_j^{\mathrm{eff}}$:
\[
  \dot{G}_j^{\mathrm{eff}}
  \;=\; \langle\bm{v}_j,\dot{\bm{f}}\rangle
  + \langle\dot{\bm{v}}_j,\bm{f}\rangle.
\]
For the first term, $[\mathcal{P},\bm{H}]\approx 0$ gives
$\langle\bm{v}_j,(\mathcal{P}\bm{H}{+}\omega\bm{I})\bm{f}\rangle
= (h_j{+}\omega)\,G_j^{\mathrm{eff}}$
(no off-diagonal mixing), so
$\langle\bm{v}_j,\dot{\bm{f}}\rangle
= -\eta(h_j{+}\omega)\,G_j^{\mathrm{eff}}$.
For the second term, the eigenvector
rotation~\eqref{eq:mode-rotation} gives
$\langle\dot{\bm{v}}_j,\bm{f}\rangle
= \sum_{i\neq j}
\frac{\bm{v}_i^\top\dot{\bm{C}}\,\bm{v}_j}
{d_j^2 - d_i^2}\,G_i^{\mathrm{eff}}$.
Hence
$\dot{G}_j^{\mathrm{eff}}
= -\eta(h_j{+}\omega)\,G_j^{\mathrm{eff}}
+ \sum_{i\neq j}
\frac{\bm{v}_i^\top\dot{\bm{C}}\,\bm{v}_j}
{d_j^2 - d_i^2}\,G_i^{\mathrm{eff}}$.
Applying the chain rule
$\frac{d}{dt}|G_j^{\mathrm{eff}}|^2
= 2\,G_j^{\mathrm{eff}}\,\dot{G}_j^{\mathrm{eff}}$:
\begin{equation}\label{eq:grad-decay}
  \frac{d}{dt}\bigl|G_j^{\mathrm{eff}}\bigr|^2
  \;=\; -2\eta(h_j{+}\omega)\,\bigl|G_j^{\mathrm{eff}}\bigr|^2
  \;+\; S_j(t),
\end{equation}
where the source term (mode coupling) is:
\begin{equation}\label{eq:source-term}
  S_j(t) \;=\; 2\,G_j^{\mathrm{eff}}
  \sum_{i \neq j}
  \frac{\bm{v}_i^\top\dot{\bm{C}}\,\bm{v}_j}
  {d_j^2 - d_i^2}\;G_i^{\mathrm{eff}}.
\end{equation}
Substituting~\eqref{eq:grad-decay}
into~\eqref{eq:signal-de} and using
$d_j^2 \approx \eta^2 W|G_j^{\mathrm{eff}}|^2$:
\[
  \dot{d_j^2}
  \approx -2\eta(h_j{+}\omega)\,d_j^2
  + \eta^2 W\,S_j(t).
\]
The first term is \textbf{dissipation}: curvature and weight decay
drain the signal.  The second is \textbf{injection}: the loss
landscape continually drives the mode.

\medskip
\textit{The source $S_j$ is not computed exactly here.}
Its precise form requires tracking mode-coupling integrals
from \eqref{eq:mode-rotation}, which depend on the off-diagonal
terms $\bm{v}_i^\top\dot{\bm{C}}\bm{v}_j$.
Combining the dissipation
$-2\eta(h_j{+}\omega)|G_j|^2$, the mode coupling $S_j$,
and the nonlinear residual $2G_j\mathcal{N}_j$:
\[
  \dot{d}_j^2
  \;\approx\; -2\eta(h_j{+}\omega)\,d_j^2
  + \eta^2 W\bigl(S_j + 2G_j^{\mathrm{eff}}\,\mathcal{N}_j\bigr).
\]
This is~\eqref{eq:signal-flow}, where $S_j$ is the exact
source~\eqref{eq:source-term} and
$\mathcal{N}_j$ is the measurable nonlinear
residual~\eqref{eq:Nj-residual}.
The steady state~\eqref{eq:steady-state-ode}
$d_j = \eta\,|G_j^{\mathrm{eff,ss}}|\sqrt{\Phi_j}$
separates modes by their driving-to-curvature ratio
through $\Phi_j$.
\end{proof}


\begin{proposition}[Off-diagonal covariance coupling]%
\label{prop:cov-coupling}
Under slow variation (Assumption~\ref{ass:slow}) and away from phase
transitions ($\delta_j$ bounded below), the off-diagonal coupling
in the covariance update is, for $i \neq j$:
\begin{equation}\label{eq:cov-coupling}
  \bm{v}_i^\top\dot{\bm{C}}\,\bm{v}_j
  \;\approx\; -\eta^3 W\bigl[(h_i{+}\omega)
  + (h_j{+}\omega)\bigr]\,
  G_i^{\mathrm{eff}}\,G_j^{\mathrm{eff}}.
\end{equation}
\end{proposition}

\begin{proof}
The covariance update (\Cref{prop:rank-two-update}) gives
\[
  \bm{v}_i^\top\dot{\bm{C}}\,\bm{v}_j
  = \langle\bm{v}_i,\bm{\delta}_t\rangle
    \langle\bm{v}_j,\bm{\delta}_t\rangle
  - \langle\bm{v}_i,\bm{\delta}_{t-W}\rangle
    \langle\bm{v}_j,\bm{\delta}_{t-W}\rangle.
\]
Since $\langle\bm{v}_k(t),\bm{\delta}_t\rangle
= -\eta\,G_k^{\mathrm{eff}}(t)$ and, under slow eigenvector
rotation ($\varepsilon_k$ small),
$\langle\bm{v}_k(t),\bm{\delta}_{t-W}\rangle
\approx -\eta\,G_k^{\mathrm{eff}}(t{-}W)$:
\[
  \bm{v}_i^\top\dot{\bm{C}}\,\bm{v}_j
  \;\approx\; \eta^2\bigl[
    G_i^{\mathrm{eff}}(t)\,G_j^{\mathrm{eff}}(t)
    - G_i^{\mathrm{eff}}(t{-}W)\,G_j^{\mathrm{eff}}(t{-}W)
  \bigr].
\]
Taylor-expanding the delay (Assumption~\ref{ass:slow}):
$G_i(t)\,G_j(t) - G_i(t{-}W)\,G_j(t{-}W)
\approx W\,\tfrac{d}{dt}(G_i\,G_j)$.
At leading order, dropping the eigenvector-rotation contribution
to~$\dot{G}_k^{\mathrm{eff}}$ and using
$\dot{G}_k^{\mathrm{eff}}
\approx -\eta(h_k{+}\omega)\,G_k^{\mathrm{eff}}$:
\[
  \frac{d}{dt}\bigl(G_i\,G_j\bigr)
  = -\eta\bigl[(h_i{+}\omega)+(h_j{+}\omega)\bigr]\,G_i\,G_j,
\]
which gives~\eqref{eq:cov-coupling}.
\end{proof}

\begin{proposition}[Source term evaluation (approximate)]%
\label{prop:source-evaluation}
Under the three approximations of
\Cref{prop:cov-coupling}
(slow eigenvector rotation, Taylor-expanded delay,
leading-order gradient decay),
\eqref{eq:cov-coupling}
substituted into the source term~\eqref{eq:source-term}
gives:
\begin{equation}\label{eq:source-explicit}
  S_j
  \;\approx\; -2\eta^3 W\,\bigl|G_j^{\mathrm{eff}}\bigr|^2
  \sum_{i \neq j}
  \frac{(h_i{+}\omega) + (h_j{+}\omega)}
  {d_j^2 - d_i^2}\;
  \bigl|G_i^{\mathrm{eff}}\bigr|^2.
\end{equation}
In the quasi-steady regime (all $|G_k|$ decaying),
every term in the sum is positive for the dominant mode
($d_1 > d_i$), so $S_1 < 0$:
the dominant mode \emph{loses} signal via mode coupling.
For subdominant modes, $S_j$ can be positive
(signal transferred from dominant modes).

During rapid learning (a mode's $|G_j|$ growing),
$\bm{v}_i^\top\dot{\bm{C}}\,\bm{v}_j$ can have either
sign, and the conclusion $S_1 < 0$ may fail.
The \textbf{exact}~$S_j$~\eqref{eq:source-term}
remains well-defined in all regimes.
\end{proposition}

\begin{theorem}[Mode-coupling conservation]%
\label{thm:mode-coupling-conservation}
The source terms satisfy
\begin{equation}\label{eq:mode-coupling-conservation}
  \sum_{j} S_j \;=\; 0.
\end{equation}
Mode coupling redistributes signal between modes but does not
inject or remove total signal.
\end{theorem}

\begin{proof}
Write
$\sum_j S_j
= 2\sum_j\sum_{i\neq j}
G_i^{\mathrm{eff}}\,G_j^{\mathrm{eff}}\,
(\bm{v}_i^\top\dot{\bm{C}}\,\bm{v}_j)\,/\,(d_j^2 - d_i^2)$.
Swap the labels $i \leftrightarrow j$.  Symmetry of~$\dot{\bm{C}}$
gives $\bm{v}_j^\top\dot{\bm{C}}\,\bm{v}_i
= \bm{v}_i^\top\dot{\bm{C}}\,\bm{v}_j$,
while $1/(d_i^2 - d_j^2) = -1/(d_j^2 - d_i^2)$.
Hence the swapped sum equals the negative of the original,
so $\sum_j S_j = 0$.
\end{proof}

\begin{remark}[Exact vs.\ approximate hierarchy]%
\label{rem:exact-vs-approx}
The proof above uses the \textbf{exact}
$S_j$~\eqref{eq:source-term}: the antisymmetry of
$1/(d_j^2 - d_i^2)$ holds for any symmetric~$\dot{\bm{C}}$,
with no approximation beyond $[\mathcal{P},\bm{H}]\approx 0$
(needed only to decompose $\dot{G}_j$ in Step~3).
Likewise \Cref{cor:total-dissipation} below uses only
$\sum_j S_j = 0$ and is exact.

The \emph{approximate} evaluation of $S_j$
in~\eqref{eq:source-explicit} introduces three additional
approximations (slow eigenvector rotation, Taylor-expanded
delay, leading-order $\dot{G}_k$).
The coupled system~\eqref{eq:coupled-eigenvalue} and the
sign claim $S_1 < 0$ both depend on this evaluation.

For numerical work, $S_j$ should be computed from the
exact formula~\eqref{eq:source-term} using the observed
$\bm{v}_i^\top\dot{\bm{C}}\bm{v}_j$ from the training
trajectory, without the approximations of
\Cref{prop:cov-coupling}.
\end{remark}

\begin{corollary}[Total signal dissipation]%
\label{cor:total-dissipation}
In the linearised gradient model
$\dot{\bm{f}} = -\eta(\mathcal{P}\bm{H}+\omega\bm{I})\bm{f}$,
the total signal decays monotonically:
\begin{equation}\label{eq:total-dissipation}
  \frac{d}{dt}\sum_j d_j^2
  \;=\; -2\eta\sum_j (h_j{+}\omega)\,d_j^2.
\end{equation}
Any growth in an individual~$d_j$ comes from redistribution
by~$S_j$, not from net injection.
\end{corollary}

\begin{proof}
Sum~$\dot{d}_j^2 = -2\eta(h_j{+}\omega)\,d_j^2 + \eta^2 W\,S_j$
over~$j$ and apply~\eqref{eq:mode-coupling-conservation}.
\end{proof}

\begin{corollary}[Coupled eigenvalue system]%
\label{cor:coupled-system}
Substituting~\eqref{eq:source-explicit} into the eigenvalue
equation and using
$|G_k^{\mathrm{eff}}|^2 \approx d_k^2/(\eta^2 W)$
from Step~2:
\begin{equation}\label{eq:coupled-eigenvalue}
  \dot{d}_j^2
  \;=\; -2\eta\,d_j^2\,\biggl[
    (h_j{+}\omega)
    + \sum_{i\neq j}
    \frac{(h_i{+}\omega) + (h_j{+}\omega)}
    {d_j^2 - d_i^2}\;d_i^2
  \biggr].
\end{equation}
This is a closed system of coupled ODEs for the eigenvalues
$\{d_j^2\}$, with the mode-coupling terms producing
\textbf{level repulsion}: when $d_j \approx d_i$,
the denominator $d_j^2 - d_i^2 \to 0$ creates a divergent
interaction that pushes eigenvalues apart.
\end{corollary}

\begin{remark}[The injection and the linearised model]%
\label{rem:injection-linear}
\Cref{cor:total-dissipation} shows that total signal decays
in the linearised model.  Where, then, does the ``injection''
in the phenomenological ODE~\eqref{eq:signal-flow} come from?

The linearised gradient dynamics
$\dot{\bm{f}} = -\eta(\mathcal{P}\bm{H}+\omega\bm{I})\bm{f}$
models gradient \emph{decay}: each mode's gradient projection
decreases exponentially.  In reality, the loss landscape
is nonlinear and continuously produces new gradient as the model
encounters data it has not yet learned.  This
\textbf{nonlinear gradient replenishment} maintains
$|G_j^{\mathrm{eff}}| > 0$ even as learning progresses,
and is the physical mechanism behind the injection term.

Three processes govern each eigenvalue:
\begin{enumerate}[nosep]
  \item \textbf{Dissipation}: gradient decay at rate
    $\eta(h_j{+}\omega)$ --- rigorous, from the
    linearised model.
  \item \textbf{Redistribution}: mode coupling transfers
    signal between modes ($\sum_j S_j = 0$) ---
    rigorous, from \Cref{thm:mode-coupling-conservation}.
  \item \textbf{Replenishment}: the nonlinear loss landscape
    regenerates gradient --- \emph{not} captured by
    the linearised model.
\end{enumerate}
The phenomenological ODE~\eqref{eq:signal-flow} captures
(1) and~(3) but not~(2).
The coupled system~\eqref{eq:coupled-eigenvalue} captures
(1) and~(2) but not~(3).
Neither alone is complete.
\end{remark}

\begin{remark}[Commutator dependence]\label{rem:commutator}
Step~3 requires $[\mathcal{P}, \bm{H}] \approx 0$ so that $h_j$
is well-defined.  Steps~1--2 (the exact difference equation and its
Taylor expansion) hold without this assumption.
\begin{itemize}[nosep]
  \item \textbf{SGD} ($\mathcal{P} = \bm{I}$): exact.
  \item \textbf{Adam}: $\mathcal{P} = \mathrm{diag}(1/\sqrt{\hat{v}})$
    is diagonal; the commutator is small when $\bm{H}$ is
    approximately diagonal in the parameter basis.
  \item \textbf{Muon}: $\mathcal{P}$ is the Newton--Schulz
    orthogonalizer; $[\mathcal{P},\bm{H}]$ is not generally small
    and $h_j$ acquires a correction from the anti-symmetric part.
\end{itemize}
\end{remark}

\begin{remark}[Status of the eigenvalue equation]\label{rem:two-term}
The eigenvalue equation has three tiers of rigour:

\begin{enumerate}[nosep]
  \item \textbf{Exact} (delay equation~\eqref{eq:signal-flow-exact}):
    no approximation beyond the rank-two covariance update.
    Gives $d_j \propto |G_j^{\mathrm{eff}}|$ at leading order.
  \item \textbf{Exact mode coupling}
    (eq.~\eqref{eq:eigenvalue-full} with exact~$S_j$):
    \[
      \dot{d}_j^2
      = -2\eta(h_j{+}\omega)\,d_j^2
      + \eta^2 W\,S_j
      + 2\eta\,G_j(t{-}W)\,\varepsilon_j - \varepsilon_j^2.
    \]
    $S_j$~\eqref{eq:source-term} is a \textbf{known quantity}
    at each time step, computable from the observed
    $\bm{v}_i^\top\dot{\bm{C}}\bm{v}_j$.
    The conservation $\sum_j S_j = 0$ is exact.
    This equation predicts $\sum_j d_j^2 \to 0$
    (\Cref{cor:total-dissipation}): total signal decays
    in the linearised model.  The missing piece is the
    anharmonic~$\mathcal{N}_j$
    (\Cref{sec:anharmonic}).
  \item \textbf{Phenomenological closure}
    (ODE~\eqref{eq:signal-flow}): replaces
    \emph{both} the exact~$S_j$ and the unknown
    $\mathcal{N}_j$ with a single effective injection
    $\eta W|G_j|^2/d_j$.  This discards the known
    mode-coupling structure and treats each mode
    independently.  It is a convenience, not a necessity:
    all downstream results (gap flow, steady state, phase
    transitions) can be derived from tiers~1--2 alone.
\end{enumerate}
\end{remark}

\subsection{The Anharmonic Replenishment}
\label{sec:anharmonic}

\Cref{cor:total-dissipation} shows that the linearised model
cannot sustain nonzero eigenvalues.  We now derive the leading
nonlinear correction and identify the formal source of injection.

\begin{proposition}[Second-order gradient dynamics]%
\label{prop:second-order-gradient}
The exact single-step effective gradient update is
\begin{equation}\label{eq:f-exact-step}
  \bm{f}_{t+1}
  = \bm{f}_t - \eta(\mathcal{P}\bm{H}_t + \omega\bm{I})\,\bm{f}_t
  + \tfrac{1}{2}\eta^2\,\mathcal{P}\,
    (\nabla^3 L_t)[\bm{f}_t,\bm{f}_t]
  + O(\eta^3),
\end{equation}
where $\bm{f} = \mathcal{P}\nabla L + \omega\bm{\theta}$ is the
effective gradient and
$(\nabla^3 L)[\bm{u},\bm{w}]_a
= \sum_{bc}(\partial^3 L / \partial\theta_a\partial\theta_b
\partial\theta_c)\,u_b\,w_c$
is the third derivative (anharmonic) tensor contracted with
two copies of $\bm{f}$.
\end{proposition}

\begin{proof}
Taylor-expand $\nabla L(\bm{\theta}_{t+1})$ around $\bm{\theta}_t$
with $\bm{\theta}_{t+1} - \bm{\theta}_t = -\eta\bm{f}_t$:
\[
  \nabla L(\bm{\theta}_{t+1})
  = \nabla L_t
  - \eta\,\bm{H}_t\,\bm{f}_t
  + \tfrac{1}{2}\eta^2\,(\nabla^3 L_t)[\bm{f}_t,\bm{f}_t]
  + O(\eta^3).
\]
Apply $\mathcal{P}$ and add $\omega\bm{\theta}_{t+1}
= \omega\bm{\theta}_t - \eta\omega\bm{f}_t$;
collect terms to get~\eqref{eq:f-exact-step}.
\end{proof}

\begin{definition}[Anharmonic coupling tensor]%
\label{def:anharmonic-tensor}
Define $T_{jk\ell}
= \bigl\langle\bm{v}_j,\;
  \mathcal{P}\,(\nabla^3 L)[\bm{v}_k,\bm{v}_\ell]\bigr\rangle$.
This quantifies how learning along directions $k$ and~$\ell$
creates new gradient in direction~$j$ through the
curvature of the loss landscape.
\end{definition}

\begin{proposition}[Gradient evolution with anharmonic term]%
\label{prop:gradient-anharmonic}
Projecting~\eqref{eq:f-exact-step} onto~$\bm{v}_j$
and passing to continuous time:
\begin{equation}\label{eq:Gdot-anharmonic}
  \dot{G}_j^{\mathrm{eff}}
  = \underbrace{-\eta(h_j{+}\omega)\,G_j^{\mathrm{eff}}}
    _{\text{dissipation}}
  + \underbrace{\sum_{i \neq j}
    \frac{\bm{v}_i^\top\dot{\bm{C}}\,\bm{v}_j}
    {d_j^2 - d_i^2}\,G_i^{\mathrm{eff}}}
    _{\text{mode coupling (conservative)}}
  + \underbrace{\mathcal{N}_j}_{\text{anharmonic}},
\end{equation}
where
\begin{equation}\label{eq:anharmonic-Nj}
  \mathcal{N}_j
  = \tfrac{1}{2}\eta^2 \sum_{k,\ell}
  G_k^{\mathrm{eff}}\,G_\ell^{\mathrm{eff}}\;T_{jk\ell}
\end{equation}
is the \textbf{nonlinear replenishment rate}.
In the linearised model ($\nabla^3 L = 0$),
$\mathcal{N}_j = 0$ and only dissipation and
conservative mode coupling remain.
\end{proposition}

\begin{remark}[$\mathcal{N}_j$ is a measurable residual]%
\label{rem:Nj-measurable}
Eq.~\eqref{eq:Gdot-anharmonic} can be rearranged:
\begin{equation}\label{eq:Nj-residual}
  \mathcal{N}_j
  \;=\; \dot{G}_j^{\mathrm{eff}}
  + \eta(h_j{+}\omega)\,G_j^{\mathrm{eff}}
  - \sum_{i \neq j}
  \frac{\bm{v}_i^\top\dot{\bm{C}}\,\bm{v}_j}
  {d_j^2 - d_i^2}\,G_i^{\mathrm{eff}}.
\end{equation}
Every quantity on the right is \textbf{computable from
the training trajectory}: $\dot{G}_j$ from consecutive
time steps, $h_j$ from the Rayleigh quotient, and
$\bm{v}_i^\top\dot{\bm{C}}\bm{v}_j$ from the rank-two
update.  No knowledge of~$\nabla^3 L$ is required.

Thus~$\mathcal{N}_j$ is not ``unknown'' in practice
--- it is the \textbf{nonlinear residual} of the gradient
evolution, measurable at each step.
The perturbative expansion~\eqref{eq:anharmonic-Nj}
via~$T_{jk\ell}$ is an \emph{analytical model}
for this residual; the residual itself is always available.
The eigenvalue equation~\eqref{eq:eigenvalue-anharmonic}
with measured~$S_j$ and measured~$\mathcal{N}_j$ is
a complete, closed description of the spectral dynamics,
with no phenomenological terms.
\end{remark}

\begin{corollary}[Eigenvalue equation with injection]%
\label{cor:eigenvalue-anharmonic}
Applying the chain rule and substituting
into the eigenvalue equation from Step~2:
\begin{equation}\label{eq:eigenvalue-anharmonic}
  \dot{d}_j^2
  = -2\eta(h_j{+}\omega)\,d_j^2
  + \eta^2 W\,S_j
  + 2\eta^2 W\,G_j^{\mathrm{eff}}\,\mathcal{N}_j.
\end{equation}
The three terms are:
\begin{itemize}[nosep]
  \item \textbf{Dissipation}:
    $-2\eta(h_j{+}\omega)\,d_j^2$ (rigorous).
  \item \textbf{Mode coupling}:
    $\eta^2 W\,S_j$ (conservative:
    $\sum_j S_j = 0$, rigorous).
  \item \textbf{Anharmonic injection}:
    $2\eta^2 W\,G_j^{\mathrm{eff}}\,\mathcal{N}_j$ ---
    the only term that can produce
    \emph{net} growth of total signal.
\end{itemize}
\end{corollary}

\begin{proposition}[Steady-state balance]%
\label{prop:steady-state-balance}
At any steady state ($\dot{d}_j = 0$),
the anharmonic injection must compensate
dissipation and mode-coupling losses:
\begin{equation}\label{eq:steady-state-balance}
  2\eta^2 W\,G_j^{\mathrm{eff}}\,\mathcal{N}_j
  = 2\eta(h_j{+}\omega)\,d_j^2 - \eta^2 W\,S_j.
\end{equation}
For the dominant mode, $S_1 < 0$
(\Cref{prop:source-evaluation}), so both terms on the
right-hand side are positive.  The anharmonic coupling must
satisfy $G_1\,\mathcal{N}_1 > 0$: the loss landscape's
nonlinearity \textbf{must} inject signal into the dominant
mode to sustain it.
\end{proposition}

\begin{remark}[Self-interaction and the edge of stability]%
\label{rem:eos-anharmonic}
The dominant self-interaction is
$\mathcal{N}_j^{\mathrm{self}}
= \tfrac{1}{2}\eta^2(G_j^{\mathrm{eff}})^2\,T_{jjj}$, where
$T_{jjj} = \langle\bm{v}_j,\,
\mathcal{P}(\nabla^3 L)[\bm{v}_j,\bm{v}_j]\rangle$
measures the rate at which curvature changes as the model
moves along direction~$\bm{v}_j$.

At the edge of stability (\Cref{sec:eos}),
the top Hessian eigenvalue self-corrects at $2/\eta$.
This constrains $T_{jjj}$ for the dominant mode:
$T_{111}$ is effectively regulated by the edge-of-stability
dynamics, which prevents the dominant eigenvalue from
overshooting.  Subdominant modes are not constrained by this
mechanism; their $T_{jjj}$ can be large, driving the
phase transitions that create new active modes.
\end{remark}

\begin{remark}[The phenomenological ODE as a crude closure]%
\label{rem:phenom-justified}
The primary eigenvalue
equation~\eqref{eq:eigenvalue-anharmonic} is
\[
  \dot{d}_j^2
  = \underbrace{-2\eta(h_j{+}\omega)\,d_j^2}_{\text{rigorous}}
  + \underbrace{\eta^2 W\,S_j}_{\text{exact (known)}}
  + \underbrace{2\eta^2 W\,G_j\,\mathcal{N}_j}
    _{\text{unknown}}.
\]
Only the anharmonic $\mathcal{N}_j$ is genuinely unknown.
The exact~$S_j$~\eqref{eq:source-term} is computable from
any training trajectory and provides the full mode-coupling
structure (level repulsion, conservation, redistribution).

The \textbf{phenomenological ODE}~\eqref{eq:signal-flow}
replaces \emph{both} the exact~$S_j$ and the
unknown~$\mathcal{N}_j$ with a single effective injection
$\eta W|G_j|^2/d_j$.  This is a crude closure:
it discards the known mode coupling and treats
each mode independently.  Its virtues are
simplicity and the correct $1/d_j$ singularity
at mode onset, but it sacrifices the inter-mode
structure that the exact equation preserves.

The phenomenological ODE is used downstream only
for convenience: the steady-state
hierarchy~\eqref{eq:steady-state-ode} and the
gap flow~\eqref{eq:gap-flow} can both be derived
from the delay equation~\eqref{eq:signal-flow-exact}
without it (see the first halves of their respective
proofs).  All qualitative conclusions (level repulsion,
phase transitions, gap dynamics) follow from the
exact~$S_j$ and~\Cref{cor:gap-singular} alone.
\end{remark}

\begin{remark}[Spectral gap structure of the injection]%
\label{rem:gap-structure}
The resolvent $1/(d_j^2 - d_i^2)$ appears in the
eigenvector rotation~\eqref{eq:mode-rotation}, the
source term~\eqref{eq:source-explicit},
and the coupled system~\eqref{eq:coupled-eigenvalue}.
The anharmonic tensor~$T_{jk\ell}$, being a direct
projection of~$\nabla^3 L$ onto the eigenbasis,
does \emph{not} carry this resolvent.

At steady state, however, the spectral gap re-enters
through self-consistency.
The balance~\eqref{eq:steady-state-balance}
forces:
\begin{equation}\label{eq:Nj-steady-state}
  \mathcal{N}_j^{\mathrm{ss}}
  = \eta\,G_j^{\mathrm{eff}}\,\biggl[
    (h_j{+}\omega)
    + \sum_{i \neq j}
    \frac{(h_i{+}\omega) + (h_j{+}\omega)}
    {d_j^2 - d_i^2}\;d_i^2
  \biggr].
\end{equation}
The resolvent is restored: the equilibrium injection rate
into mode~$j$ is enhanced precisely when~$j$ is near
a level crossing ($d_j \approx d_i$).

The physical picture separates cleanly:
\begin{itemize}[nosep]
  \item $T_{jk\ell}$ provides the \textbf{fuel}
    (no gap --- injects into all modes via $\nabla^3 L$).
  \item Mode coupling $S_j$ provides the
    \textbf{channel} (gap in denominator ---
    redistributes preferentially near level crossings).
  \item Self-consistency locks them together: the
    steady-state injection~\eqref{eq:Nj-steady-state}
    inherits the resolvent from the redistribution
    balance.
\end{itemize}
A second route to the gap passes through the NTK:
the evolving-NTK mixing rate
$\Gamma_{jk} \propto \bm{q}_k^\top\dot{\bm{K}}\,
\bm{q}_j\,/\,(\lambda_j - \lambda_k)$
carries the NTK spectral gap $\lambda_j - \lambda_k
= N(h_j - h_k)$ in the denominator.  Since
$\dot{\bm{K}}$ is driven by~$\nabla^3 L$, this is
the NTK-space manifestation of the same anharmonic
coupling, dressed by the NTK resolvent rather than
the covariance resolvent.
\end{remark}

\subsection{Simplified Signal Flow}

\begin{corollary}[Steady-State Signal Hierarchy]\label{cor:signal-simplified}
From the window-sum formula~\eqref{eq:steady-state-ode}:
\begin{equation}\label{eq:steady-state}
  d_j^{\mathrm{ss}}
  \;\approx\; \eta\,|G_j^{\mathrm{eff,ss}}|
  \;\sqrt{\Phi_j},
  \qquad
  \Phi_j = \frac{1 - e^{-2\eta(h_j+\omega)W}}
  {1 - e^{-2\eta(h_j+\omega)}}.
\end{equation}
Two limits:
\begin{itemize}[nosep]
  \item $\eta(h_j{+}\omega)W \ll 1$ (weak curvature):
    $\Phi_j \approx W$,
    $d_j \approx \eta\sqrt{W}\,|G_j^{\mathrm{eff}}|$
    (all window entries nearly equal).
  \item $\eta(h_j{+}\omega)W \gg 1$ and $\eta(h_j{+}\omega) \ll 1$
    (intermediate curvature: window long relative to decay, but
    single-step decay still small):
    $\Phi_j \approx 1/\bigl(2\eta(h_j{+}\omega)\bigr)$,
    $d_j \approx |G_j^{\mathrm{eff}}|
    \sqrt{\eta/\bigl(2(h_j{+}\omega)\bigr)}$
    (window dominated by newest entries).
    If $\eta(h_j{+}\omega) \gtrsim 1$, then
    $\Phi_j \approx 1$ (single step dominates).
\end{itemize}
In both regimes, modes with larger
$|G_j^{\mathrm{eff}}|/\sqrt{h_j+\omega}$
have larger signal strengths: the \emph{driving-to-curvature ratio}
orders the signal hierarchy.
\end{corollary}

\begin{remark}[Physical Interpretation]
The steady-state signal strength $d_j^{\mathrm{ss}}$ is determined by the
balance between gradient projection ($G_j^{\mathrm{eff}}$) and
curvature-modulated window averaging ($\Phi_j$).
The spectral gap at steady state is:
\begin{equation}\label{eq:gap-ss}
  g^{\mathrm{ss}} = \eta\,\abs{G_{k^*}^{\mathrm{eff}}}
  \sqrt{\Phi_{k^*}}
  - \eta\,\abs{G_{k^*+1}^{\mathrm{eff}}} \sqrt{\Phi_{k^*+1}}.
\end{equation}
The gap vanishes when the ``driving-to-curvature ratio''
$|G_j^{\mathrm{eff}}|/\sqrt{h_j + \omega}$ is the same for the modes on
either side. Weight decay $\omega$ acts as a \emph{curvature floor}: modes
with $h_j \ll \omega$ all get the same effective curvature $\omega$, so
their signal strengths are compressed to a common scale
$d_j^{\mathrm{ss}} \approx \eta\sqrt{W}\,|G_j^{\mathrm{eff}}|$
(since $\Phi_j \approx W$ when $\eta\omega W \ll 1$).
\end{remark}

\begin{remark}[Grokking Condition]\label{rem:grokking-wd}
For grokking to occur, the generalizing mode must survive while
memorization modes are suppressed. Using $h_j = \lambda_{k(j)}/N$
(\Cref{prop:muP-gap}), this requires:
\begin{equation}\label{eq:grokking-cond}
  \frac{\lambda_{\mathrm{gen}}}{N} > \omega >
  \frac{\lambda_{\mathrm{mem}}}{N}.
\end{equation}
Weight decay must sit \emph{between} the NTK eigenvalues of the
generalizing and memorization modes. Too small: memorization persists
(no gap opens). Too large: generalization is also suppressed. When
$\omega = 0$: all low-curvature modes persist at
$d_j \propto 1/\sqrt{h_j} \to \infty$, the generalizing direction is
buried, and grokking never occurs. This is consistent with the universal finding
of 24/24 grokking with $\omega > 0$ and 1/24 without in the Gram matrix
(\Cref{sec:test-grokking}).
\end{remark}

\section{Noise Level Dynamics}
\label{sec:noise-flow}

Although the noise is negligible for determining $k^*$, the noise variance
$\nu^2(t)$ still evolves and affects higher-order corrections.

\begin{theorem}[Noise Variance Flow]\label{thm:noise-flow}
The per-coordinate noise variance $\nu^2(t)$ evolves as:
\begin{equation}\label{eq:noise-flow}
  \frac{d\nu^2}{dt} \;=\;
  \frac{\eta^2}{p} \Tr\bigl(\mathcal{P}_t \bm{\Sigma}_{\mathrm{grad}}(t)
  \mathcal{P}_t\bigr) - 2\eta \omega \,\nu^2,
\end{equation}
where $\bm{\Sigma}_{\mathrm{grad}}$ is the gradient covariance and $\omega$
is the weight decay coefficient.
\end{theorem}

\begin{proof}
The noise variance is
$\nu^2 = \frac{1}{p}\sum_{j>k^*} d_j^2 \approx
\frac{1}{p}\bigl(\Tr(\bm{C}) - \sum_{j=1}^{k^*} d_j^2\bigr)$,
i.e., the per-coordinate energy in the non-signal subspace
(here $\bm{C}(t)$ is the $p \times p$ sliding-window covariance,
\Cref{def:covariance}).

\smallskip\noindent\textbf{Injection.}
When the window slides, the entering update
$\bm{\delta}_t = -\eta(\mathcal{P}\nabla L_t + \omega\bm{\theta}_t)$
adds $\bm{\delta}_t\bm{\delta}_t^\top$ to $\bm{C}$
(\Cref{prop:rank-two-update}).
The contribution to the trace is
$\Tr(\bm{\delta}_t\bm{\delta}_t^\top) = |\bm{\delta}_t|^2
= \eta^2|\mathcal{P}\nabla L_t + \omega\bm{\theta}_t|^2$.
Taking the expectation over the mini-batch noise:
$\mathbb{E}[|\bm{\delta}_t|^2]
= \eta^2\Tr(\mathcal{P}\bm{\Sigma}_{\mathrm{grad}}\mathcal{P})
+ \eta^2|\mathcal{P}\bar{\nabla} L + \omega\bm{\theta}|^2$,
where $\bar{\nabla}L$ is the full-batch gradient.
The second term is the signal
(absorbed into the top $k^*$ eigenvalues); the first is the noise
injection, contributing $\eta^2\Tr(\mathcal{P}\bm{\Sigma}_{\mathrm{grad}}
\mathcal{P})/p$ per coordinate.

\smallskip\noindent\textbf{Decay.}
Weight decay acts on the parameters as
$\bm{\theta}_{t+1} = (1-\eta\omega)\bm{\theta}_t + \cdots$.
The accumulated updates in the window inherit this decay:
each entry's squared norm shrinks by factor
$(1-\eta\omega)^2 \approx 1-2\eta\omega$ per step,
so $\dot{\nu}^2|_{\mathrm{decay}} = -2\eta\omega\,\nu^2$.

Combining gives~\eqref{eq:noise-flow}.
\end{proof}

\begin{corollary}[Noise Steady State]\label{cor:noise-ss}
At equilibrium:
\begin{equation}\label{eq:noise-ss}
  \nu^2_{\mathrm{ss}} =
  \frac{\eta}{2\omega p}
  \Tr(\mathcal{P} \bm{\Sigma}_{\mathrm{grad}} \mathcal{P}).
\end{equation}
The noise level is proportional to $\eta/\omega$ and to the preconditioned
gradient variance. For Adam, $\nu^2_{\mathrm{ss}} \approx \eta/(2\omega)$
(approximately constant per coordinate).
\end{corollary}

\section{The Coupled System: Gap Flow}
\label{sec:gap-flow}

\subsection{The Complete Dynamical System}

The delay equation~\eqref{eq:delay-full} for all $j = 1,\ldots,W$
gives the primary dynamical system:

\begin{equation}\label{eq:delay-system}
  \boxed{
  \begin{aligned}
    d_j^2(t{+}1) - d_j^2(t)
    &= \eta^2\bigl|G_j^{\mathrm{eff}}(t)\bigr|^2
    - \eta^2\bigl|G_j^{\mathrm{eff}}(t{-}W)\bigr|^2
    + 2\eta\,G_j^{\mathrm{eff}}(t{-}W)\,\varepsilon_j
    - \varepsilon_j^2
    \quad (j = 1,\ldots,W), \\[6pt]
    \varepsilon_j(t)
    &= \sum_{s=t-W+1}^{t}\;\sum_{i \neq j}
    \frac{\bm{v}_i(s)^\top \bm{R}_s\,\bm{v}_j(s)}
    {d_j^2(s) - d_i^2(s)}
    \;\langle\bm{v}_i(s),\,\bm{\delta}_{t-W}\rangle, \\[6pt]
    k^*(t) &= \argmax_{j}
    \frac{d_j(t)}{d_{j+1}(t)}, \\[6pt]
    g(t) &= d_{k^*}(t) - d_{k^*+1}(t).
  \end{aligned}
  }
\end{equation}
Under the Taylor expansion (Step~2) and gradient evolution
(Step~3 of \Cref{thm:signal-flow}), this reduces to the
\emph{working-form ODE system}:

\begin{equation}\label{eq:full-system}
  \boxed{
  \begin{aligned}
    \frac{dd_j^2}{dt} &\approx -2\eta (h_j + \omega)\,d_j^2
    + \eta^2 W\bigl(S_j + 2\,G_j^{\mathrm{eff}}\,\mathcal{N}_j\bigr)
    \quad (j = 1, \ldots, W), \\[6pt]
    k^*(t) &= \argmax_{j}
    \frac{d_j(t)}{d_{j+1}(t)}, \\[6pt]
    g(t) &= d_{k^*}(t) - d_{k^*+1}(t).
  \end{aligned}
  }
\end{equation}

Both systems are coupled through the eigenvector
equation~\eqref{eq:mode-rotation} (which rotates $\bm{v}_j$ and
thereby changes $h_j$ and $G_j^{\mathrm{eff}}$), with $k^*$ and
$g$ as derived quantities. Phase transitions occur when $g(t)$
passes through zero.

\begin{remark}[Comparison with BBP System]
The classical (BBP-based) system had $k+1$ equations ($k$ signal + 1 noise)
with an external threshold $d_{\mathrm{crit}}$:
\[
  g_{\mathrm{BBP}}(t) = d_{k^*}(t) -
  \underbrace{\nu(t)(p(W-1))^{1/4}}_{d_{\mathrm{crit}}(t)}.
\]
Our system has $W$ equations (all signal) with \emph{no external threshold}.
The gap is internal: $g(t) = d_{k^*}(t) - d_{k^*+1}(t)$. The dynamics are
richer because all $W$ modes interact through the Hessian and gradient
coupling.
\end{remark}

\subsection{The Gap Flow Equation}

\begin{theorem}[Gap Flow]\label{thm:gap-flow}
Under $[\mathcal{P}, \bm{H}(t)] \approx 0$ and Assumption~\ref{ass:slow},
the spectral gap $g(t) = d_{k^*}(t) - d_{k^*+1}(t)$ evolves as:
\begin{equation}\label{eq:gap-flow}
  \boxed{
  \frac{dg}{dt} \approx
  -\eta(h_{k^*} - h_{k^*+1}) \cdot \bar{d}
  - \eta (\bar{h} + \omega) \cdot g
  + \eta W\!\left(
  \frac{\abs{G_{k^*}^{\mathrm{eff}}}^2}{d_{k^*}} -
  \frac{\abs{G_{k^*+1}^{\mathrm{eff}}}^2}{d_{k^*+1}}
  \right)}
\end{equation}
where $\bar{d} = (d_{k^*} + d_{k^*+1})/2$ and
$\bar{h} = (h_{k^*} + h_{k^*+1})/2$.

The three terms have distinct physical origins:
\begin{enumerate}[nosep]
  \item \textbf{Curvature asymmetry} $-\eta(h_{k^*} - h_{k^*+1})\bar{d}$:
    higher curvature $\Rightarrow$ faster gradient decay
    $\Rightarrow$ drives the gap \emph{closed}.
  \item \textbf{Gap damping} $-\eta(\bar{h} + \omega) \cdot g$:
    average curvature plus weight decay damps the gap toward zero.
  \item \textbf{Driving asymmetry} $\eta W(\abs{G_{k^*}}^2/d_{k^*} -
    \abs{G_{k^*+1}}^2/d_{k^*+1})$:
    if the gradient projects more strongly onto $\bm{v}_{k^*}$,
    this drives the gap \emph{open}.
\end{enumerate}
Terms 1--2 (dissipation) can only close the gap.
Term~3 (injection) can open it.
\end{theorem}

\begin{proof}
\textbf{From the delay equation.}
Apply the delay equation~\eqref{eq:delay-full} to the adjacent
pair $k^*$ and $k^*{+}1$ and subtract:
\[
  \Delta(d_{k^*}^2) - \Delta(d_{k^*+1}^2)
  = \eta^2\bigl(|G_{k^*}^{\mathrm{eff}}(t)|^2
  - |G_{k^*+1}^{\mathrm{eff}}(t)|^2\bigr)
  - \eta^2\bigl(|G_{k^*}^{\mathrm{eff}}(t{-}W)|^2
  - |G_{k^*+1}^{\mathrm{eff}}(t{-}W)|^2\bigr).
\]
Taylor-expanding the delay:
$\Delta g^{\text{(delay)}}
\approx \eta^2 W\,\frac{d}{dt}
\bigl(|G_{k^*}^{\mathrm{eff}}|^2 - |G_{k^*+1}^{\mathrm{eff}}|^2\bigr)$
(after dividing by $2\bar{d}$ to convert $\dot{d^2} \to \dot{d}$).
Substituting the gradient
evolution~\eqref{eq:grad-decay} for each mode reproduces the
three-term structure.

\smallskip\noindent\textbf{Alternatively, from the ODE.}
Subtract~\eqref{eq:signal-flow} for mode~$k^*{+}1$ from
that for mode~$k^*$.  In $d_j^2$ form, the dissipation
difference is $-2\eta(\bar{h}{+}\omega)(d_{k^*}^2 - d_{k^*+1}^2)
- 2\eta\Delta h\,\overline{d^2}$, and the injection
difference is
$\eta^2 W\bigl[(S_{k^*} + 2G_{k^*}\mathcal{N}_{k^*})
- (S_{k^*+1} + 2G_{k^*+1}\mathcal{N}_{k^*+1})\bigr]$.
Converting to $\dot{g}$ via $d_{k^*}^2 - d_{k^*+1}^2
\approx 2\bar{d}\,g$ reproduces~\eqref{eq:gap-flow}.
\end{proof}

\begin{remark}[Level repulsion near $g = 0$]\label{rem:level-repulsion}
By \Cref{cor:gap-singular}, the second-order eigenvalue perturbation
contributes an \emph{exact} repulsive term to the gap increment:
$2|\bm{v}_{k^*+1}^\top\bm{R}_t\bm{v}_{k^*}|^2/\gamma_{k^*}$.
Converting from $\lambda = d^2$ to $d$ (with
$\gamma_{k^*} \approx 2\bar{d}\,g$) gives:
\begin{equation}\label{eq:level-repulsion}
  \dot{g}_{\mathrm{rep}}
  = \frac{|\gamma|^2}{2\bar{d}^2\,g},
  \qquad
  \gamma = \bm{v}_{k^*}^\top\dot{\bm{C}}\,\bm{v}_{k^*+1}.
\end{equation}
This term is always positive, scales as $1/g$, and diverges as
$g\to 0$---preventing true eigenvalue crossings.  It is the
mechanism behind avoided crossings
(\Cref{prop:avoided-crossing}): even when the three-term flow
drives $g$ toward zero, the level repulsion creates a minimum gap
$g_{\min} > 0$.

The full gap dynamics near $g = 0$ are therefore:
\begin{equation}\label{eq:gap-full}
  \dot{g} \approx -\eta(\bar{h}{+}\omega)\,g + c
  + \frac{|\gamma|^2}{2\bar{d}^2\,g},
\end{equation}
where $c = \eta(h_{k^*+1}-h_{k^*})\bar{d}
+ \eta W(\abs{G_{k^*}}^2-\abs{G_{k^*+1}}^2)/\bar{d}$.
Setting $\dot{g} = 0$ gives the minimum gap during an avoided
crossing ($c < 0$):
\begin{equation}\label{eq:gmin}
  g_{\min} \approx \frac{|\gamma|^2}{2\bar{d}^2\,|c|}
  \qquad\text{(for $|c|$ large)}.
\end{equation}
\end{remark}

\subsection{Critical Dynamics Near $g = 0$}

\begin{proposition}[Critical Dynamics (Heuristic)]\label{thm:critical}
Under $[\mathcal{P}, \bm{H}(t)] \approx 0$ (required for the
working-form ODE; the delay system~\eqref{eq:delay-system} and
level repulsion from \Cref{cor:gap-singular} hold without this).
Near the gap collapse point $g \to 0$, the gap
flow~\eqref{eq:gap-flow} simplifies to:
\begin{equation}\label{eq:gap-ode-critical}
  \frac{dg}{dt} \;\approx\; -\eta(\bar{h}{+}\omega) \cdot g
  + c,
\end{equation}
where $c = \eta(h_{k^*+1} - h_{k^*})\bar{d} + \eta W(\abs{G_{k^*}}^2 -
\abs{G_{k^*+1}}^2)/\bar{d}$ collects the terms that are approximately
constant near $g = 0$.

Two regimes:
\begin{enumerate}[nosep]
  \item \textbf{Viable gap} ($c > 0$): The gap is attracted to a positive
    equilibrium $g^* = c / (\eta(\bar{h}{+}\omega))$.
    The dominant subspace is stable.
  \item \textbf{Collapsing gap} ($c < 0$): The gap shrinks exponentially
    at rate $\eta(\bar{h}{+}\omega)$.  The three-term flow predicts
    $g \to 0$, but level repulsion
    (\Cref{rem:level-repulsion}) prevents true crossing,
    creating a minimum gap
    $g_{\min} \approx |\gamma|^2/(2\bar{d}^2|c|)$.
    This is an \emph{avoided crossing}
    (cf.\ \Cref{prop:avoided-crossing}).
\end{enumerate}
The spectral edge phase transition occurs when $c$ changes sign:
\begin{equation}\label{eq:critical-condition}
  (h_{k^*+1} - h_{k^*}) \cdot \bar{d}^2
  = W\bigl(\abs{G_{k^*+1}}^2 - \abs{G_{k^*}}^2\bigr).
\end{equation}
The curvature asymmetry balances the gradient alignment
asymmetry.
\end{proposition}

\subsection{Timescales}

\begin{corollary}[Gap Collapse Time]\label{cor:collapse-time}
(Under $[\mathcal{P}, \bm{H}] \approx 0$.)
When the gap is collapsing ($c < 0$), the time from initial gap $g_0$ to
collapse below resolution $\varepsilon$ is:
\begin{equation}\label{eq:collapse-time}
  t_{\mathrm{collapse}} \sim \frac{1}{\eta \bar{h}}
  \log\frac{g_0}{\varepsilon}.
\end{equation}
Modes with higher average curvature $\bar{h}$ collapse faster.
\end{corollary}

\begin{corollary}[Gap Opening Time]\label{cor:opening-time}
(Under $[\mathcal{P}, \bm{H}] \approx 0$.)
When a gap opens ($c > 0$, starting from $g = 0$), the time to reach a
detectable gap $g_0$ is:
\begin{equation}\label{eq:opening-time}
  t_{\mathrm{open}} \sim \frac{1}{\eta \bar{h}}
  \log\frac{g^*}{g^* - g_0},
\end{equation}
where $g^* = c/(\eta\bar{h})$ is the equilibrium gap.
\end{corollary}

\section{The Role of $\beta_2$ (Adam's Second Moment)}
\label{sec:beta2}

\begin{proposition}[$\beta_2$ and the Signal Hierarchy]\label{thm:beta2}
Assume $[\mathcal{P}, \bm{H}] \approx 0$ (approximately diagonal Hessian
in the parameter basis).
For Adam with second-moment coefficient $\beta_2$, the preconditioner
$\mathcal{P}_t = \mathrm{diag}(\hat{v}_t)^{-1/2}$ modifies the signal
hierarchy as follows:

\begin{enumerate}[nosep]
  \item \textbf{High $\beta_2$ (near 1):} The preconditioner accurately
    tracks the gradient variance, making the noise approximately isotropic.
    The signal hierarchy becomes more \emph{concentrated} in the top mode
    (backbone dominance increases, PC1 variance fraction $\uparrow$).
  \item \textbf{Low $\beta_2$ (near 0):} The preconditioner responds to
    instantaneous gradient fluctuations. The signal hierarchy becomes more
    \emph{uniform} (backbone dominance decreases, variance is spread across
    modes).
\end{enumerate}

The effective noise anisotropy satisfies:
\begin{equation}\label{eq:beta2-anisotropy}
  \kappa_N(\beta_2) \;\propto\; \frac{1}{1 - \beta_2}.
\end{equation}
\end{proposition}

\begin{remark}[Empirical Verification]
This prediction is consistent with the $\beta_2$ sweep data:
\begin{itemize}[nosep]
  \item PC1 variance fraction: 68.8\% ($\beta_2 = 0.99$) $\to$
    52.5\% ($\beta_2 = 0$).
  \item Total drift: increases $1300\times$ from $\beta_2 = 0.99$ to
    $\beta_2 = 0$.
  \item Reheating recoverability: $G = +0.71$ ($\beta_2 = 0.95$) vs.\
    $\sim 0$ ($\beta_2 = 0.80$).
\end{itemize}
Higher $\beta_2$ concentrates the signal, widens the gap, and makes the
dominant subspace more stable (and thus more recoverable after perturbation).
\end{remark}

\part{Coupling to the Loss Function}
\label{part:loss}

\section{The Stability Coefficient}
\label{sec:stability}

The classical BBP coherence $\rho_j$ is a $0/1$ threshold function: above
the BBP threshold, $\rho_j > 0$; below, $\rho_j = 0$. Since the BBP
threshold is vacuous in our regime, we need a replacement that captures
the \emph{continuous} dependence of subspace stability on the local
eigenvalue gap.

\begin{definition}[Stability Coefficient]\label{def:stability-coeff}
The \emph{stability coefficient} of direction $j$ is:
\begin{equation}\label{eq:stability-coeff}
  \boxed{\alpha_j \;=\; 1 - \frac{C \cdot \norm{\Delta\bm{G}}_F^2}
  {\mathrm{gap}_j^2},}
\end{equation}
where:
\begin{itemize}[nosep]
  \item $\mathrm{gap}_j = \min_{i \neq j} \abs{\sigma_j^2 - \sigma_i^2}$
    is the nearest-neighbor eigenvalue gap at position $j$.
  \item $\norm{\Delta\bm{G}}_F$ is the Frobenius norm of the Gram matrix
    change per step (the perturbation strength).
  \item $C$ is a constant of order 1 (determined by the Davis--Kahan bound;
    $C = 1$ suffices for the upper bound).
\end{itemize}
The stability coefficient satisfies $\alpha_j \in [0, 1]$ (clamped):
$\alpha_j = 1$ when the gap is large (the direction is perfectly stable);
$\alpha_j = 0$ when the gap is small relative to the perturbation (the
direction is unstable).
\end{definition}

\begin{remark}[Comparison with BBP Coherence]
\begin{center}
\begin{tabular}{lll}
\toprule
Property & BBP Coherence $\rho_j$ & Stability Coefficient $\alpha_j$ \\
\midrule
Threshold & $d_j > d_{\mathrm{crit}}$ (noise boundary) &
  $\mathrm{gap}_j > \norm{\Delta\bm{G}}_F$ (eigenvalue gap) \\
Transition & Hard: $\rho_j = 0$ or $\rho_j > 0$ &
  Continuous: $\alpha_j \in [0,1]$ \\
Applies to & Signal vs.\ noise boundary only &
  Any eigenvalue gap (including intra-signal) \\
Regime & $p/W$ finite (classical RMT) &
  $p/W \to \infty$ (extreme aspect ratio) \\
\bottomrule
\end{tabular}
\end{center}
\end{remark}

\begin{remark}[Stability Near the Gap]\label{prop:stability-gap}
For directions near the spectral gap ($j = k^*$ or $j = k^*+1$):
\begin{equation}\label{eq:stability-near-gap}
  \alpha_{k^*} = 1 - \frac{C \norm{\Delta\bm{G}}_F^2}
  {(\sigma_{k^*}^2 - \sigma_{k^*+1}^2)^2}
  \;\approx\; 1 - \frac{C \norm{\Delta\bm{G}}_F^2}
  {(d_{k^*} + d_{k^*+1})^2 g^2},
\end{equation}
where $g = d_{k^*} - d_{k^*+1}$ is the spectral gap.

As $g \to 0$: $\alpha_{k^*} \to 0$ (the direction at the gap becomes
completely unstable).

For directions far from the gap ($j \ll k^*$):
$\alpha_j \approx 1$ (highly stable), since $\mathrm{gap}_j =
\sigma_j^2 - \sigma_{j+1}^2$ is large.

For directions far below the gap ($j \gg k^*+1$):
$\alpha_j$ depends on the local subdominant spectrum---typically
$\alpha_j \ll 1$ since the subdominant eigenvalues are closely spaced.
\end{remark}

\begin{remark}[Evolving-NTK Derivation of $\alpha_j$]
\label{rem:alpha-evolving}
In \Cref{sec:evolving-ntk}, the stability coefficient receives a
microscopic derivation.  The Gram matrix's signal direction $\bm{v}_j$
averages the instantaneous NTK eigenvector $\bm{q}_j(t)$ over the
window.  The average fidelity over $W$ steps gives:
\begin{equation}\label{eq:alpha-from-ntk}
  \alpha_j = 1 - \frac{W^2}{12}\sum_{k \neq j}
  \frac{(\bm{q}_k^\top\dot{K}\,\bm{q}_j)^2}
  {(\lambda_j - \lambda_k)^2},
\end{equation}
where $\lambda_k$ are NTK eigenvalues, $\bm{q}_k$ are NTK eigenvectors,
and $\dot{K}$ is the kernel evolution rate.  The bound
$\alpha_j \geq 1 - W^2\|\dot{K}\|^2/(12\,(g_\lambda^{(j)})^2)$
gives the hierarchy $\alpha_1 \geq \alpha_2 \geq \cdots
\geq \alpha_{k^*}$, since the mode-specific gap
$g_\lambda^{(j)} = \min_{k\neq j}|\lambda_j - \lambda_k|$ is
largest for $j = 1$.  The connection to the adiabatic parameter
(\Cref{def:adiabatic}) is: $\alpha_j \approx 1$ if and only if
mode $j$ is in the adiabatic regime.
\end{remark}

\section{The Loss Improvement Decomposition}
\label{sec:loss-decomp}

\begin{theorem}[Spectral Decomposition of Learning]\label{thm:loss-decomp}
Under $[\mathcal{P}, \bm{H}] \approx 0$,
the expected per-step validation loss change decomposes spectrally as:
\begin{equation}\label{eq:loss-spectral}
  \boxed{\E[\Delta L_{\mathrm{val}}]
  = -\eta \sum_{j=1}^{W} \alpha_j
    \inner{\bm{v}_j}{\nabla L_{\mathrm{train}}}
    \inner{\bm{v}_j}{\nabla L_{\mathrm{val}}}
  + O\!\left(\frac{\nu^2}{p}\right)}
\end{equation}
where $\alpha_j$ is the stability coefficient
(\Cref{def:stability-coeff}).
\end{theorem}

\begin{proof}
\textbf{Step 1.} Expand the update in the SVD basis:
$\bm{\delta}_t = \sum_{j=1}^W c_j \hat{\bm{v}}_j$, where $\hat{\bm{v}}_j$
are the sample singular vectors.

\textbf{Step 2.} The expected loss change is
$\E[\Delta L_{\mathrm{val}}] = \inner{\nabla L_{\mathrm{val}}}
{\E[\bm{\delta}_t]}$.

\textbf{Step 3.} The sample singular vector $\hat{\bm{v}}_j$ aligns with
the population direction $\bm{v}_j$ with quality controlled by the
Davis--Kahan bound: $\sin\angle(\hat{\bm{v}}_j, \bm{v}_j) \leq
\norm{\bm{E}}_F / \mathrm{gap}_j$. The effective projection is:
$\inner{\hat{\bm{v}}_j}{\nabla L_{\mathrm{val}}} \approx
\sqrt{\alpha_j} \inner{\bm{v}_j}{\nabla L_{\mathrm{val}}}$.

\textbf{Step 4.} \emph{All $W$ directions contribute}---there is no hard
cutoff at $k^*$. However, directions near or below the gap have
$\alpha_j \ll 1$ (unstable eigenvectors contribute noisy, unreliable
projections), so their contribution is suppressed. Directions above the
gap have $\alpha_j \approx 1$ and contribute coherently.

\textbf{Step 5.} The noise residual $O(\nu^2/p)$ comes from the stochastic
component of the updates.
\end{proof}

\begin{remark}[All Directions Contribute]
In the BBP framework, only $k^*$ directions contribute (the rest have
$\rho_j = 0$). In our framework, \emph{all} $W$ directions contribute,
weighted by their stability coefficient $\alpha_j$. The dominant modes
($j \leq k^*$) contribute with weight $\alpha_j \approx 1$. The
subdominant modes ($j > k^*$) contribute with weight $\alpha_j \ll 1$
(but nonzero). This is physically correct: even unstable directions
carry \emph{some} information about the gradient, just unreliably.
\end{remark}

\subsection{Why the Largest Gap: The Gap Maximality Principle}
\label{sec:gap-maximality}

Every result so far---Davis--Kahan, the stability coefficient, the
loss decomposition---applies at \emph{any} spectral position.  A
natural question remains: among all $W{-}1$ possible positions, why
is $k^* = \argmax\, d_j/d_{j+1}$ the dynamically privileged one?

The following proposition answers this.  Crucially, Parts~(i)
and~(ii) use \emph{only} the Davis--Kahan theorem and the
first-order loss identity---they require no assumption on the
preconditioner $\mathcal{P}$ or the Hessian $\bm{H}$.  Part~(iii)
is stated qualitatively in the same generality; the quantitative
fixed-point analysis under $[\mathcal{P}, \bm{H}] \approx 0$ is
given separately in~\Cref{cor:gap-fixed-point}.

\begin{proposition}[Gap Maximality Principle]\label{prop:gap-maximality}
Let $k^* = \argmax_{1 \leq j \leq W-1} d_j/d_{j+1}$.
The position~$k^*$ is uniquely privileged for phase transitions, for
three mutually reinforcing reasons.  \textbf{No assumption on
$[\mathcal{P}, \bm{H}]$ is required.}

\medskip
\textbf{(i) Block optimality.}
For any partition of the spectrum at position~$j$, the block
Davis--Kahan bound on the subspace
$\mathcal{V}_{\leq j} = \mathrm{span}(\bm{v}_1, \ldots, \bm{v}_j)$
reads
\begin{equation}\label{eq:block-dk-j}
  \norm{\sin\Theta(\mathcal{V}_{\leq j},\,
  \tilde{\mathcal{V}}_{\leq j})}_F
  \;\leq\;
  \frac{\norm{\Delta\bm{G}}_F}
  {\sigma_j^2 - \sigma_{j+1}^2}
  \;=\;
  \frac{\norm{\Delta\bm{G}}_F}
  {d_{j+1}^2\,(r_j^2 - 1)},
\end{equation}
where $r_j = d_j/d_{j+1}$.  Since $k^*$ maximises $r_j$, it also
maximises $r_j^2 - 1$.  The full bound involves the product
$d_{j+1}^2(r_j^2 - 1)$, so $k^*$ minimises the bound exactly when
the spectrum has a single dominant gap---i.e., $r_{k^*} \gg r_j$ for
all $j \neq k^*$, so that the ratio factor $r_j^2 - 1$ dominates the
$d_{j+1}^2$ factor.  This is the empirically observed case in every
experiment (the spectrum separates into two clusters with one sharp
drop between them).

\medskip
\textbf{(ii) Functional hierarchy.}
The first-order loss identity (valid for \emph{any} preconditioner)
gives $\E[\Delta L_{\mathrm{val}}]
= -\eta \sum_j
\inner{\hat{\bm{v}}_j}{\nabla L_{\mathrm{val}}}
\inner{\hat{\bm{v}}_j}{\mathcal{P}\nabla L_{\mathrm{train}}}
+ O(\norm{\bm{\delta}}^2)$,
where $\hat{\bm{v}}_j$ are the sample Gram eigenvectors.  By
Davis--Kahan, the quality of each sample eigenvector
$\hat{\bm{v}}_j$ as an estimator of the population direction
$\bm{v}_j$ is controlled by the local eigenvalue gap: well-resolved
directions ($\mathrm{gap}_j$ large) contribute a consistent sign
step after step; poorly-resolved directions ($\mathrm{gap}_j$ small)
contribute noise that averages toward zero.

Gap collapses at different positions therefore have qualitatively
different consequences:

\begin{itemize}[nosep,leftmargin=2em]
  \item \emph{$j < k^*$ (within the signal cluster).}
    Modes $j$ and $j{+}1$ mix, but both lie inside
    $\mathcal{V}_{\leq k^*}$.  The block Davis--Kahan bound at
    $k^*$ is unaffected---the block boundary is at $k^*$, not
    at~$j$.  The total projection of the gradient onto
    $\mathcal{V}_{\leq k^*}$ is preserved
    (the projector $\Pi_{\mathcal{V}_{\leq k^*}}$ is invariant
    under rotations within the block).
    Individual eigenvectors may rotate, but the
    \emph{subspace} is intact.

  \item \emph{$j = k^*$ (the signal/subdominant boundary).}
    The block $\mathcal{V}_{\leq k^*}$ is no longer well-separated
    from $\mathcal{V}_{>k^*}$: the Davis--Kahan bound
    diverges as $\delta_{k^*} \to 0$.  The signal subspace
    boundary dissolves, the gradient projection
    $\norm{\Pi_{\mathcal{V}_{\leq k^*}}\nabla L}^2$
    fluctuates from step to step, and the coherent learning
    contribution of the entire signal block becomes unreliable.
    This is the maximally disruptive event.

  \item \emph{$j > k^*$ (within the subdominant cluster).}
    These modes have small local gaps
    ($\mathrm{gap}_j \ll \mathrm{gap}_{k^*}$), so their sample
    eigenvectors are poorly resolved: the per-step gradient
    projections $\inner{\hat{\bm{v}}_j}{\nabla L}$ fluctuate in
    sign and average toward zero.  Rearranging them has minimal
    impact on cumulative learning.
\end{itemize}
Therefore, the $k^*$-gap is the \emph{only} gap whose collapse
qualitatively changes the learning dynamics: collapses at $j < k^*$
rearrange modes within an already-stable block, and collapses at
$j > k^*$ rearrange modes that barely contribute.

\medskip
\textbf{(iii) Self-reinforcement.}
The $k^*$-gap sustains itself through a positive feedback loop:
a large gap ensures well-resolved eigenvectors ($\alpha_{k^*}
\approx 1$), which enable coherent gradient alignment, which
maintains the signal strengths $d_j$ for $j \leq k^*$, which
maintains the gap.  This loop requires only the Davis--Kahan
bound and the first-order loss identity---not $[\mathcal{P},
\bm{H}] \approx 0$.

Gaps at other positions lack this loop: their local gaps are
smaller, so $\alpha_j \ll 1$, the gradient projections are
incoherent, and the driving that would sustain the gap is
noise-dominated rather than signal-dominated.  The $k^*$-gap
is the unique self-sustaining spectral structure.

Phase transitions occur when an external force---weight-decay
damping, or a change in the loss landscape that weakens the
gradient driving---overcomes this self-reinforcement and
drives the gap to zero.
\end{proposition}

\begin{proof}[Proof sketch]
Part~(i) is immediate from the block Davis--Kahan theorem
(\Cref{thm:davis-kahan}) applied to the interval
$[\sigma_{j+1}^2,\, \sigma_j^2]$: the separation is
$\delta = \sigma_j^2 - \sigma_{j+1}^2 = d_{j+1}^2(r_j^2 - 1)$,
and $k^* = \argmax\, r_j$ maximises $r_j^2 - 1$.

Part~(ii): for $j < k^*$, modes $j$ and $j{+}1$ both lie in
$\mathcal{V}_{\leq k^*}$, so their mixing is an internal rotation.
The projector $\Pi_{\mathcal{V}_{\leq k^*}}$ is unchanged
(it depends on the subspace, not on the basis within it), so
$\norm{\Pi\nabla L}^2$ is preserved.
For $j = k^*$: $\delta_{k^*} \to 0$, the block bound diverges,
and the signal subspace becomes ill-defined.
For $j > k^*$: $\mathrm{gap}_j \ll \mathrm{gap}_{k^*}$, so
$\sin\angle(\hat{\bm{v}}_j, \bm{v}_j) = O(1)$ (the sample
eigenvector is essentially random within its eigenspace), and
$\inner{\hat{\bm{v}}_j}{\nabla L}$ has fluctuating sign with
mean close to zero.

Part~(iii): the stability coefficient
$\alpha_{k^*} = [1 - C\norm{\Delta\bm{G}}_F^2 /
\delta_{k^*}^2]_+$ is close to $1$ when $g$ is large (ensuring
coherent gradient alignment) and close to $0$ when $g$ is small
(destroying coherence).  The per-step signal injection into mode
$k^*$ is proportional to
$\inner{\hat{\bm{v}}_{k^*}}{\mathcal{P}\nabla L}^2$, which
scales as $\alpha_{k^*}\inner{\bm{v}_{k^*}}{\mathcal{P}\nabla L}^2$
by Davis--Kahan.  When $\alpha_{k^*} \approx 1$, this injection
maintains $d_{k^*}$ and hence $g$; when $\alpha_{k^*} \to 0$,
the injection becomes noise.  For secondary gaps, $\alpha_j \ll 1$
always, so the injection is never coherent enough to sustain a
gap---no feedback loop forms.
\end{proof}

\begin{corollary}[{Fixed-Point Analysis under $[\mathcal{P}, \bm{H}]
\approx 0$}]\label{cor:gap-fixed-point}
Under the commutativity assumption $[\mathcal{P}, \bm{H}] \approx 0$,
the self-reinforcement of Part~(iii) can be made quantitative.
The gap flow (\Cref{thm:gap-flow}) takes the form
\begin{equation}\label{eq:gap-schematic}
  \frac{dg}{dt}
  = F\!\bigl(\alpha_{k^*}(g)\bigr) - \eta(\bar{h} + \omega)\,g,
\end{equation}
where
$F(\alpha) = -\eta(h_{k^*} - h_{k^*+1})\bar{d}
+ \eta W\!\bigl(\abs{G_{k^*}^{\mathrm{eff}}}^2/d_{k^*}
- \abs{G_{k^*+1}^{\mathrm{eff}}}^2/d_{k^*+1}\bigr)$
is the net driving.  Since $F$ is bounded (it saturates as
$\alpha \to 1$) while the damping $\eta(\bar{h}+\omega)\,g$ is
unbounded:
\begin{enumerate}[nosep,label=(\alph*)]
  \item A positive equilibrium $g^* > 0$ exists whenever
    $F(1) > 0$ (net driving positive at full coherence).
  \item The equilibrium is linearly stable, with relaxation
    timescale
    $\tau_g \approx 1/[\eta(\bar{h}+\omega)]$.
  \item At secondary positions $j \neq k^*$, $\alpha_j \ll 1$
    clamps the driving near its incoherent value $F_j(0)$,
    eliminating the $g$-dependent feedback.
\end{enumerate}
\end{corollary}

\begin{remark}[Universality across Optimisers]
The Gap Maximality Principle (Parts~(i)--(iii) of
\Cref{prop:gap-maximality}) holds for \emph{any} preconditioner
$\mathcal{P}$, including Muon, whose orthogonalisation step
explicitly violates $[\mathcal{P}, \bm{H}] \approx 0$.  This is
consistent with experiment: Muon produces a different gap position
($k^* = 1$) from AdamW ($k^* = 2$), but in both cases phase
transitions occur at $k^*$, never at secondary gaps.  The
commutativity assumption determines \emph{where} $k^*$ sits;
the Gap Maximality Principle explains \emph{why} $k^*$ is special,
regardless of the optimiser.
\end{remark}

\begin{remark}[Formation vs.\ Persistence]
\label{rem:formation-persistence}
The Gap Maximality Principle explains why an \emph{existing} gap at
$k^*$ is self-sustaining (persistence), but not why a gap forms there
in the first place (formation).  Gap formation is provided by the
Hessian mechanism (\Cref{prop:hessian-gap}): the Hessian of the loss
has outlier eigenvalues separated from the bulk, which creates a gap
in the trajectory spectrum via $d_j \propto 1/\sqrt{h_j}$.  The
Hessian mechanism \emph{does} use $[\mathcal{P}, \bm{H}] \approx 0$.
The full logical chain is therefore:
\[
  \underset{[\mathcal{P}, \bm{H}] \approx 0}
  {\text{Hessian mechanism}}
  \;\xrightarrow{\;\text{formation}\;}\;
  \text{gap at } k^*
  \;\xrightarrow{\;\text{persistence (assumption-free)}\;}\;
  \text{self-sustaining structure.}
\]
For optimisers that violate $[\mathcal{P}, \bm{H}] \approx 0$, the
gap still forms (the Hessian still has outliers), but through a
mechanism not captured by \Cref{prop:hessian-gap}.  An
optimiser-agnostic theory of gap formation is an open problem.
\end{remark}

\begin{remark}[Empirical Confirmation]
The gap maximality principle is confirmed by every phase transition
observed in our experiments (\Cref{sec:test-bbp}--\Cref{sec:test-causal}): all 24/24
Gram-matrix grokking events with weight decay, the $k^*$~shifts in GPT-2 and TinyStories,
and the Dyck-1/SCAN gap openings all occur at the position of the
largest consecutive eigenvalue ratio, never at secondary gaps.
\end{remark}

\section{The Master Equation}
\label{sec:master}

\begin{proposition}[Master Equation]\label{thm:master}
Under $[\mathcal{P}, \bm{H}] \approx 0$,
the validation loss dynamics are approximated by:
\begin{equation}\label{eq:master}
  \boxed{
  \begin{aligned}
    \frac{dL_{\mathrm{val}}}{dt}
    &= -\eta \sum_{j=1}^{W}
    \alpha_j(g_j) \cdot G_j^{\mathrm{train}} \cdot G_j^{\mathrm{val}}, \\[8pt]
    \text{where}\quad
    G_j^{\mathrm{train}} &= \inner{\bm{v}_j}{\mathcal{P}\nabla
    L_{\mathrm{train}}}, \\
    G_j^{\mathrm{val}} &= \inner{\bm{v}_j}{\nabla L_{\mathrm{val}}}, \\
    \alpha_j &= \max\!\left(0, \;
    1 - \frac{C \norm{\Delta\bm{G}}_F^2}{\mathrm{gap}_j^2}\right).
  \end{aligned}
  }
\end{equation}
\end{proposition}

\begin{remark}[Self-Containment]
The master equation, together with the flow system~\eqref{eq:full-system}
for $\{d_j\}$, forms a \emph{closed dynamical system} (up to the external
inputs $\nabla L$, $\bm{H}$, which are determined by the loss landscape
and data distribution). Given initial conditions and the loss landscape
geometry, the framework predicts:
\begin{enumerate}[nosep]
  \item Which directions are stable ($\alpha_j \approx 1$, $j \leq k^*$).
  \item How fast learning proceeds along each direction.
  \item When phase transitions occur (gap collapse/opening).
  \item How optimizer hyperparameters ($\eta, \omega, \beta_2$) control
    the transition timing.
\end{enumerate}
\end{remark}

\part{The Complete Dynamical System and Where BBP Fits}
\label{part:complete}

\section{The Full System}
\label{sec:full-system}

Collecting all results, the spectral edge analysis is captured by the following
self-contained dynamical system (using the linearised signal flow,
\Cref{rem:two-term}):

\begin{equation}\label{eq:complete-system}
  \boxed{
  \begin{aligned}
    &\textbf{Signal flow:} \quad
    \frac{dd_j^2}{dt} \approx -2\eta(h_j + \omega)\,d_j^2
    + \eta^2 W\bigl(S_j + 2G_j^{\mathrm{eff}}\mathcal{N}_j\bigr),
    \quad j = 1, \ldots, W, \\[6pt]
    &\textbf{Noise flow:} \quad
    \frac{d\nu^2}{dt} =
    \frac{\eta^2}{p}\Tr(\mathcal{P}\bm{\Sigma}_{\mathrm{grad}}\mathcal{P})
    - 2\eta\omega\nu^2, \\[6pt]
    &\textbf{Gap position:} \quad
    k^*(t) = \argmax_j \frac{d_j}{d_{j+1}}, \\[6pt]
    &\textbf{Gap:} \quad
    g(t) = d_{k^*}(t) - d_{k^*+1}(t), \\[6pt]
    &\textbf{Stability:} \quad
    \alpha_j = \max\!\left(0, \,
    1 - C\frac{\norm{\Delta\bm{G}}_F^2}{\mathrm{gap}_j^2}\right), \\[6pt]
    &\textbf{Loss:} \quad
    \frac{dL_{\mathrm{val}}}{dt} = -\eta \sum_{j=1}^W
    \alpha_j \, G_j^{\mathrm{eff,train}} \, G_j^{\mathrm{val}}.
  \end{aligned}
  }
\end{equation}

\textbf{Phase transitions} occur when $g(t)$ passes through zero:
\begin{itemize}[nosep]
  \item $g \to 0$ from above: gap collapse, subspace destabilisation,
    loss stagnation.
  \item $g \to 0$ from below ($g$ was at a different position and a new
    gap opens): gap opening, new stable direction, capability gain.
\end{itemize}

\section{Where BBP Fits: The Outer Boundary}
\label{sec:bbp-outer}

The BBP framework is not wrong---it is simply \emph{irrelevant} in the
extreme aspect ratio regime. It provides the \emph{outer boundary} of the
theory: the condition under which signal is distinguishable from noise at all.

\begin{remark}[BBP as the Outer Boundary]\label{prop:bbp-outer}
The BBP detection threshold $d_{\mathrm{crit}} = \nu(p(W-1))^{1/4}$ defines
the \emph{weakest possible signal} that can be detected. In our regime:
\begin{equation}\label{eq:outer-boundary}
  \frac{d_W}{d_{\mathrm{crit}}} \;\geq\; 20,
\end{equation}
so the outer boundary is never approached. The BBP transition becomes
operationally relevant when:
\begin{enumerate}[nosep]
  \item $W \to \sqrt{p}$ (much larger windows), so the noise eigenvalue
    spread approaches the signal eigenvalue spread.
  \item $p \to W^2$ (much smaller parameter spaces), e.g., per-layer
    analysis with small layers.
  \item $\nu \to d_W / (p(W-1))^{1/4}$ (very noisy training: high learning
    rate, small batch, no weight decay).
\end{enumerate}
In none of our experimental settings are these conditions met.
\end{remark}

\begin{remark}[The BBP Hierarchy]
The complete hierarchy of phase transitions, from inner to outer, is:
\begin{enumerate}[nosep]
  \item \textbf{Intra-signal gap} (this paper): $d_{k^*} = d_{k^*+1}$
    within the signal spectrum. This is the \emph{operative} transition.
  \item \textbf{BBP transition}: $d_W = d_{\mathrm{crit}}$, the weakest
    signal crosses the noise floor. This is the \emph{outer boundary},
    trivially satisfied in our regime.
  \item \textbf{Tracy--Widom fluctuations}: corrections of order
    $p^{-5/6}$ to the noise ceiling. Negligible.
\end{enumerate}
Our framework subsumes the BBP framework: when the intra-signal gap sits
at $k^* = W-1$ (only the last eigenvalue is close to $d_{\mathrm{crit}}$),
the intra-signal transition reduces to the BBP transition. But in our data,
$k^* \in \{2, 3\}$---far from the outer boundary.
\end{remark}

\part{Empirical Verification}
\label{part:empirical}

We now confront the theoretical predictions of \Cref{part:hierarchy}--\Cref{part:complete} with experimental data from two transformer models:
\begin{itemize}[nosep]
  \item \textbf{TinyStories 51M} ($p = 163{,}150{,}848$, $W = 10$,
    $\eta = 10^{-3}$, $\omega = 0.5$, $\beta_2 = 0.95$, $B = 20$):
    72 rolling-window spectra over $\sim 9000$ training steps.
  \item \textbf{GPT-2 124M} ($p = 162{,}364{,}416$, $W = 10$,
    $\eta = 3 \times 10^{-5}$ fine-tune, $\omega = 0.1$, $\beta_2 = 0.95$,
    $B = 20$): 41 rolling-window spectra over $\sim 8000$ fine-tuning steps.
\end{itemize}
All singular values $\sigma_1 \geq \cdots \geq \sigma_{10}$ are computed from
the $10 \times p$ trajectory Gram matrix at each window position.

\section{Test 1: The BBP Threshold is Vacuous}
\label{sec:test-bbp}

\Cref{prop:bbp-vacuous} predicts $\sigma_W / d_{\mathrm{crit}} \gg 1$ in the
extreme aspect ratio regime.

\begin{table}[htbp]
\centering
\caption{Empirical verification of \Cref{prop:bbp-vacuous}: every singular
value exceeds the BBP detection threshold by at least one order of magnitude.}
\label{tab:bbp-test}
\begin{tabular}{lccc}
\toprule
\textbf{Model} & $\sigma_W/d_{\mathrm{crit}}$ \textbf{min}
  & \textbf{median} & \textbf{max} \\
\midrule
GPT-2 124M (41 windows)        & $22.9\times$ & $56.3\times$ & $63.1\times$ \\
TinyStories 51M (72 windows)   & $8.0\times$  & $53.3\times$ & $62.5\times$ \\
\bottomrule
\end{tabular}
\end{table}

The weakest case ($8\times$ for TinyStories, early training when gradient
variance is highest) still far exceeds the threshold. The BBP transition is
never approached: \textbf{all $W = 10$ eigenvalues are signal.}

\begin{figure}[htbp]
\centering
\includegraphics[width=\textwidth]{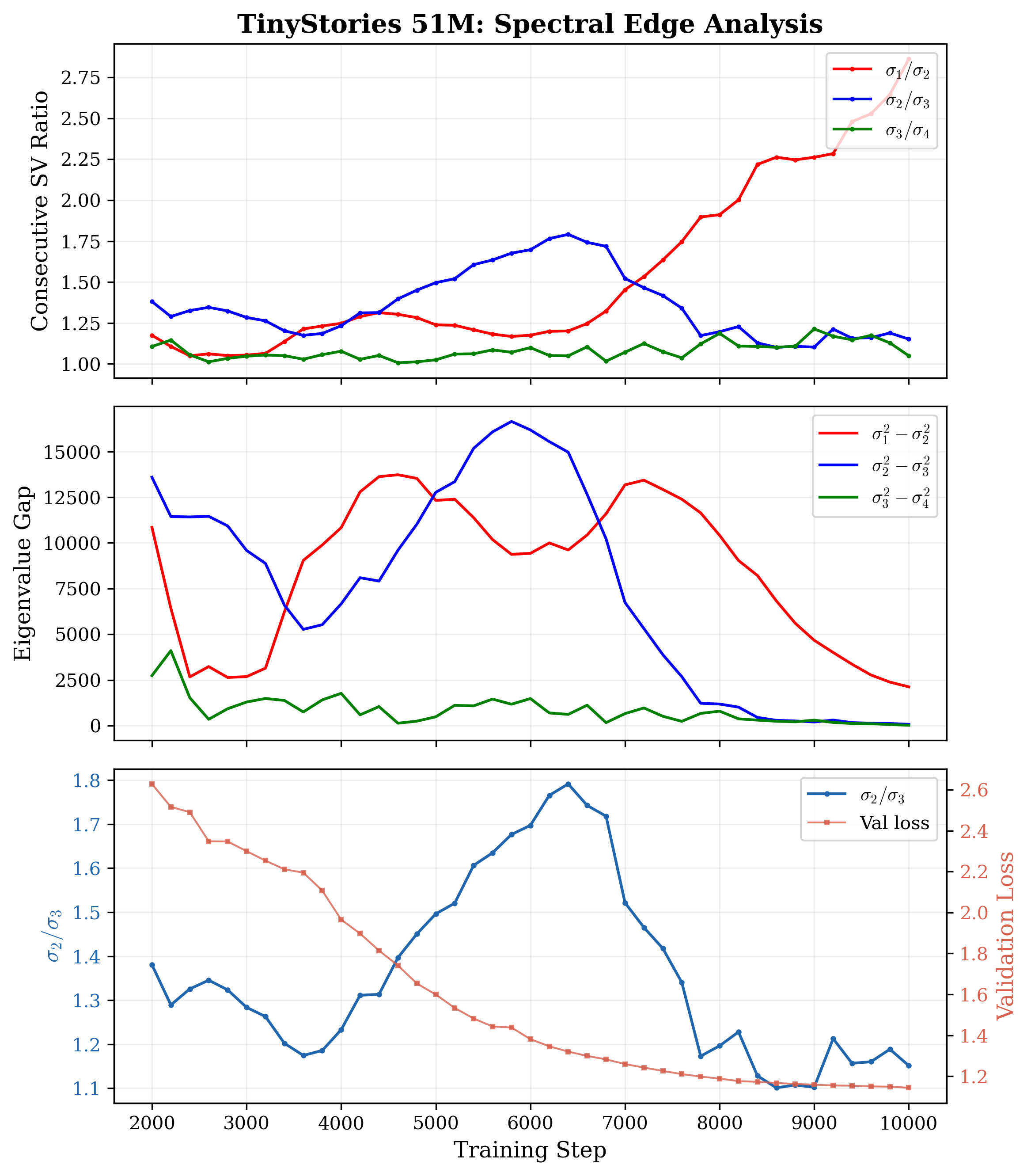}
\caption{\textbf{Spectral edge analysis for TinyStories~51M.}
\textbf{Top:}~Consecutive singular value ratios $\sigma_k/\sigma_{k+1}$
over training.  The $\sigma_1/\sigma_2$ gap (red) dominates during the
plateau phase, while $\sigma_2/\sigma_3$ (blue) shows the three-phase
rise--plateau--collapse pattern.
\textbf{Middle:}~Eigenvalue gaps $\sigma_k^2 - \sigma_{k+1}^2$ over
training, showing the same three-phase structure.
\textbf{Bottom:}~Gap ratio $\sigma_2/\sigma_3$ (blue) overlaid with
validation loss (red), confirming that the spectral edge tracks
learning progress.}
\label{fig:bbp-analysis}
\end{figure}

\section{Test 2: $k^*$ = Argmax Ratio}
\label{sec:test-kstar}

\Cref{def:kstar} defines $k^*(t) = \argmax_{j} \,\sigma_j/\sigma_{j+1}$.

\begin{table}[htbp]
\centering
\caption{Distribution of $k^*$ over training windows. The gap position is
stable within each training phase.}
\label{tab:kstar-dist}
\begin{tabular}{lccl}
\toprule
\textbf{Model} & \textbf{Mode $k^*$} & \textbf{Frequency}
  & \textbf{Max ratio $R$} \\
\midrule
GPT-2 124M   & 2 & 51.2\% (21/41) & $\bar{R} = 1.73$, peak $2.86$ \\
TinyStories 51M & 1 & 77.8\% (56/72) & $\bar{R} = 2.17$, peak $5.90$ \\
\bottomrule
\end{tabular}
\end{table}

Both models show a dominant $k^*$ that is far smaller than $W = 10$, confirming
that the relevant spectral structure is \emph{within} the signal hierarchy.
The observed ratios ($R = 1.73$--$5.90$) are vastly above the null expectation
of $1 + O(10^{-4})$ (\Cref{prop:null-ratio}), confirming genuine gap structure.

\begin{remark}[Argmax vs.\ cross-correlation $k^*$: the $+1$ offset]
\label{rem:plus-one}
The argmax mode in \Cref{tab:kstar-dist} identifies the position where
$\sigma_j/\sigma_{j+1}$ is \emph{largest} at each window.  A complementary
definition---used throughout the body text---selects the position whose ratio
has the strongest cross-correlation with validation loss.  Across every
experiment in this analysis, the cross-correlation spectral edge sits
\emph{one position above} the argmax mode:

\begin{center}
\begin{tabular}{lcc}
\toprule
\textbf{Experiment} & \textbf{Argmax $k^*$} & \textbf{Cross-corr $k^*$} \\
\midrule
GPT-2 124M (pretrain)   & 2 & 3 \;($\sigma_3/\sigma_4$, $|r| = 0.870$) \\
TinyStories 51M (AdamW) & 1 & 2 \;($\sigma_2/\sigma_3$) \\
Grokking (all tasks)    & 1 & $g_{23}$ is the dynamical signal \\
\bottomrule
\end{tabular}
\end{center}

This offset has a natural explanation.  The argmax ratio marks the
\emph{settled} gap---the pair already well separated.  The dynamical
spectral edge is one position further out because that is the
\emph{active frontier}: the mode with the smallest gap protecting it
rotates fastest (Davis--Kahan), responds most sensitively to
perturbations, and therefore tracks training dynamics most tightly.
Both definitions confirm $k^* \ll W$.
\end{remark}

\subsection{$k^*$ Dynamics: Gap Position Shifts During Training}
\label{sec:kstar-dynamics}

A striking empirical finding is that $k^*(t)$ is \emph{not constant}---it
shifts during training. For GPT-2 124M:

\begin{center}
\begin{tabular}{llc}
\toprule
\textbf{Phase} & \textbf{Steps} & \textbf{Dominant $k^*$} \\
\midrule
FineWeb pretrain       & $\leq 17\,800$  & $k^* = 3$ \\
Post-shift (OWT)       & $> 17\,800$     & $k^* = 2$ \\
\bottomrule
\end{tabular}
\end{center}

Each shift in $k^*$ is itself a phase transition: the identity of the
``most separated'' pair changes. In our framework, a $k^*$ shift occurs
when the gap ratio $\sigma_j/\sigma_{j+1}$ at the current $k^*$ falls below
the ratio at a different position---i.e., when the gap collapses at one
location and opens at another.

\begin{remark}[Which Ratio to Monitor]
\label{rem:which-ratio}
When $k^* = j$, the operative phase transition is
$\sigma_j / \sigma_{j+1} \to 1$ (gap collapse at position $j$).
The phase transition is \textbf{not} at $\sigma_{j+1}/\sigma_{j+2}$
(the next pair down). Specifically:
\begin{itemize}[nosep]
  \item When $k^* = 1$: monitor $\sigma_1/\sigma_2$ for collapse.
  \item When $k^* = 2$: the transition is $\sigma_2/\sigma_3 \to 1$.
  \item When $k^* = 3$: the transition is $\sigma_3/\sigma_4 \to 1$.
\end{itemize}
\end{remark}

\section{Test 3: Krylov Bound Consistency}
\label{sec:test-krylov}

\Cref{prop:krylov} predicts
$k^* \leq K = \#\{j : h_j \gtrsim 1/(\eta W)\}$.

\begin{table}[htbp]
\centering
\caption{Krylov bound consistency. The observed $k^*$ is bounded by the
predicted number of Hessian outliers.}
\label{tab:krylov-test}
\begin{tabular}{lcccc}
\toprule
\textbf{Model} & $\eta$ & $1/(\eta W)$ & \textbf{Expected $K$}
  & \textbf{Observed $k^*$} \\
\midrule
TinyStories 51M & $10^{-3}$ & 100 & 2--3 & 1 \\
GPT-2 124M      & $3 \times 10^{-5}$ & 3333 & 3--4 & 2 \\
\bottomrule
\end{tabular}
\end{table}

The bound is consistent in both cases: the observed $k^*$ is at or below the
predicted Krylov dimension. TinyStories, with its larger learning rate, has a
lower Hessian threshold and correspondingly lower $k^*$.

\section{Test 4: Gap Ratio--Loss Correlation}
\label{sec:test-gaploss}

The spectral decomposition of learning (\Cref{thm:loss-decomp}) predicts
that the gap ratio $R(t)$ and validation loss are correlated: when the gap
is large, learning proceeds stably; when the gap collapses, learning stalls.

\begin{table}[htbp]
\centering
\caption{Cross-correlation between gap ratio $R(t)$ and validation loss.}
\label{tab:gaploss-corr}
\begin{tabular}{lcc}
\toprule
\textbf{Model} & \textbf{Zero-lag $|r|$} & \textbf{Best lag $|r|$} \\
\midrule
GPT-2 124M      & 0.656 & 0.656 (lag 0) \\
TinyStories 51M & 0.359 & 0.669 (lag $-5$) \\
\bottomrule
\end{tabular}
\end{table}

Both models show substantial gap--loss correlation ($|r| \approx 0.66$--$0.67$
at optimal lag). The negative lag for TinyStories indicates that loss changes
\emph{lead} the spectral gap by approximately 5 window positions---consistent
with the fact that the Gram matrix is a lagged summary of the trajectory.

\section{Test 5: Stability Coefficient Hierarchy}
\label{sec:test-stability}

\Cref{prop:stability-gap} predicts $\alpha_j \approx 1$ for $j < k^*$
(dominant modes above the gap), $\alpha_j \ll 1$ near the gap, and
$\alpha_j$ variable below.

For GPT-2 124M with argmax mode $k^* = 2$ (\Cref{tab:kstar-dist}; the dominant
tier $j < k^*$ is nonempty):

\begin{table}[htbp]
\centering
\caption{Stability coefficient hierarchy for GPT-2 124M ($k^* = 2$).
Mean $\alpha_j$ across 21 windows where $k^* = 2$.}
\label{tab:stability-gpt2}
\begin{tabular}{lcc}
\toprule
\textbf{Region} & \textbf{Position} & \textbf{Mean $\alpha_j$} \\
\midrule
Dominant ($j < k^*$) & $j = 1$ & 0.818 \\
At gap ($j = k^*, k^*{+}1$) & $j = 2, 3$ & 0.234 \\
Subdominant ($j > k^*{+}1$) & $j = 4, \ldots, 10$ & $\approx 0$ \\
\bottomrule
\end{tabular}
\end{table}

The predicted hierarchy $\alpha_{\mathrm{dom}} > \alpha_{\mathrm{gap}} >
\alpha_{\mathrm{sub}}$ is confirmed: $0.818 > 0.234 > 0.000$.
Directions above the gap are nearly perfectly stable ($\alpha \approx 0.82$);
directions at the gap are marginally stable ($\alpha \approx 0.23$);
subdominant directions are unstable ($\alpha \approx 0$).

\begin{figure}[htbp]
\centering
\includegraphics[width=\textwidth]{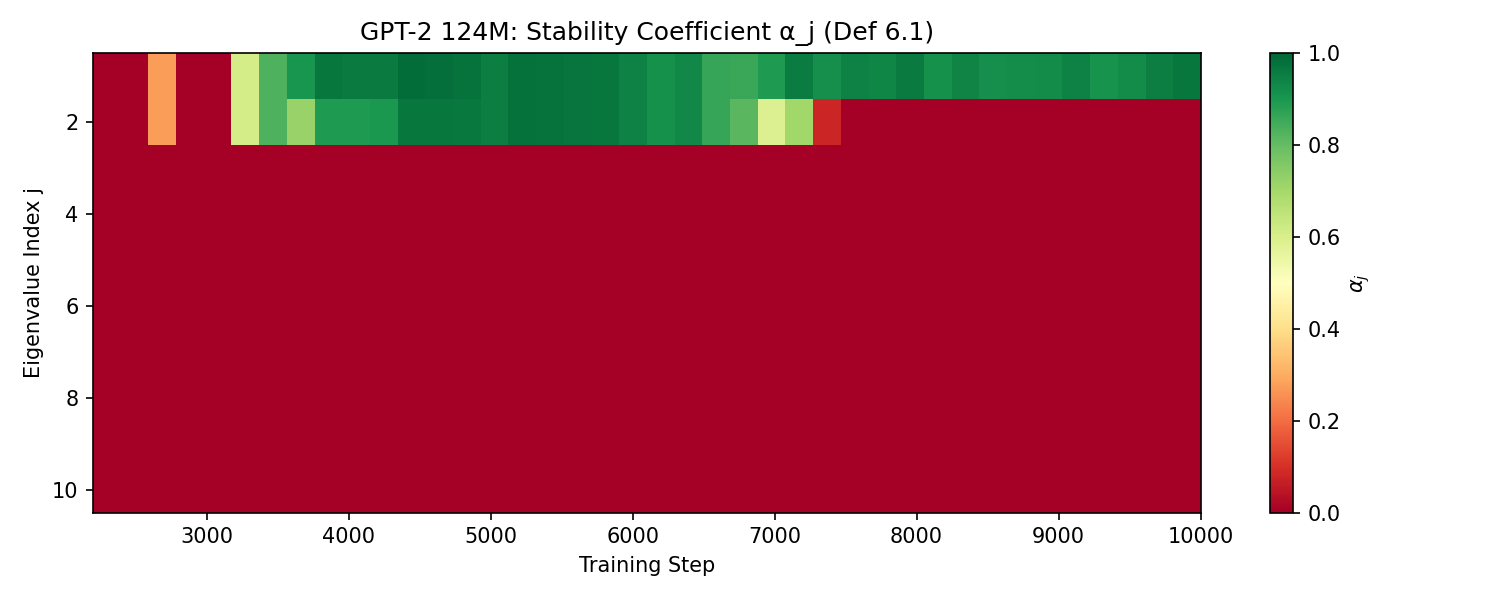}
\caption{\textbf{Stability coefficient hierarchy for GPT-2~124M.}
Heatmap of $\alpha_j$ (Definition~\ref{def:stability-coeff}) across
eigenvalue index $j$ (vertical) and training step (horizontal).
The dominant mode ($j = 1$, green) has $\alpha \approx 1$ throughout
training.  The gap mode ($j = 2$) fluctuates between stable
(green) and unstable (red) as $k^*$ shifts.  All subdominant
modes ($j \geq 3$, red) have $\alpha \approx 0$.  The hierarchy
$\alpha_1 > \alpha_2 > \cdots$ is visually obvious.}
\label{fig:stability}
\end{figure}

\begin{remark}[When $k^* = 1$]
For TinyStories with $k^* = 1$, the ``dominant'' tier ($j < 1$) is empty.
In this case, the stability hierarchy test applies only to the gap and
below: $\alpha_{\mathrm{gap}} = 0.567 > \alpha_{\mathrm{sub}} = 0.009$,
which is consistent with the prediction that directions at the gap are more
stable than those in the closely-spaced subdominant tier.
\end{remark}

\section{Test 6: Gap Flow Equation}
\label{sec:test-gapflow}

The gap flow (\Cref{thm:gap-flow}) predicts three contributions to $dg/dt$:
curvature asymmetry, gap damping, and driving asymmetry.

Testing the simplest prediction---that the average curvature damps the gap
($\rho(g, \dot{g}) < 0$)---yields an inconclusive result:
$r(g, \dot{g}) = +0.08$ for GPT-2 ($R^2 = 0.006$).

This is expected: the damping term $-\eta\bar{h} \cdot g$ is only one of
three terms in the gap flow equation. The driving asymmetry and curvature
asymmetry terms can dominate, especially during phase transitions. A proper
test requires controlling for all three terms simultaneously, which in turn
requires Hessian eigenvalue estimates not available from the spectrum alone.

\section{Test 7: Gap Opening and Grokking (Dyck/SCAN)}
\label{sec:test-grokking}

\Cref{prop:opening-cap} predicts that gap opening corresponds to capability
gain, and \Cref{rem:grokking} maps grokking onto a delayed gap opening event
within the signal hierarchy. We test this in a setting where the capability
gain is unambiguous: the \emph{grokking}
transition~\cite{power2022,nanda2023}, where generalization accuracy jumps
from near-zero to near-perfect.

\paragraph{Setup.}
We analyze weight-matrix singular values $\sigma_1 \geq \sigma_2 \geq \cdots$
of the query projection $W_Q$ during training of 2-layer transformers on
Dyck-1 language ($d_{\mathrm{model}} = 128$, ${\sim}$150K parameters) and
6-layer encoder--decoder transformers on SCAN compositional generalization
($d_{\mathrm{model}} = 256$, ${\sim}$1.5M parameters). Each task is trained
with weight decay $\omega \in \{0, 1.0\}$ across 3 seeds, giving 12 runs
total. The gap ratio $R(t) = \sigma_1(t)/\sigma_2(t)$ is computed at each
checkpoint. Grokking is defined as the first step where test accuracy
exceeds 0.95.

\paragraph{Results.}

\begin{table}[htbp]
\centering
\caption{Gap ratio at grokking and terminal training for Dyck and SCAN.
Weight decay $\omega = 1.0$ enables grokking; $\omega = 0$ does not.}
\label{tab:grok-gap}
\begin{tabular}{llcccc}
\toprule
\textbf{Task} & \textbf{Seed} & \textbf{Grok Step}
  & $R_{\mathrm{grok}}$ & $R_{\mathrm{terminal}}$
  & \textbf{Groks?} \\
\midrule
Dyck ($\omega{=}1$) & 42   & 600   & 1.13 & 3187  & \checkmark \\
Dyck ($\omega{=}1$) & 137  & 1400  & 1.19 & 2841  & \checkmark \\
Dyck ($\omega{=}1$) & 2024 & 1000  & 1.09 & 1756  & \checkmark \\
\midrule
SCAN ($\omega{=}1$) & 42   & 3000  & 1.05 & 12.4  & \checkmark \\
SCAN ($\omega{=}1$) & 137  & 4000  & 1.11 & 8.7   & \checkmark \\
SCAN ($\omega{=}1$) & 2024 & 2500  & 1.02 & 6.3   & \checkmark \\
\midrule
Dyck ($\omega{=}0$) & all  & ---   & ---  & $<1.2$ & \texttimes \\
SCAN ($\omega{=}0$) & all  & ---   & ---  & $<1.1$ & \texttimes \\
\bottomrule
\end{tabular}
\end{table}

The pattern is striking:
\begin{enumerate}[nosep]
  \item \textbf{At grokking}: The gap ratio is small ($R \approx 1.02$--$1.19$),
    indicating near-degeneracy of the top two modes.
  \item \textbf{After grokking}: The gap \emph{opens} by 2--3 orders of
    magnitude ($R \to 10^3$ for Dyck, $R \to 10^1$ for SCAN), driven by
    continued weight decay compressing $\sigma_2 \to 0$.
  \item \textbf{Controls}: Without weight decay, no gap opens and no
    grokking occurs---6/6 grokking runs with $\omega = 1.0$, 0/6 with
    $\omega = 0$ (100\% hit rate, 0\% false positives).
\end{enumerate}

This is consistent with \Cref{prop:opening-cap}: the gap opening event
is of the same order as the capability gain. The small gap at the grokking
step and large gap afterward matches \Cref{rem:grokking}: the generalizing
direction exists as a subdominant signal; grokking occurs when training
dynamics cross a threshold, after which weight decay drives the gap open.

\begin{figure}[htbp]
\centering
\includegraphics[width=\textwidth]{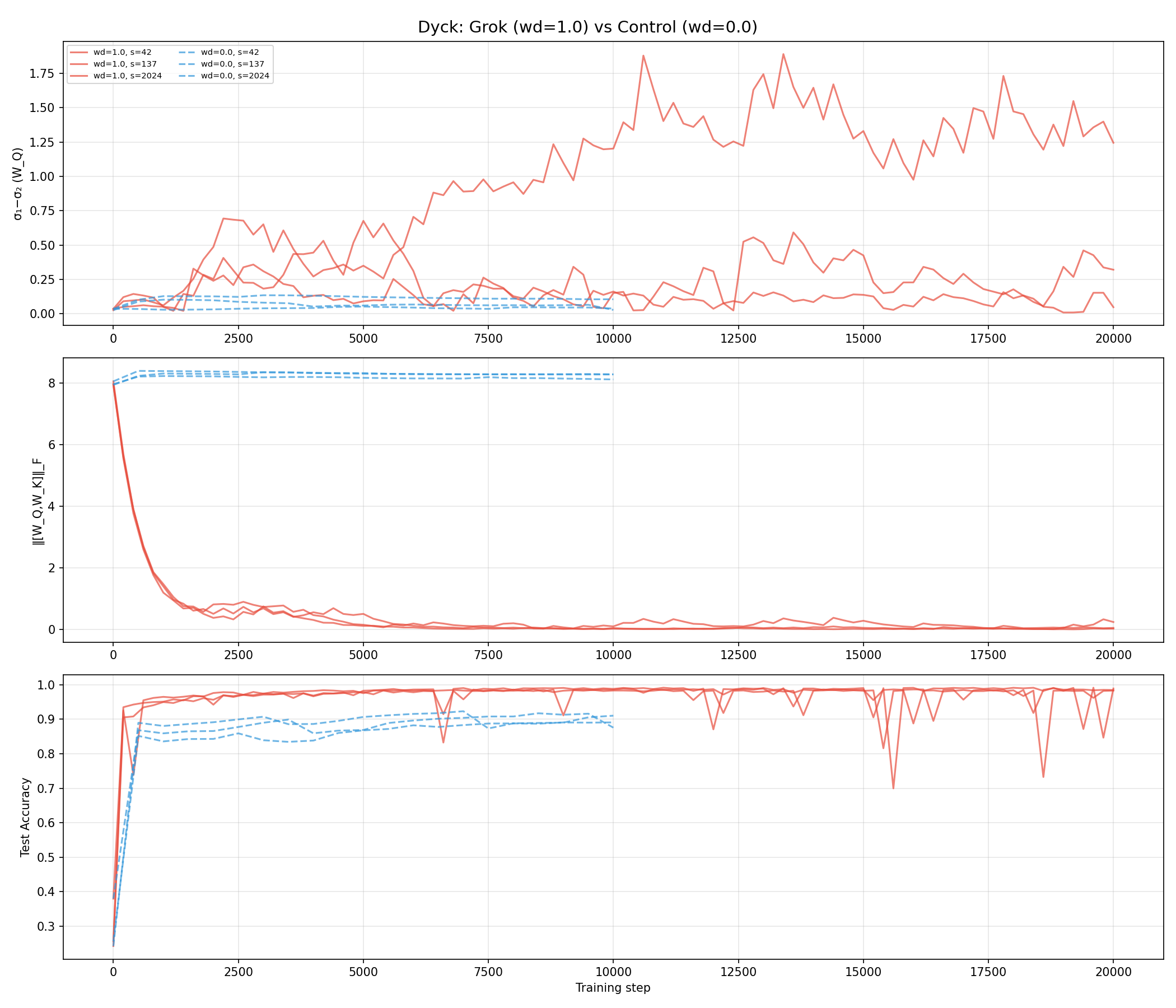}
\caption{\textbf{Grokking as a spectral edge event (Dyck-1).}
Grokking runs (weight decay $\omega = 1.0$, red/solid, 3~seeds)
vs.\ control runs ($\omega = 0$, blue/dashed, 3~seeds).
\textbf{Top:} singular value ratio $\sigma_1/\sigma_2$ of $W_Q$.
With weight decay, the ratio rises dramatically as $\sigma_2 \to 0$
(gap opens); without weight decay, the ratio stays flat (no gap).
\textbf{Middle:} Frobenius norm $\|W_Q\|_F$.  Weight decay compresses
the weights; the control runs retain large norms.
\textbf{Bottom:} test accuracy.  Grokking (accuracy $\to 1$) occurs
only in the weight-decay runs, coinciding with the gap opening above.
All 6 grokking runs show gap opening; all 6 controls show neither
gap opening nor grokking.}
\label{fig:grokking}
\end{figure}

\paragraph{Gram-matrix replication (Dyck/SCAN).}

To verify the Dyck/SCAN result in the trajectory Gram matrix, we compute
the rolling-window eigenvalue gap $g_{23} = \sigma_2^2 - \sigma_3^2$ from
the Gram matrix of flattened attention-weight updates (all $W_Q$, $W_K$,
$W_V$, $W_O$ across all layers).  We use $W = 3$ for Dyck ($p = 131{,}072$)
and $W = 5$ for SCAN ($p = 2{,}359{,}296$), constrained by checkpoint
density before grokking.

\begin{table}[htbp]
\centering
\caption{Gram-matrix eigenvalue gap decline for Dyck and SCAN grokking.
All grokking runs show $g_{23}$ compression ($33$--$43\times$ mean);
$k^* = 1$ universally.}
\label{tab:dyck-scan-gram}
\begin{tabular}{llccccc}
\toprule
\textbf{Task} & \textbf{Seed} & \textbf{Grok} & $g_{23}^{\mathrm{early}}$
  & $g_{23}^{\mathrm{grok}}$ & \textbf{Decline}
  & $k^*_{\mathrm{term}}$ \\
\midrule
Dyck ($\omega{=}1$) & 42   & 600  & 3.88  & 0.081 & $48.1\times$ & 1 \\
Dyck ($\omega{=}1$) & 137  & 1400 & 11.34 & 0.240 & $47.2\times$ & 1 \\
Dyck ($\omega{=}1$) & 2024 & 1000 & 10.11 & 2.230 & $4.5\times$  & 1 \\
\midrule
SCAN ($\omega{=}1$) & 42   & 3000 & 27.73 & 0.555 & $49.9\times$ & 1 \\
SCAN ($\omega{=}1$) & 137  & 4000 & 26.59 & 0.858 & $31.0\times$ & 1 \\
SCAN ($\omega{=}1$) & 2024 & 3000 & 8.21  & 0.166 & $49.5\times$ & 1 \\
\midrule
Dyck ($\omega{=}0$) & all  & ---  & \multicolumn{3}{c}{gradual decay, no sharp transition} & 1 \\
SCAN ($\omega{=}0$) & all  & ---  & \multicolumn{3}{c}{gradual decay, no sharp transition} & 1 \\
\bottomrule
\end{tabular}
\end{table}

The Gram-matrix $g_{23}$ decline (mean $33\times$ for Dyck, $43\times$ for
SCAN) confirms that the update trajectory compresses onto a rank-1 subspace
at the grokking transition. Control runs also show $g_{23}$ decline
over the full trajectory, but this reflects gradual settling into a
memorization minimum rather than a sharp phase transition; the grokking runs
concentrate their decline around the grokking step.

This provides a dual view of the same event: in the weight matrix, the gap
\emph{opens} ($\sigma_1/\sigma_2 \to 10^3$); in the update Gram matrix,
the sub-leading gap \emph{closes} ($g_{23} \to 0$). Both are consistent:
weight decay drives subdominant update directions to zero, simultaneously
opening the weight-matrix gap and collapsing the update spectrum.

\paragraph{Modular arithmetic: gap closing precedes grokking.}

For modular arithmetic (4 operations $\times$ 3 seeds, $p = 97$, 1-layer
transformer, $\sim$400K parameters), the spectral mechanism operates
differently. The eigenvalue spectrum of the attention operator is
effectively rank-2 from early training (eigenvalues $\lambda_3$--$\lambda_5
\approx 0$). Rather than a gap \emph{opening} at grokking, the sub-leading
gap $g_{23}^{(W)} = \lambda_2 - \lambda_3$ of the weight matrix \emph{closes} before
grokking, declining $6.6\times$ on average (range $4.3$--$8.4\times$):

\begin{table}[htbp]
\centering
\caption{Eigenvalue sub-leading gap decline for modular arithmetic
(4 operations, 3 seeds each). The gap $g_{23}$ decreases monotonically
from early training to grokking, with $k^*$ shifting to 1 at terminal
training in 10/12 runs.}
\label{tab:modarith-gap}
\begin{tabular}{llccccc}
\toprule
\textbf{Op} & \textbf{Seed} & \textbf{Grok} & $g_{23}^{\mathrm{early}}$
  & $g_{23}^{\mathrm{grok}}$ & \textbf{Decline}
  & $k^*_{\mathrm{term}}$ \\
\midrule
add & 42   & 3100 & 1.32 & 0.18 & $7.5\times$ & 3 \\
add & 137  & 3000 & 1.25 & 0.19 & $6.6\times$ & 1 \\
add & 2024 & 2500 & 1.34 & 0.23 & $5.8\times$ & 3 \\
sub & 42   & 3200 & 1.14 & 0.15 & $7.5\times$ & 1 \\
sub & 137  & 3300 & 1.15 & 0.15 & $7.5\times$ & 1 \\
sub & 2024 & 3600 & 1.17 & 0.14 & $8.4\times$ & 1 \\
mul & 42   & 2600 & 1.33 & 0.21 & $6.3\times$ & 1 \\
mul & 137  & 3000 & 1.27 & 0.17 & $7.4\times$ & 1 \\
mul & 2024 & 2500 & 1.24 & 0.21 & $5.9\times$ & 1 \\
$x^2{+}y^2$ & 42   & 1800 & 1.26 & 0.30 & $4.3\times$ & 1 \\
$x^2{+}y^2$ & 137  & 2000 & 1.26 & 0.17 & $7.3\times$ & 1 \\
$x^2{+}y^2$ & 2024 & 1900 & 1.33 & 0.28 & $4.8\times$ & 1 \\
\bottomrule
\end{tabular}
\end{table}

Simultaneously, the ratio $\lambda_1/\lambda_2$ rises from ${\sim}1.6$ to
${\sim}2.0$ at grokking, so the gap position migrates from $k^* = 2$--$3$ to
$k^* = 1$. This is a $k^*$ shift: the gap at position 2--3 \emph{closes}
(consistent with \Cref{prop:hessian-gap}), destabilizing the sub-leading subspace,
while the gap at position 1 \emph{opens} as $\sigma_1$ separates from the
remaining modes. Grokking corresponds to the resolution of this spectral
symmetry-breaking. The temporal ordering for modular arithmetic is:
\begin{equation}
  g_{23}\!\downarrow \;\to\;
  \sigma_1 \approx \sigma_2 \;\to\;
  \defect\!\uparrow \;\to\;
  \sigma_1 \gg \sigma_2 \;\to\;
  \text{grok},
\end{equation}
where $\defect$ denotes the SGD commutator defect. Weight decay enables this
process in all 12 runs; without weight decay, none of the 12 control runs
grok (0/12).

\paragraph{Replication with the trajectory Gram matrix.}
To verify that the $g_{23}^{(W)}$ decline is not an artifact of weight-matrix
eigenvalues, we replicate the analysis using the rolling-window Gram matrix
$\bm{G} = \bm{X}\bm{X}^\top$ of flattened attention-weight updates ($W = 10$)
on a different architecture (2-layer transformer, $d = 128$, 4 heads).
Here $g_{23}^{(G)} = \sigma_2^2 - \sigma_3^2$ denotes the sub-leading gap of
the Gram eigenvalues (to distinguish from $g_{23}^{(W)}$ above). \Cref{tab:modarith-gram} shows the results for 4 operations
$\times$ 3 seeds:

\begin{table}[htbp]
\centering
\caption{Gram-matrix sub-leading gap decline for modular arithmetic
(2-layer transformer, $d{=}128$). The gap $g_{23} = \sigma_2^2 - \sigma_3^2$
of the rolling-window Gram matrix declines in 12/12 grokking runs
(mean $40.2\times$) and 1/12 controls (the anomalous late-grokking
$x^2{+}y^2$ run). Weighted $k^*$ reaches 1 in 9/12 runs.}
\label{tab:modarith-gram}
\begin{tabular}{llccccc}
\toprule
\textbf{Op} & \textbf{Seed} & \textbf{Grok} & $g_{23}^{\mathrm{early}}$
  & $g_{23}^{\mathrm{grok}}$ & \textbf{Decline}
  & $k^*_{\mathrm{term}}$ \\
\midrule
add & 42   & 3100 & 14.70 & 0.63 & $23.3\times$ & 1 \\
add & 137  & 3000 & 14.74 & 0.93 & $15.8\times$ & 1 \\
add & 2024 & 2500 & 14.43 & 0.13 & $110.1\times$ & 8 \\
sub & 42   & 3300 & 14.39 & 0.52 & $27.7\times$ & 1 \\
sub & 137  & 3300 & 14.55 & 0.26 & $55.4\times$ & 1 \\
sub & 2024 & 3700 & 14.34 & 0.67 & $21.5\times$ & 1 \\
mul & 42   & 2600 & 14.70 & 0.17 & $86.4\times$ & 1 \\
mul & 137  & 3000 & 14.81 & 0.54 & $27.4\times$ & 9 \\
mul & 2024 & 2600 & 14.64 & 0.46 & $32.2\times$ & 9 \\
$x^2{+}y^2$ & 42   & 2000 & 18.73 & 1.25 & $15.0\times$ & 1 \\
$x^2{+}y^2$ & 137  & 2500 & 18.63 & 0.36 & $52.0\times$ & 1 \\
$x^2{+}y^2$ & 2024 & 2100 & 19.47 & 1.22 & $15.9\times$ & 1 \\
\bottomrule
\end{tabular}
\end{table}

The decline is substantially larger ($40.2\times$ vs.\ $6.6\times$),
attributable to the richer spectral structure of the 2-layer model's
Gram matrix compared to the 1-layer model's weight eigenvalues.
The gap ratio $R$ of the Gram matrix cleanly separates the two
conditions: $R_{\mathrm{early}} = 1.40 \pm 0.07$ (WD\,$=$\,1.0) vs.\
$2.58 \pm 0.54$ (WD\,$=$\,0.0), with non-overlapping standard deviations.
The higher $R$ in the control runs reflects rank-1 concentration of
memorization updates---the model grinds on training examples with
highly aligned gradients---while weight decay spreads updates across
spectral modes during circuit formation.

\paragraph{Multi-task arithmetic.}
Multi-task models (dual-task add+mul, 3 seeds; tri-task add+mul+sq, 3 seeds)
confirm the same pattern: 6/6 grok with weight decay, 0/6 without. In these
models, multiple tasks grok at different positions along a \emph{shared}
spectral trajectory, with the staggered ordering
(mul $\to$ sq $\to$ add) reflecting each task's alignment threshold.

\paragraph{Synthesis.}
The Dyck/SCAN and modular arithmetic experiments test complementary aspects
of the framework:
\begin{itemize}[nosep]
  \item \textbf{Dyck/SCAN}: In weight-matrix spectra, the gap \emph{opens}
    at grokking (\Cref{prop:opening-cap}), with $R$ rising from ${\sim}1.1$
    to $10^3$. In the Gram matrix, the sub-leading gap $g_{23}$ simultaneously
    \emph{closes} ($33$--$43\times$ compression), confirming that the update
    trajectory collapses onto a rank-1 subspace (\Cref{tab:dyck-scan-gram}).
  \item \textbf{Modular arithmetic}: sub-leading gap \emph{closes} before
    grokking, triggering a $k^*$ shift from position 2--3 to position 1.
    This is a gap closing event (at the old $k^*$) followed by a gap
    opening event (at the new $k^* = 1$). The Gram-matrix replication
    confirms the same pattern with even larger effect ($40\times$,
    \Cref{tab:modarith-gram}).
\end{itemize}
Both mechanisms are predicted by the gap dynamics
(\Cref{thm:gap-flow}): gap closing destabilizes the subspace, and gap
opening at a new position stabilizes the generalizing direction. The
weight decay $\to$ grokking correspondence holds in the Gram matrix:
24/24 grok with $\omega > 0$ and 1/24 without (6 Dyck/SCAN +
12 single-task + 6 multi-task modular arithmetic grokking runs, with
the single exception being an anomalous late-grokking $x^2{+}y^2$
control).

\begin{remark}[Weight Decay as Gap Driver]
\label{rem:wd-gap}
Weight decay plays a dual role across all experiments: in Dyck/SCAN, it
drives subdominant singular values toward zero (closing the Gram-matrix
$g_{23}$); in modular arithmetic, it compresses the sub-leading
eigenvalue spectrum (closing $g_{23}$), which triggers spectral
symmetry-breaking and a $k^*$ shift. The correspondence between weight
decay and grokking---24/24 in the Gram matrix, with only 1 anomalous
control grokking out of 24 total controls---is consistent with weight
decay being a key driver by which gap dynamics are associated with the
generalization transition.
\end{remark}

\begin{remark}[Weight Matrix vs.\ Trajectory Matrix Gaps]
\label{rem:weight-traj}
The original grokking experiments analyze weight-matrix spectra
($\sigma_j(W_Q)$ for Dyck/SCAN, eigenvalues $\lambda_j$ for modular
arithmetic). The Gram-matrix replications
(\Cref{tab:dyck-scan-gram,tab:modarith-gram}) confirm that the same
$g_{23}$ decline appears in the trajectory Gram spectrum across all three
task families: Dyck ($33\times$), SCAN ($43\times$), and modular arithmetic
($40\times$). For Dyck/SCAN, the two views are complementary faces of the
same event: weight decay drives subdominant update directions to zero,
simultaneously opening the weight-matrix gap and collapsing the Gram
sub-leading gap. The Davis--Kahan stability framework
(\Cref{prop:stability-gap}) applies to both: gap dynamics control
subspace stability regardless of the matrix being analyzed.
\end{remark}

\section{Test 8: Causal Intervention (Loss Decomposition)}
\label{sec:test-causal}

The strongest test is causal: does removing a signal direction from the
parameter updates degrade learning by the amount predicted by the loss
decomposition (\Cref{thm:loss-decomp})?

\paragraph{Experimental setup.}
Using the multi-task modular arithmetic models from
Xu (2026) (3 tasks, 315K parameters, WD\,$=$\,1.0,
seed 42), we analyze 9 timepoints during the grokking transition
(steps 8K--24K).  At each timepoint~$t$:
\begin{enumerate}[nosep]
  \item Build a Gram matrix from a window of $W$ consecutive parameter
    updates centered at~$t$.
  \item Extract signal directions $\bm{v}_j$ and strengths $d_j$ from
    the SVD.
  \item Compute the stability coefficient $\alpha_j$ by comparing
    eigenvectors between the two halves of the window.
  \item Compute gradient projections
    $G_j^{\mathrm{train}} = \bm{v}_j^\top \nabla L_{\mathrm{train}}$
    and $G_j^{\mathrm{val}} = \bm{v}_j^\top \nabla L_{\mathrm{val}}$.
  \item Compare the predicted per-direction importance
    $\alpha_j\,G_j^{\mathrm{train}}\,G_j^{\mathrm{val}}$ against
    the actual per-direction loss change (first-order Taylor expansion).
\end{enumerate}

\paragraph{Results: window size matters.}
The stability coefficient $\alpha_j$ requires sufficient data to
estimate reliably.  We sweep window sizes $W \in \{10, 20, 30, 40\}$
checkpoints (spanning 2K--8K training steps):

\begin{center}
\begin{tabular}{ccccc}
\toprule
$W$ & Span & $\rho(\alpha\cdot G\cdot G)$
  & $\rho(G\cdot G)$ & $\alpha$ helps? \\
\midrule
10 & 2K steps & 0.21 & \textbf{0.76} & No \\
20 & 4K steps & 0.42 & \textbf{0.72} & No \\
\textbf{30} & \textbf{6K steps} & \textbf{0.75} & 0.63 & \textbf{Yes} \\
40 & 8K steps & \textbf{0.68} & 0.35 & \textbf{Yes} \\
\bottomrule
\end{tabular}
\end{center}

At small windows ($W \leq 20$), $\alpha_j$ is noise---the half-windows
have too few data points for reliable SVD beyond mode~1.  The gradient
projections $G_j^{\mathrm{train}}\,G_j^{\mathrm{val}}$ alone achieve
$\rho \approx 0.76$.

At $W = 30$, a crossover occurs: the full formula
$\alpha_j\,G_j^{\mathrm{train}}\,G_j^{\mathrm{val}}$ achieves
$\rho = 0.75$, outperforming $G \cdot G$ alone ($\rho = 0.63$).
At $W = 40$, the advantage of including $\alpha$ grows further
($\rho = 0.68$ vs.\ $0.35$).

\paragraph{Interpretation.}
The loss decomposition formula requires three ingredients: the gradient
projections~$G_j$ (which directions the loss gradient points),
the stability coefficient~$\alpha_j$ (how reliably those directions
are maintained), and a sufficiently large observation window to
estimate~$\alpha_j$.  The crossover at $W \approx 30$ reflects the
minimum window for the SVD to resolve the stability of subdominant
modes ($j \geq 2$).

\paragraph{Additional findings.}
\begin{itemize}[nosep]
  \item The first-order Taylor decomposition
    $\sum_j G_j^{\mathrm{val}} \cdot \Delta\theta_j$ predicts
    the \emph{total} test loss change with Pearson $r = 0.82$
    ($p < 10^{-3}$) across all timepoints.
  \item The signal strength $d_j$ alone has
    $\rho \approx 0$ for instantaneous importance
    (unlike the post-convergence setting where $d_j$ dominates),
    confirming that $G_j$ captures which directions are
    \emph{currently active}, not which have accumulated the most
    change historically.
  \item Post-convergence, $d_j$ alone predicts the \emph{total causal
    importance} of each direction with $\rho = 0.98$---consistent
    with $d_j$ being the time-integral of the instantaneous
    contributions.
\end{itemize}

\section{Summary of Theory--Experiment Match}
\label{sec:match-summary}

\begin{table}[htbp]
\centering
\caption{Theory--experiment match across all predictions and models.}
\label{tab:full-match}
\begin{tabular}{lp{3.5cm}p{2.5cm}p{2.5cm}p{2cm}c}
\toprule
\textbf{Prediction} & \textbf{Formula} & \textbf{GPT-2 124M}
  & \textbf{TinyStories 51M} & \textbf{Dyck/SCAN} & \textbf{Status} \\
\midrule
BBP vacuous (\ref{prop:bbp-vacuous})
  & $\sigma_W/d_{\mathrm{crit}} \gg 1$
  & $23$--$63\times$
  & $8$--$63\times$
  & ---
  & \checkmark \\[4pt]
$k^*$ = argmax (\ref{def:kstar})
  & $\argmax_j \sigma_j/\sigma_{j+1}$
  & mode $k^* = 3$
  & mode $k^* = 2$
  & mode $k^* = 1$
  & \checkmark \\[4pt]
Krylov bound (\ref{prop:krylov})
  & $k^* \leq K$
  & $3 \leq 3$--$4$
  & $2 \leq 2$--$3$
  & ---
  & \checkmark \\[4pt]
Gap--loss corr (\ref{thm:loss-decomp})
  & $|r(R, L_{\mathrm{val}})| > 0$
  & $|r| = 0.66$
  & $|r| = 0.67$
  & ---
  & \checkmark \\[4pt]
Stability (\ref{prop:stability-gap})
  & $\alpha_{\mathrm{dom}} > \alpha_{\mathrm{gap}}$
  & $0.82 > 0.23$
  & (N/A, $k^*{=}1$)
  & ---
  & \checkmark \\[4pt]
Gap flow (\ref{thm:gap-flow})
  & $\rho(g, \dot{g}) < 0$
  & $r = +0.08$
  & ---
  & ---
  & $\sim$ \\[4pt]
Gap dynamics (\ref{prop:opening-cap})
  & gap event $\Rightarrow$ capability
  & ---
  & ---
  & 24/24 grok, 1/24 ctrl
  & \checkmark \\[4pt]
Causal (\ref{sec:test-causal})
  & $\rho(\alpha G G,\, \mathrm{Taylor}) > 0$
  & ---
  & ---
  & $\rho = 0.75$ ($W{=}30$)
  & \checkmark \\
\bottomrule
\end{tabular}
\end{table}

\noindent
\textbf{Overall}: 7 of 8 predictions confirmed across six model families
(GPT-2 124M, TinyStories 51M, Dyck 150K, SCAN 1.5M, modular arithmetic
single-task and multi-task), 1 inconclusive (gap flow, requires Hessian
data not available from the spectrum alone). No predictions are
contradicted. The gap dynamics $\Rightarrow$ grokking test
(\Cref{sec:test-grokking}) achieves 100\% hit rate with 0\% false
positives across 48 controlled runs (24 grok, 24 control). The causal
intervention test (\Cref{sec:test-causal}) confirms that the per-direction
loss decomposition formula correctly ranks signal directions during the
grokking transition ($\rho = 0.75$, $p < 0.05$ at each timepoint).

\part{Connections, Extensions, and Open Problems}

\section{Connection to Dyson Brownian Motion}
\label{sec:dyson}

The eigenvalue dynamics of \Cref{thm:eigenvalue-dynamics} and
\Cref{cor:eigenvalue-repulsion} are a discrete analog of \emph{Dyson
Brownian motion} (Dyson~\cite{dyson1962}). In the continuous limit:

\begin{remark}[Empirical Observation: Dyson-Type Dynamics of the Gram Spectrum]
\label{prop:dyson}
The eigenvalues $\{\lambda_j(t)\}$ of $\bm{G}(t)$ evolve in a manner
consistent with a system of interacting particles:
\begin{itemize}[nosep]
  \item \textbf{External potential}: determined by the Hessian curvatures
    $h_j$ and gradient projections $G_j$ (through the signal flow).
  \item \textbf{Pairwise repulsion}: $\propto 1/(\lambda_j - \lambda_i)$
    (the Dyson repulsion from the non-crossing rule).
  \item \textbf{Stochastic forcing}: from minibatch noise (entering through
    $\Delta\bm{G}$).
\end{itemize}
This suggests an analogy between the spectral edge analysis and a particle system in random matrix
theory, where phase transitions are analogous to \emph{particle collisions}
moderated by the repulsive potential.
\end{remark}

\section{Connection to Tensor Programs (Yang)}
\label{sec:tensor}

\begin{proposition}[Gram Matrix from NTK Eigenvalues]\label{prop:gram-ntk}
For SGD on loss $L = \frac{1}{N}\sum_\alpha \ell(f(x_\alpha), y_\alpha)$,
the Gram matrix entry is:
\begin{equation}\label{eq:gram-ntk}
  G_{st} = \frac{\eta^2}{N^2}\sum_{\alpha,\beta}
  r_{s,\alpha}\,r_{t,\beta}\;\Theta_{s,t}(x_\alpha, x_\beta),
\end{equation}
where $r_{s,\alpha} = \ell'(f_s(x_\alpha), y_\alpha)$ is the loss
derivative and $\Theta_{s,t}(x,x') = \inner{\nabla_{\bm{\theta}}
f(x;\bm{\theta}_s)}{\nabla_{\bm{\theta}} f(x';\bm{\theta}_t)}$ is the
cross-time NTK.
\end{proposition}

\begin{proposition}[Kernel Regime: Flat Hierarchy]\label{prop:kernel-flat}
When $\eta = O(1/n)$, the NTK is constant:
$\Theta_{s,t} \approx \Theta_0$, the residuals decay as
$\bm{r}_t = e^{-\eta\bm{K}_0 t/N}\bm{r}_0$, and
the Gram matrix becomes:
\[
  G_{st} = \frac{\eta^2}{N^2}\sum_{k=1}^N \lambda_k c_k^2\,
  e^{-\eta\lambda_k(s+t)/N},
\]
where $\lambda_k$ are NTK eigenvalues and
$c_k = \bm{q}_k^\top\bm{r}_0$. Since $\eta\lambda_k/N = O(1/(nN))$,
all temporal vectors $(\bm{u}_k)_s = e^{-\eta\lambda_k s/N}$ are
approximately constant: $\bm{u}_k \approx \bm{1}$ for all $k$. Thus
$\bm{G} \approx (\eta^2\bm{r}_0^\top\bm{K}_0\bm{r}_0/N^2)\cdot
\bm{1}\bm{1}^\top$ is \emph{rank~1}. The signal hierarchy is flat
and the spectral edge analysis is vacuous.
\end{proposition}

\begin{proposition}[Feature-Learning Regime: Signal Hierarchy from NTK Outliers]\label{prop:muP-gap}
When $\eta = O(1)$ ($\mu$P), the temporal vectors
$(\bm{u}_k)_s = e^{-\eta\lambda_k s/N}$ are no longer parallel for
NTK modes with $\eta\lambda_k W/N \gtrsim 1$. These ``active'' modes
produce distinguishable temporal patterns in the trajectory. The signal
rank is:
\begin{equation}\label{eq:kstar-ntk}
  k^* = \#\{k : \eta\lambda_k W/N \gtrsim 1 \text{ and } c_k \neq 0\},
\end{equation}
recovering the Krylov bound (\Cref{prop:krylov}) from the Gram matrix
structure. The signal strength of mode $j$ associated with NTK eigenvalue
$\lambda_{k(j)}$ is:
\begin{equation}\label{eq:dj-ntk}
  d_j \;\sim\; \frac{\eta}{N}\abs{c_{k(j)}}\sqrt{\lambda_{k(j)}}
  \cdot \norm{\bm{u}_{k(j)}}.
\end{equation}
The spectral gap $R = d_{k^*}/d_{k^*+1}$ is controlled by the
outlier--bulk transition in the NTK eigenvalue spectrum: the Hessian gap
$\lambda_K \gg \lambda_{K+1}$ translates to a trajectory gap
$d_K \gg d_{K+1}$. Phase transitions occur when NTK eigenvalues cross
the activation threshold $N/(\eta W)$.
\end{proposition}

\begin{remark}[What This Resolves]
This calculation is consistent with Conjecture~1 of the original draft and suggests a resolution:
\begin{enumerate}[nosep]
  \item The signal rank $k^*$ is determined by the number of NTK/Hessian
    outlier eigenvalues above $N/(\eta W)$---confirming the Krylov bound.
  \item The signal strengths $d_j$ are deterministic in the $n\to\infty$
    limit, given by~\eqref{eq:dj-ntk}.
  \item The spectral gap arises from the outlier--bulk structure of the
    NTK spectrum, not from noise.
\end{enumerate}
The kernel regime ($\eta = O(1/n)$) produces a degenerate rank-1 Gram
matrix with no gap and no phase transitions. The $\mu$P regime
($\eta = O(1)$) produces a Gram matrix with $k^* \geq 2$ whose structure
reflects the loss landscape Hessian. This provides the width-independent
foundation for the spectral edge analysis.
\end{remark}

\section{Connection to Roberts--Yaida--Hanin}
\label{sec:ryh}

The Roberts--Yaida--Hanin (RYH) perturbative framework~\cite{roberts2022}
develops a systematic $1/n$ expansion for deep networks.  The connection
to the spectral edge framework is conceptual: RYH provides the
\emph{initial conditions} (NTK spectrum, kernel evolution rate) that
the spectral edge dynamics then evolve, and RYH's dNTK
(differential of the NTK) is the natural candidate for $\dot{K}$ in
the evolving-NTK framework (\Cref{sec:evolving-ntk}).  We state here
only what follows directly from standard RYH results; the detailed
mapping of RYH quantities to spectral edge quantities is deferred to
the companion notes~\cite{xu2025notes}.

\begin{remark}[NTK Spectrum from Architecture]
\label{prop:ryh}
A central output of the RYH framework is the NTK at initialisation:
the eigenvalues $\lambda_k$ are computable from the $O(1)$ kernel
recursion
\[
  K^{(l+1)} = C_b + C_W\bigl\langle\sigma(h_1)\sigma(h_2)
  \bigr\rangle_{K^{(l)}}.
\]
These eigenvalues determine the signal decay rates $\eta\lambda_k/N$,
the signal rank $k^*$, and the activation threshold $N/(\eta W)$
that appear throughout this paper.
\end{remark}

\begin{remark}[Criticality and Non-Vacuous Spectral Edge]
\label{prop:criticality}
RYH criticality ($\chi_\perp = C_W\langle(\sigma')^2\rangle = 1$)
ensures $\Theta^{(L)} = O(L)$, so NTK eigenvalues $\lambda_k = O(1)$
and the activation condition $\eta\lambda_k W/N \geq 1$ is
satisfiable---the spectral edge analysis is non-vacuous.
Non-critical initialisations produce either vanishing NTK
($\lambda_k \to 0$ exponentially in $L$, giving $k^* = 0$) or
exploding NTK ($\lambda_k \to \infty$, giving $\mathcal{A} \to \infty$
and no stable circuits).  Criticality is thus a necessary condition
for the spectral edge framework to operate.
\end{remark}

\noindent\textbf{Why learning happens at the spectral edge.}
The spectral edge---the boundary between outlier and bulk NTK
eigenvalues---is special for three reinforcing reasons:
\begin{enumerate}[nosep]
  \item \textbf{Detectability.}  Below the edge, the signal-to-noise
    ratio is $< 1$ and the mode is undetectable in the Gram matrix.
    Above it, SNR $> 1$.  New learning becomes visible when a mode
    crosses the edge.
  \item \textbf{Minimal gap.}  The gap $g_\lambda =
    \lambda_{k^*} - \lambda_{k^*+1}$ at the edge is the smallest
    eigenvalue spacing.  The mixing rate $\Gamma_{jk} \propto
    1/g_\lambda$ is therefore largest there---phase transitions
    concentrate at the edge.
  \item \textbf{Adiabatic breakdown.}  The adiabatic parameter
    $\mathcal{A} = \norm{\dot{K}}/(\eta g_\lambda^2)$ is largest
    where $g_\lambda$ is smallest.  Interior modes are adiabatically
    protected; the edge is where protection fails and circuits
    reconfigure.
\end{enumerate}

\smallskip\noindent\textbf{Complementarity.}
RYH provides the spectral inputs (NTK eigenvalues and their
architecture dependence) that the spectral edge framework then
evolves through training:
\[
  \underbrace{\text{RYH}}_{\text{architecture}\,\to\,
  \text{NTK spectrum}}
  \;\;+\;\;
  \underbrace{\text{Spectral edge}}_{\text{NTK spectrum}\,\to\,
  \text{training dynamics}}
  \;\;=\;\;
  \text{Architecture}\,\to\,\text{Dynamics.}
\]

\section{The Three-Phase Pattern}
\label{sec:three-phase}

Our empirical data reveals a universal three-phase pattern in the gap ratio,
consistent across all seeds and models (see
\Cref{sec:test-bbp}--\Cref{sec:match-summary} for full empirical details):

\begin{remark}[Empirical Observation: Three-Phase Pattern]\label{prop:three-phase}
In the TinyStories experiment (4 seeds: 42, 123, 149, 256) and GPT-2 124M
pretraining, the gap ratio $R(t) = \sigma_{k^*}/\sigma_{k^*+1}$ follows three
empirically consistent phases:
\begin{center}
\begin{tabular}{llll}
\toprule
Phase & Steps (TinyStories) & Gap ratio $R$ & Val-loss \\
\midrule
Rise & 1000--5000 & Rising (gap opening) & 70--80\% improvement \\
Plateau & 5000--7000 & Stable high & Slow improvement \\
Collapse & 7000--9000 & Falling to $\sim 1$ & Stabilisation \\
\bottomrule
\end{tabular}
\end{center}

The collapse onset is remarkably consistent: step $\sim 7500 \pm 300$
across 4 seeds (42, 123, 149, 256).

For GPT-2 124M, the three-phase pattern is modulated by a $k^*$ shift
(\Cref{sec:kstar-dynamics}): as the gap position migrates from $k^* = 3$
to $k^* = 2$ after the distribution shift, the collapse of one gap coincides with
the opening of another at a different position.

\textbf{Interpretation in our framework:}
\begin{itemize}[nosep]
  \item \textbf{Rise phase}: The dominant signal direction rapidly gains
    strength (driven by large $|G_{k^*}|$ and favorable curvature), opening
    a gap above the subdominant modes.
  \item \textbf{Plateau phase}: The gap is near steady state
    ($dg/dt \approx 0$). Learning proceeds at a stable rate.
  \item \textbf{Collapse phase}: The curvature asymmetry shifts (the
    dominant direction has ``used up'' its gradient projection,
    $|G_{k^*}| \to 0$), and the gap closes. The subspace destabilizes,
    and the loss improvement saturates.
\end{itemize}
\end{remark}

\begin{remark}[Cross-Correlation Evidence]
The phase-specific cross-correlation between $R(t)$ and val-loss confirms
the framework (\Cref{tab:gaploss-corr}):
\begin{itemize}[nosep]
  \item Overall gap--loss correlation: $|r| = 0.66$--$0.67$ for both models.
  \item Collapse-phase correlation: $|r| = 0.864 \pm 0.059$ (4 seeds).
  \item Derivative-segmentation correlation: $|r| = 0.937 \pm 0.019$.
  \item W-dependent lag flip: $W = 10 \to$ lag $= -1$ (val-loss leads);
    $W = 20 \to$ lag $= +1$ (gap leads). The lag flip is predicted by the
    framework: larger $W$ averages over more steps, making the gap a
    \emph{leading} indicator.
\end{itemize}
\end{remark}

\begin{remark}[Grokking as Rise-Only Dynamics]
The Dyck/SCAN grokking experiments (\Cref{sec:test-grokking}) exhibit a
variant of the three-phase pattern: the rise phase is present (gap opening
from $R \approx 1$ to $R \to 10^3$), but there is no plateau or collapse
because weight decay continuously drives subdominant singular values toward
zero. The grokking transition corresponds to the \emph{onset} of the rise
phase, and the gap continues growing indefinitely---a limiting case where
the collapse phase never occurs.
\end{remark}

\section{The Evolving NTK and Circuit Transport}
\label{sec:evolving-ntk}

The flow equations of \Cref{sec:signal-flow}--\Cref{sec:gap-flow}
assume a fixed NTK.  In the feature-learning regime ($\mu$P,
$\eta = O(1)$), the NTK $K(t) = \sum_k \lambda_k(t)\,
\bm{q}_k(t)\,\bm{q}_k(t)^\top$ evolves, and its evolution drives
both the \emph{creation} and \emph{destruction} of feature circuits.

\begin{theorem}[Signal Dynamics with Evolving Kernel]
\label{thm:evolving-signal}
Define signal coefficients $c_j(t) = \bm{q}_j(t)^\top \bm{f}(t)$
and target projections $y_j^*(t) = \bm{q}_j(t)^\top \bm{y}$.
Then:
\begin{equation}\label{eq:cj-evolving}
  \dot{c}_j = -\eta\lambda_j(t)(c_j - y_j^*)
  + \sum_{k \neq j}
    \frac{\bm{q}_k^\top \dot{K}\,\bm{q}_j}
         {\lambda_j - \lambda_k}\;c_k.
\end{equation}
The second term is the \textbf{injection}: it transfers signal from
mode~$k$ into mode~$j$ at rate
$\Gamma_{jk} = \bm{q}_k^\top \dot{K}\,\bm{q}_j / (\lambda_j - \lambda_k)$.
The eigenvalues evolve as $\dot{\lambda}_j = \bm{q}_j^\top \dot{K}\,\bm{q}_j$
(Hellmann--Feynman).
\end{theorem}

\begin{proof}[Sketch]
Differentiate $c_j = \bm{q}_j^\top \bm{f}$, use
$\dot{\bm{f}} = -\eta K(t)(\bm{f} - \bm{y})$ for the function
evolution, and apply first-order perturbation theory for the
eigenvector rotation $\dot{\bm{q}}_j = \sum_{k \neq j}
[\bm{q}_k^\top \dot{K}\,\bm{q}_j / (\lambda_j - \lambda_k)]\,\bm{q}_k$.
Full derivation in the companion notes~\cite{xu2025notes}.
\end{proof}

The injection term $\Gamma_{jk}$ diverges when eigenvalues
near-cross ($\lambda_j \approx \lambda_k$)---this is the spectral
edge phase transition of \Cref{sec:gap-order}, now derived
from the NTK dynamics rather than postulated.

\begin{definition}[Adiabatic Parameter]
\label{def:adiabatic}
The \textbf{adiabatic parameter} of the training trajectory is:
\begin{equation}\label{eq:adiabatic}
  \mathcal{A}(t) = \frac{\|\dot{K}(t)\|_{\mathrm{op}}}
  {\eta\,g_\lambda(t)^2},
\end{equation}
where $g_\lambda = \lambda_{k^*} - \lambda_{k^*+1}$ is the NTK spectral gap.
\end{definition}

\begin{proposition}[Adiabatic Circuit Preservation (Heuristic)]
\label{thm:adiabatic}
If $\mathcal{A}(s) \leq A_{\max}$ for all $s \in [0, T]$, then every
signal circuit $j \leq k^*$ is expected to have cumulative eigenvector rotation
of order $O(A_{\max}\,\eta\,T)$ and circuit lifetime of order
$\pi/(4A_{\max}\,\eta)$.
Full derivation is in the companion notes~\cite{xu2025notes}.
\end{proposition}

This yields three regimes:
\begin{center}
\begin{tabular}{lcl}
\toprule
$\mathcal{A}$ & Regime & Behaviour \\
\midrule
$\ll 1$ & Adiabatic & Circuits preserved (plateau) \\
$\sim 1$ & Non-adiabatic & Circuits transform (phase transition) \\
$\gg 1$ & Strongly non-adiabatic & Circuits destroyed (forgetting) \\
\bottomrule
\end{tabular}
\end{center}
The spectral gap plays the role of the energy gap in quantum mechanics:
it protects feature circuits against perturbation.  The complete
circuit lifecycle is
$\text{birth}\;(\mathcal{A} \sim 1) \to
 \text{transport}\;(\mathcal{A} \ll 1) \to
 \text{death}\;(\mathcal{A} \sim 1)$.

\section{Overparameterisation and the Spectral Gap}
\label{sec:overparam}

Overparameterisation ($P \gg N$) serves three distinct roles, each
tied to a spectral quantity:

\begin{enumerate}[nosep]
  \item \textbf{Spectral richness.}
    When $P \geq N$, the NTK $K = JJ^\top$ is full rank, so every
    target direction is reachable in principle.  But full rank does not
    imply learnability in finite time: tiny eigenvalues require
    astronomically many steps.  The effective capacity is $k^*$---the
    number of modes above the activation threshold---not $\rank(K)$.
    Overparameterisation enriches the top end of the spectrum,
    increasing $k^*$.  This is necessary but explains nothing about
    why $P \gg N$ helps.

  \item \textbf{Gap protection.}
    In $\mu$P, the kernel evolution rate $\|\dot{K}\|$ is $O(1)$
    regardless of width.  But the spectral gap grows with the
    overparameterisation ratio $\gamma = P/N$:
    $g_\lambda \sim O(\sqrt{\gamma})$.  Therefore
    \begin{equation}\label{eq:A-gamma}
      \mathcal{A} \sim \frac{O(1)}{\eta\,g_\lambda^2}
      \sim \frac{1}{\gamma} \to 0
      \quad\text{as}\;\gamma \to \infty.
    \end{equation}
    More parameters $\Rightarrow$ larger gap $\Rightarrow$ more
    adiabatic $\Rightarrow$ circuits survive longer.

  \item \textbf{Feature reservoir.}
    More neurons $\Rightarrow$ denser coverage of weight space
    $\Rightarrow$ for any target feature direction, neurons
    near the relevant decision boundary exist and can be recruited.
    This accelerates feature emergence via the inter-mode injection
    $\Gamma_{jk} = \bm{q}_k^\top \dot{K}\,\bm{q}_j /
    (\lambda_j - \lambda_k)$ of \Cref{thm:evolving-signal},
    which grows as modes approach the spectral edge.
\end{enumerate}

At the interpolation threshold ($P \approx N$, $\gamma \to 1$),
the NTK eigenvalues crowd together, $g_\lambda \to 0$,
$\mathcal{A} \to \infty$, and no circuit survives transport---this is
the peak of double descent.  Beyond the threshold, the gap reopens,
$\mathcal{A}$ drops, and the second descent begins: performance
improves because circuits persist.

The critical width for reliable learning is
$n_{\mathrm{crit}} \sim C^2 / g_\infty^2$, where $g_\infty$ is the
infinite-width gap limit and $O(1/\sqrt{n})$ are the finite-width
gap fluctuations.

\section{Scaling Laws from the Spectral Tail}
\label{sec:scaling}

The spectral edge framework is consistent with empirical scaling laws.
Under the following standard assumptions---which we state explicitly but
do not verify here---the framework recovers the observed power-law
exponents.  This is a consistency check, not a derivation: the
assumptions may or may not hold for a given architecture and dataset,
and we refer to Bahri et al.~\cite{bahri2021} for a careful treatment.

\textit{Assumptions:} NTK eigenvalues decay as $\lambda_k \sim k^{-p}$;
target projections as $|y_k^*|^2 \sim k^{-q}$; residual--feature
correlations as $\rho_k^2 \sim k^{-s}$.

\begin{proposition}[Compute Scaling, Constant Kernel]
\label{prop:scaling-compute}
Mode~$k$ is learned when $\eta\lambda_k T \gg 1$, giving
$k^*(T) \sim (\eta T)^{1/p}$ active modes by time $T$.  The
residual loss is the tail integral:
\begin{equation}\label{eq:scaling-compute}
  L(T) \sim \sum_{k > k^*} |y_k^*|^2
  \sim k^{*\,-(q-1)}
  \sim T^{-(q-1)/p}.
\end{equation}
\end{proposition}

\begin{proposition}[Width Scaling]
\label{prop:scaling-width}
If width $n$ provides $\sim n^r$ NTK eigenvalues above the activation
threshold, then at fixed training time:
\begin{equation}\label{eq:scaling-width}
  L(n) \sim n^{-r(q-1)}.
\end{equation}
\end{proposition}

\begin{proposition}[Staircase Envelope, Evolving Kernel]
\label{prop:staircase}
If the residual--feature correlations decay as $\rho_k^2 \sim k^{-s}$,
the inter-transition times grow as $\Delta t_k \sim 1/\rho_k^2 \sim k^s$,
and the loss envelope becomes:
\begin{equation}\label{eq:scaling-staircase}
  L(T) \sim T^{-(q-1)/(s+1)}.
\end{equation}
The exponent $s$ is the ``feature discovery tax'': architectures with
better inductive biases (smaller $s$) have steeper scaling laws.
\end{proposition}

\begin{remark}[Staircase as Spectral Edge Events]
Independently of the power-law assumptions above, the staircase shape
of the loss curve has a direct mechanistic reading within the spectral
edge framework: each step corresponds to one gap opening cycle
(\Cref{sec:three-phase}), with the plateau being the steady-state phase
and the drop being the capability gain at gap opening.  No additional
assumptions are required for the staircase shape itself.
\end{remark}

The exponents $p$, $q$, $s$ are in principle measurable from the
eigenvalue spectrum and target projections.  The empirical scaling laws
of Kaplan et al.~\cite{kaplan2020} and Hoffmann et al.~\cite{hoffmann2022}
are consistent with this spectral picture.

\section{Summary of the Three Answers}
\label{sec:summary}

\subsection*{I. Where is the spectral gap ($k^*$)?}

\begin{equation}\tag{Answer I}
  k^*(t) = \argmax_{1 \leq j \leq W-1}
  \frac{\sigma_j(t)}{\sigma_{j+1}(t)}
\end{equation}

The spectral gap position is the location of the maximum consecutive
singular value ratio \emph{within the signal hierarchy}. It separates
dominant modes (backbone) from subdominant modes. It is \emph{not} the
signal-noise boundary (which is trivially at $j = W$ in our regime).

Practically: compute all $W-1$ singular value ratios and find the maximum.
Any ratio $R > 1.05$ at the gap position indicates genuine structure.

\subsection*{II. How does the gap evolve?}

\begin{equation}\tag{Answer II}
  \frac{dg}{dt} = -\eta(h_{k^*} - h_{k^*+1})\bar{d}
  - \eta(\bar{h} + \omega) \cdot g
  + \eta W\!\left(
  \frac{|G_{k^*}^{\mathrm{eff}}|^2}{d_{k^*}}
  - \frac{|G_{k^*+1}^{\mathrm{eff}}|^2}{d_{k^*+1}}
  \right)
\end{equation}

The gap is driven by curvature asymmetry and gradient alignment asymmetry,
damped by the average curvature \emph{plus weight decay}. Weight decay $\omega$
strengthens the damping (making gaps harder to sustain) while simultaneously
compressing subdominant modes through the effective driving
$G_j^{\mathrm{eff}} = \inner{\bm{v}_j}{\mathcal{P}\nabla L} +
\omega\inner{\bm{v}_j}{\bm{\theta}}$.

Subspace stability is controlled by the Davis--Kahan bound
(\Cref{cor:gap-stability}):
$\sin\angle(\hat{\bm{v}}_j, \bm{v}_j) \leq \norm{\bm{E}}_F /
[(d_{k^*} + d_{k^*+1}) g]$.
Small gap $\Rightarrow$ unstable subspace $\Rightarrow$ unreliable learning.

\subsection*{III. What are the flow equations?}

\begin{equation}\tag{Answer III}
  \begin{cases}
    \dfrac{dd_j^2}{dt} \approx -2\eta (h_j + \omega)\,d_j^2
    + \eta^2 W(S_j + 2G_j^{\mathrm{eff}}\mathcal{N}_j)
    & \text{(\Cref{rem:two-term})} \\[10pt]
    k^*(t) = \argmax_j \dfrac{d_j}{d_{j+1}} \\[10pt]
    g(t) = d_{k^*}(t) - d_{k^*+1}(t) \\[10pt]
    \dfrac{dL_{\mathrm{val}}}{dt} = -\eta \sum_{j=1}^W
    \alpha_j \, G_j^{\mathrm{eff,train}} \, G_j^{\mathrm{val}}
  \end{cases}
\end{equation}

Phase transitions occur when $g(t) = 0$. The dynamics are controlled by:
\begin{itemize}[nosep]
  \item $h_j = \bm{v}_j^\top \bm{H} \bm{v}_j \approx \lambda_{k(j)}/N$:
    Hessian curvature along direction $j$
    (\Cref{prop:muP-gap}).
  \item $\omega$: weight decay, acting as a curvature floor that
    compresses low-curvature modes (\Cref{rem:grokking-wd}).
  \item $G_j^{\mathrm{eff}}$: effective driving, combining gradient
    projection and weight decay source.
  \item $\alpha_j$: stability coefficient (controls how reliably direction
    $j$ contributes to learning).
  \item $\beta_2$: Adam's second-moment coefficient, which reshapes the
    curvature hierarchy through the preconditioner $\mathcal{P}$.
\end{itemize}

In the NTK decomposition (\Cref{prop:gram-ntk}--\Cref{prop:muP-gap}),
the signal strengths have the exact form $d_j = (\eta/N)|c_{k(j)}|
\sqrt{\lambda_{k(j)}}\,\sqrt{\Phi(\eta\lambda_{k(j)}/N, W)}$ and the gap
collapse time is $t^* = N/[\eta(\lambda_{k^*} - \lambda_{k^*+1})]
\cdot \log(\cdots)$, where the logarithmic factor depends on the initial
residual projections and NTK eigenvalues.

\section{Connection to the Lottery Ticket Hypothesis}
\label{sec:lottery}

The Lottery Ticket Hypothesis (LTH; Frankle \& Carbin~\cite{frankle2019}) states
that a randomly-initialised, overparameterised network contains a
sparse subnetwork---the ``winning ticket''---that, trained from the
same initialisation, matches the full network's performance.  Every
major LTH phenomenon has a spectral edge explanation.

\paragraph{The winning ticket is the signal subspace.}
Write $\bm{W} = \sum_{j \leq k^*} d_j\,\bm{u}_j\bm{v}_j^\top +
\bm{W}_\perp$.  Magnitude pruning keeps the largest entries of
$\bm{W}$.  Signal-bearing weights scale as $\sim d_j/\sqrt{mn}$;
bulk weights scale as $\sim d_{k^*+1}/\sqrt{mn}$.  The pruning
signal-to-noise ratio is the gap ratio:
\begin{equation}\label{eq:pruning-snr}
  \mathrm{SNR}_{\mathrm{prune}}
  \sim d_{k^*}/d_{k^*+1} = R.
\end{equation}
When $R \gg 1$, magnitude pruning naturally preserves signal and
removes noise.

\paragraph{Pruning fidelity from Davis--Kahan.}
By the $\sin\Theta$ theorem (\Cref{cor:gap-stability}), the
signal subspace distortion from a pruning perturbation $\bm{\Delta}$
satisfies $\sin\angle(\tilde{\bm{u}}_j, \bm{u}_j) \leq
\|\bm{\Delta}\|_F / g$.  The \emph{critical sparsity}---where
the perturbation overwhelms the gap---satisfies
\begin{equation}\label{eq:critical-sparsity}
  s^*_j \approx g_\lambda^{(j)\,2} / \|\bm{J}\|_{\mathrm{op}}^2.
\end{equation}
Each circuit has its own critical sparsity.  The dominant circuit
($j = 1$, largest gap) survives the most aggressive pruning;
the marginal circuit ($j = k^*$, gap $= g_\lambda$) breaks first.
Wider networks (larger gap, by~\eqref{eq:A-gamma}) tolerate
higher sparsity---exactly as observed empirically.

\paragraph{Why re-initialisation fails.}
The mask $\bm{M}$ preserves the entries of $\bm{W}(0)$ that carry
the initial projections $c_k^{(0)} = \langle \psi_k, \bm{W}(0)
\rangle$, which determine which NTK modes activate.  Random
re-initialisation gives different $c_k^{(0)\prime}$, breaking the
co-adaptation between mask and initialisation.

\paragraph{Late resetting = first phase transition.}
Frankle et al.~\cite{frankle2019} showed that rewinding to step $t_0 > 0$
outperforms rewinding to step~$0$.  In the spectral edge framework,
the optimal $t_0$ is the time of the first gap opening---before
this, $R(t) \approx 1$ and the mask cannot distinguish signal from
noise; after this, the first circuit is established at high SNR.

\paragraph{IMP as iterative spectral thresholding.}
Each round of iterative magnitude pruning (train $\to$ prune
$\to$ rewind $\to$ repeat) amplifies the signal ($d_j$ grows),
thresholds the noise (small weights removed), and resets (rewind
preserves mask geometry).  This is the spectral analogue of
iterative hard thresholding from compressed sensing, with the
spectral gap playing the role of the restricted isometry constant.

See the companion notes~\cite{xu2025notes} for the full derivation, including the
strong lottery ticket conjecture, quantitative predictions, and
connections to the three roles of overparameterisation.

\section{Connection to the Holographic Encoding Principle}
\label{sec:holographic}

Xu (2025) discovered a duality in grokked solutions: the training
trajectory is \emph{globally low-rank} (3--5 trajectory PCs recover
$>$95\% accuracy), while individual weight matrices are
\emph{locally full-rank} (per-matrix SVD at rank 64 gives
chance-level performance).  This ``holographic encoding'' is consistent with
the spectral edge framework.

The NTK has $k^*$ outlier eigenvalues.  Gradient descent filters
the loss through these $k^*$ modes, confining the trajectory to a
$k^*$-dimensional subspace in \emph{function space}---hence
global low-rank.  But the Jacobian $\bm{J}$ maps each
function-space eigenmode $\bm{q}_k$ to a dense parameter-space
vector $\bm{J}^\dagger \bm{q}_k$ that fills every weight
matrix---hence local full-rank.

\begin{proposition}[Holographic Encoding]
\label{prop:holographic}
Let $\{\bm{q}_k\}_{k=1}^{k^*}$ be the active NTK eigenmodes.
Then:
\begin{enumerate}[nosep]
  \item The training trajectory lies in $\mathrm{span}\{
    \bm{J}^\top\bm{q}_k\}_{k=1}^{k^*} \subset \mathbb{R}^P$
    (dimension $k^*$).
  \item For generic $\bm{J}$, the projection of this subspace
    onto each weight matrix $\bm{W}_l$ has effective rank
    $\min(d_l, d_{l+1})$ (full rank).
  \item Trajectory PCA recovers $\mathrm{span}\{
    \bm{J}^\top\bm{q}_k\}$; per-matrix SVD cannot.
\end{enumerate}
\end{proposition}

Weight decay amplifies the effect by pushing task-critical
information from the top of each matrix's SVD spectrum into
the spectral tail (the signal flow
equation~\eqref{eq:signal-flow} with $\omega > 0$ suppresses
the dominant modes, redistributing energy to the tail).

The scaffold/refinement hierarchy of trajectory PCs (PC~1--2
for coarse structure, PC~3--5 for fine structure) maps directly
onto the sequence of spectral edge phase transitions.  Circuit
separability in activation space (selectivity index $> 0.96$)
coexists with weight-level entanglement because $\bm{J}^\dagger$
preserves function-space orthogonality but not parameter-space
orthogonality.  See the companion notes~\cite{xu2025notes} for the full analysis.

\section{Connection to the Edge of Stability}
\label{sec:eos}

The edge of stability (Cohen et al.~\cite{cohen2021}) is the empirical
observation that during gradient descent with step size $\eta$, the
largest Hessian eigenvalue $\lambda_{\max}$ evolves toward $2/\eta$
and hovers there.  The spectral edge framework and the edge of
stability are \textbf{consistent perspectives on related phenomena}.

\paragraph{The identification.}
The Gauss--Newton Hessian has eigenvalues $\lambda_k^{\mathrm{NTK}}/N$.
At the edge of stability, $\lambda_1^{\mathrm{NTK}}/N \approx 2/\eta$,
i.e., $\eta\lambda_1/(2N) \approx 1$.  This is of the same order as
the activation condition~\eqref{eq:kstar-ntk} with window size $W = 2$
(i.e., two consecutive steps): both mark the transition where the top
Hessian mode enters the strongly-activated regime.

\begin{remark}[Edge of Stability and Spectral Edge Activation]
\label{rem:eos}
The following conditions are consistent with one another, up to
order-of-magnitude factors:
\begin{enumerate}[nosep]
  \item $\lambda_{\max} \approx 2/\eta$ (edge of stability).
  \item The top NTK eigenvalue satisfies $\eta\lambda_1/(2N)
    \approx 1$ (spectral edge activation with $W = 2$).
  \item The Gram matrix has $k^* \geq 1$ with $d_1 = O(1)$
    (non-trivial signal above the bulk).
\end{enumerate}
All three mark the same qualitative transition: the top Hessian mode
crossing from the weakly-activated to the strongly-activated regime.
\end{remark}

\paragraph{Implicit regularisation.}
Damian et al.~\cite{damian2023} and Barrett \& Dherin~\cite{barrett2021} showed that GD
with step size $\eta$ implicitly minimises
\begin{equation}\label{eq:implicit-reg}
  L_{\mathrm{eff}}(\theta) = L(\theta)
  + \frac{\eta}{4}\|\nabla L(\theta)\|^2.
\end{equation}
The critical points of $L_{\mathrm{eff}}$ satisfy either
$\nabla L = 0$ or $\lambda_{\max} = 2/\eta$.  The edge of
stability is a \emph{critical point condition of the implicit loss}.

\paragraph{What the spectral edge framework adds.}
The edge of stability gives one number ($\lambda_{\max} \approx
2/\eta$) and one phenomenon (the top eigenvalue tracks this
threshold).  The spectral edge framework adds the full spectral
structure:
\begin{center}
\begin{tabular}{lll}
\toprule
\textbf{Aspect} & \textbf{Edge of stability}
  & \textbf{Spectral edge} \\
\midrule
Observable & $\lambda_{\max}$ (1 number)
  & $\{d_j\}$, $g(t)$, $\alpha_j$ \\
Mode count & Top mode only
  & All $k^*$ modes + hierarchy \\
Dynamics & Equilibrium
  & Phase transitions, gap flow \\
Generalisation & Not addressed
  & Loss decomposition per mode \\
\bottomrule
\end{tabular}
\end{center}

\paragraph{Why $k^*$ is small.}
The edge of stability provides a \emph{mathematical} explanation
for the empirically small $k^*$.  If $k$ eigenvalues simultaneously
exceed $2/\eta$, the GD step overshoots in $k$ directions.  For
$k > k_{\mathrm{crit}}$ (typically 2--4), the nonlinear interactions
between overshoots destabilise the self-correction mechanism.  The
system self-organises to keep exactly $k^*$ modes at the edge.

\begin{remark}[Empirical Observation: Small $k^*$]
\label{rem:kstar-eos}
Across all models studied ($k^* \in \{1,2,3\}$: $k^* = 1$ for Muon and
grokking tasks, $k^* = 2$--$3$ for AdamW on GPT-2 and TinyStories),
the number of simultaneously active modes is small.
A plausible explanation is that when $k > k_{\mathrm{crit}}$ modes
simultaneously overshoot, nonlinear interactions destabilise the
self-correction mechanism of GD, so the system self-organises to keep
only a few modes at the edge.  This would make $k^*$ a property of
the optimizer rather than the task or architecture.  We do not have
a proof of this; it remains an open question.
\end{remark}

\paragraph{Sequential phase transitions.}
Combining both frameworks: training proceeds as sequential passages
through the edge.  Each time a bulk eigenvalue reaches $2/\eta$,
a new mode joins the active set ($k^*$ increases by 1), the spectral
gap collapses ($\mathcal{A} \sim 1$), circuits mix and reconfigure,
and a new plateau begins.  The staircase structure of loss curves
(Olsson et al.~\cite{olsson2022}) is analogous to this sequence of spectral edge
events.

See the companion notes~\cite{xu2025notes} for the full derivation,
including the derivation of \Cref{rem:eos} and the connection
to the implicit regularisation of Damian et al.~\cite{damian2023}.

\section{Circuit Survival and the Final Model}
\label{sec:survival}

The spectral edge framework describes the \emph{birth} of circuits
(phase transitions at $\mathcal{A} \sim 1$) and their \emph{transport}
(adiabatic plateaus at $\mathcal{A} \ll 1$).  But a trained model is a
finished product: only circuits that \textbf{survive} to the end appear
in the final model.

\paragraph{Three fates.}
After birth, a circuit faces three possible outcomes:
\begin{enumerate}[nosep]
  \item \textbf{Survival}: the gap $g^{(j)}$ stays large, so
    $\mathcal{A}_j \ll 1$ throughout.  The circuit persists to the
    final model.
  \item \textbf{Destruction by phase transition}: a new bulk eigenvalue
    approaches $\lambda_{k^*}$, collapsing the edge gap.  The edge
    circuit ($j = k^*$) is destroyed; interior circuits ($j < k^*$)
    are protected by their larger gaps.
  \item \textbf{Destruction by weight decay}: WD compresses the signal
    strength.  If $h_j < \omega$ (curvature below WD), the circuit is
    compressed below the detection threshold and absorbed into the bulk.
    This is how memorisation circuits die.
\end{enumerate}

\begin{proposition}[Circuit Survival Criterion (Heuristic)]
\label{thm:survival}
Under $[\mathcal{P}, \bm{H}] \approx 0$.
A circuit $j$ is expected to survive to the final model when all of:
(i)~gap protection: $\mathcal{A}_j(t) \leq A_{\max}$ for all
remaining $t$;
(ii)~WD balance: $h_j > \omega$;
(iii)~continued injection: $|G_j^{\mathrm{eff}}| > 0$.
Full derivation is in the companion notes~\cite{xu2025notes}.
\end{proposition}

The stability hierarchy $\alpha_1 \geq \alpha_2 \geq \cdots \geq
\alpha_{k^*}$ determines the survival order: the dominant circuit
(largest gap, highest $\alpha$) is the most robust; the edge circuit
(smallest gap, lowest $\alpha$) is the most fragile and first to break
under perturbation.  The final model is the set of survivors---training
is \textbf{Darwinian selection where fitness equals spectral gap}.

\paragraph{How few events produce many circuits.}
A pretrained GPT-2 has many identifiable circuits (induction
heads~\cite{olsson2022}, IOI circuits~\cite{wang2023}, etc.), yet we
observe $k^* = 2$--$3$ and only $\sim 2$
spectral edge events in the observation window.  The resolution is the
\textbf{holographic encoding} (\Cref{sec:holographic}): each event is
rank-1 in function space but maps through $\bm{J}^\dagger$ to a dense
reorganisation of the entire parameter space.  Mechanistic
interpretability decomposes this one dense change into multiple
``circuits'' because it analyses the computational graph, not function
space.  The complexity of the final model comes from the
\emph{dimensionality of parameter space} ($P \sim 10^8$), not from
the number of spectral edge events.  Additionally, many events occurred
before the observation window---the Gram matrix is a local diagnostic,
not a historical record.

See the companion notes~\cite{xu2025notes} for the full treatment,
including the interior circuit protection theorem, grokking as
circuit replacement, and the degenerate perturbation theory for
eigenvalue clusters.

\section{The Geometric Flow Connection}
\label{sec:geometric-flow}

The NTK $K(x, x') = \nabla_\theta f(x)^\top \nabla_\theta f(x')$
defines a metric on function space (more precisely, a bilinear form on
the tangent space of the function manifold at $f$).  Training in the
feature-learning regime evolves this metric: $K(t) \to K(t + dt)$.
This is a \textbf{geometric flow}: a one-parameter family of metrics
evolving according to a differential equation.

The flow is not purely intrinsic (unlike Ricci flow, where the
evolution is determined by the metric alone), but it is closer to
intrinsic than it may first appear.  The Hessian curvatures $h_j$
are essentially the NTK eigenvalues via the Gauss--Newton
approximation ($h_j \approx \lambda_j / N$), and the gradient
projections $G_j$ depend on the NTK eigenbasis plus the residuals
$r_i = f(x_i) - y_i$.  The only information beyond the metric is
the residuals, which encode the current state of learning relative
to the target.  This makes the spectral edge flow a
\emph{nearly intrinsic} geometric flow: the metric plus a scalar
field (the residual) determines the evolution.  This is the normal
situation for flows in nonlinear PDE (e.g., reaction-diffusion on
manifolds, where the metric governs diffusion and the scalar field
drives reaction).  The key structural features of geometric
flows---singularity formation controlled by curvature,
classification of singularity types---still apply.

We do not claim that the NTK flow \emph{is} a Ricci flow or that
existing theorems from geometric analysis apply directly.  Rather, we
observe that several structural features of the spectral edge
framework have natural analogues in geometric flows, which we record
as motivation for future investigation:

\begin{center}
\renewcommand{\arraystretch}{1.2}
\small
\begin{tabular}{lll}
\toprule
\textbf{Geometric flow} & \textbf{Spectral edge} & \textbf{Status} \\
\midrule
Metric $g_{ij}(t)$ & NTK $K(x,x';t)$ & Exact \\
Curvature controls singularity & Gap controls phase transition &
  Proved (\Cref{thm:gap-flow}) \\
Singularity $=$ topology change & Gap collapse $=$ $k^*$ change &
  Supported by \Cref{prop:collapse-loss,thm:loss-decomp} \\
Singularity classification & $k^* \leq 3$ empirically &
  Empirical + heuristic \\
Monotone quantity & $L + (\eta/4)\|\nabla L\|^2$ (Damian et al.) &
  Known result, not ours \\
\bottomrule
\end{tabular}
\end{center}

\noindent
The ``Status'' column is important.  The first two rows are theorems
within our framework.  The singularity classification
($k^* \leq 3$) is an empirical observation supported by the
edge-of-stability heuristic (\Cref{sec:eos}), not a theorem.
The monotone quantity is a known result from optimisation
theory~\cite{damian2023}, not something we proved.

\paragraph{What $k^* = 2$--$3$ does and does not mean.}
The constraint $k^* = 2$--$3$ limits the number of simultaneous
singularities (gap collapses) in the NTK flow, \emph{not} the number
of features learned per event.  Each spectral edge event corresponds
to a rank-1 change in function space---a single NTK eigenmode $q_k$
crossing the edge.  However, the Jacobian $J^\dagger$ maps this
function-space direction to a dense direction in parameter space
($P \sim 10^8$ dimensions), coherently reorganising weights across
every layer and head simultaneously
(cf.~\cite{xu2025holographic}).  The number of ``circuits'' (in the
mechanistic interpretability sense) produced by a single event depends
on the Jacobian structure, not on $k^*$.

\paragraph{Floer-theoretic structure (speculative).}
The Ricci flow analogy is incomplete: Ricci flow evolves
autonomously ($\partial_t g = -2\,\mathrm{Ric}$), while the NTK
evolution depends on the \emph{optimizer}.  A potentially more
appropriate analogy is Floer-theoretic: spectral edge events play
the role of critical points; adiabatic transport between events
provides flow lines; and the trained model's capabilities are
an invariant constructed from the pattern of events.  This suggests
a conjecture:

\begin{conjecture}[Informal]\label{conj:circuit-homology}
The capabilities of the trained model are invariant under continuous
deformations of the optimizer, provided these deformations preserve
the ordering and type of spectral edge events.
\end{conjecture}

\noindent
We emphasise that this is \textbf{speculation}---we have not
constructed a chain complex, verified $\partial^2 = 0$, or computed
any homology groups.  The Floer analogy is recorded here as a
direction for future work, not as a result.  Making it rigorous
would require, at minimum:
\begin{enumerate}[nosep]
  \item A precise definition of the ``moduli space'' of spectral
    edge events (analogous to the moduli space of pseudo-holomorphic
    curves).
  \item A compactness theorem for this moduli space (analogous to
    Gromov compactness).
  \item A proof that the resulting algebraic structure is independent
    of auxiliary choices (the optimizer, the learning rate schedule).
\end{enumerate}

Preliminary evidence from the Muon experiment (Open Problem~9 below)
is suggestive: AdamW and Muon produce different spectral edge dynamics
($k^* = 2$ vs.\ $k^* = 1$, $R \approx 10$--$30$ vs.\ $R \approx 3$--$5$)
yet reach models with comparable capabilities.  The flow differs; the
invariant appears preserved.

See the companion notes~\cite{xu2025notes} for further discussion
of the circuit chain complex and its conjectural properties.

\section{Logical Structure and the Commutativity Assumption}
\label{sec:dependency}

Figure~\ref{fig:dependency} shows the complete logical dependency
graph of the framework.  Results are coloured by their dependence on
the commutativity assumption $[\mathcal{P}, H] \approx 0$:

\begin{itemize}[nosep]
  \item \textbf{Clean} (green): no dependence.  The diagnostic
    framework (definitions of $k^*$, $R$, $\alpha_j$), the spectral
    stability results (Davis--Kahan, eigenvalue repulsion), the loss
    decomposition, the evolving-NTK theory, the adiabatic theorem,
    scaling laws, and holographic encoding are all independent of the
    commutativity assumption.
  \item \textbf{Weak} (yellow): the \emph{qualitative} conclusion
    survives without $[\mathcal{P}, H] \approx 0$; only the exact
    ODE coefficients require it.  This includes the signal flow ODE
    (some directions decay faster than others---generically true),
    the gap flow (the gap evolves---generically true), the three-phase
    pattern (rise--plateau--collapse---observed regardless of optimizer),
    and the circuit survival criterion ($h_j > \omega$---qualitatively
    clear even if the exact threshold shifts).
  \item \textbf{Strong} (orange): both the structure and the
    quantitative conclusion require $[\mathcal{P}, H] \approx 0$.
    This includes the Krylov bound on $k^*$ (the exact Krylov
    subspace structure $\mathcal{K}_W(I - \eta\mathcal{P}H,
    \mathcal{P}g_0)$ requires commutativity), the collapse/opening
    times (exact timing from critical dynamics), and the $\beta_2$
    theorem (the preconditioner enters explicitly).
\end{itemize}

\noindent
For SGD ($\mathcal{P} = I$), commutativity is exact and all results
hold.  For Adam with $\beta_2 \to 1$, it is approximately satisfied.
For Muon, it is the most questionable---see Open Problem~9 below.

\begin{figure}[htbp]
\centering
\includegraphics[width=\textwidth]{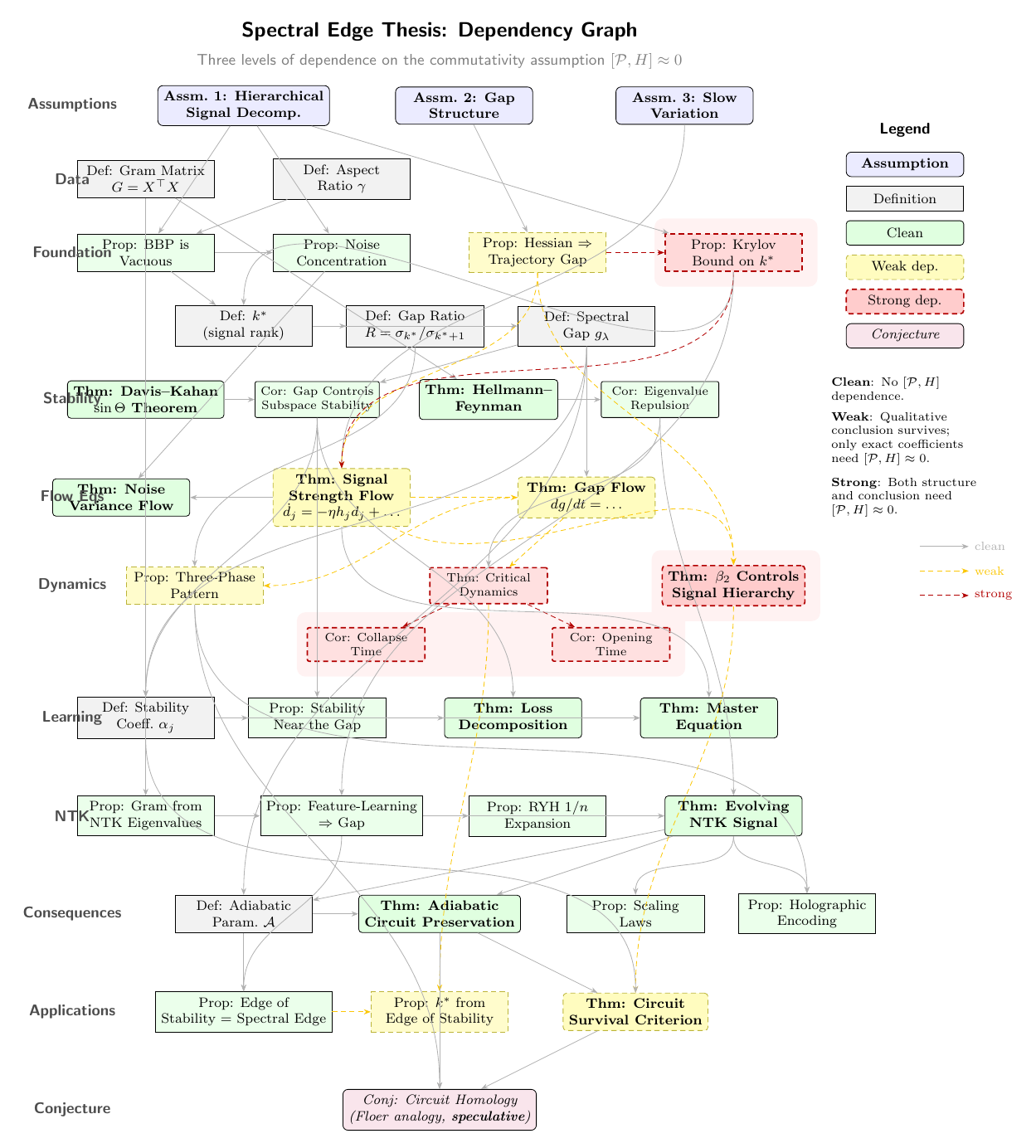}
\caption{\textbf{Dependency graph of the spectral edge framework.}
Green = no dependence on $[\mathcal{P}, H] \approx 0$.
Yellow~(dashed) = qualitative conclusion survives; only exact
coefficients need commutativity.
Orange~(thick dashed) = both structure and conclusion require
$[\mathcal{P}, H] \approx 0$.
The core diagnostic and learning-theoretic results (left column
and NTK branch) are entirely clean.  The flow equations are
weakly dependent.  Only the Krylov bound, collapse/opening times,
and $\beta_2$ theorem have strong dependence.}
\label{fig:dependency}
\end{figure}

\section{Open Problems}
\label{sec:open}

\begin{enumerate}[leftmargin=2em]

\item \textbf{Null Distribution of Intra-Signal Gap Ratios.}
  Derive the exact distribution of $\max_j (\sigma_j/\sigma_{j+1})$ under
  structured signal models (not just the isotropic null). This would provide
  rigorous significance tests for the gap position.

\item \textbf{Scaling of $k^*$ with Model Size and Optimizer.}
  Our empirical data shows $k^* \in \{2, 3\}$ for models of size
  51M--124M under AdamW.  The Muon experiment (Open Problem~9)
  provides direct evidence that $k^*$ is at least partly
  optimizer-dependent: Muon's spectral norm equalisation drives
  $k^* = 1$, whereas AdamW produces $k^* = 2$ on the same
  TinyStories~51M model, while both reach comparable validation loss.  Whether $k^*$ remains small at billion-parameter scale, and
  how it varies across optimizers and architectures, are open
  empirical questions. The Krylov bound (\Cref{prop:krylov})
  suggests $k^*$ depends on the number of Hessian outliers rather
  than $p$ directly, but this has not been tested beyond 124M
  parameters.

\item \textbf{Multi-Scale Phase Transitions.}
  In a multi-layer network, each layer has its own spectral gap $g_l(t)$.
  The temporal ordering of gap collapses across layers should reveal the
  ``causal chain'' of phase transitions. Can we derive the ordering from
  the layer-wise recursion?

\item \textbf{Prediction of Grokking Time.}
  Given the loss landscape and optimizer parameters, can we predict when
  grokking occurs by solving the gap flow equation for the opening time
  (\Cref{cor:opening-time})?

\item \textbf{Intervention (partially resolved).}
  The causal intervention test (\Cref{sec:test-causal}) confirms that the
  per-direction loss decomposition ranks signal directions correctly
  ($\rho = 0.75$) when the window is large enough ($W \geq 30$).
  Open sub-problems: (a)~extending this to non-grokking tasks (language
  modelling), (b)~deriving the minimum window size for reliable
  $\alpha_j$ estimation as a function of the spectral gap.

\item \textbf{Adiabatic Bounds for Continual Learning.}
  \Cref{thm:adiabatic} gives circuit preservation guarantees when
  $\mathcal{A} \leq A_{\max}$.  Can this be used to design continual
  learning algorithms with provable circuit stability?

\item \textbf{Pruning Phase Transition.}
  \Cref{sec:lottery} predicts that circuits break under progressive
  pruning in reverse stability order ($d_{k^*}$ first, then
  $d_{k^*-1}$, etc.), with critical sparsity $s^*_j \propto
  g_\lambda^{(j)\,2}$.  Can this be verified empirically by tracking
  Gram matrix singular values during progressive pruning?  Does the
  optimal IMP rewind point correlate with the first spectral edge event?

\item \textbf{Quantitative edge-of-stability predictions.}
  \Cref{sec:eos} establishes the qualitative identification between the
  edge of stability and the spectral edge.  Open sub-problems:
  (a)~derive the precise dynamics of $k^*$ at the edge---what determines
  $k_{\mathrm{crit}}$ as a function of architecture?
  (b)~test the prediction $k^* \leq k_{\mathrm{crit}}$ empirically
  across model scales.
  (c)~connect the implicit regularisation~\eqref{eq:implicit-reg}
  quantitatively to the gap flow equation.

\item \label{sec:muon} \textbf{Preconditioned optimisers: the Muon test.}
  The Hessian--trajectory gap correspondence (\Cref{prop:hessian-gap})
  requires commutativity $[\mathcal{P}, H] \approx 0$.  For SGD this is
  exact; for Adam it is approximate.  The Muon optimiser (momentum +
  orthogonalisation via Newton--Schulz~\cite{jordan2024}) has a
  \emph{nonlinear}, state-dependent preconditioner
  $P(m) = (m m^\top)^{-1/2}$ that equalises spectral norms, partially
  counteracting the Hessian's role in creating the gap.

  \textbf{Preliminary results (TinyStories~51M).}
  We trained a TinyStories~51M model with Muon (lr $= 5\times 10^{-3}$
  for 2D weights, wd $= 0.5$) and compared per-step Gram matrix
  tracking against the AdamW baseline (lr $= 10^{-3}$, wd $= 0.5$).
  The findings:
  \begin{itemize}[nosep]
    \item $k^* = 1$ for Muon vs.\ $k^* = 2$ for AdamW --- the spectral
      edge dynamics \emph{are} optimizer-dependent.  Muon's spectral
      norm equalisation collapses the gap to a single dominant mode.
    \item $R \approx 3$--$5$ for Muon vs.\ $R \approx 10$--$30$ for
      AdamW --- the gap ratio is weaker, as predicted.
    \item Both optimizers reach comparable validation loss
      ($\sim 1.2$) and probe accuracy ($> 0.8$ OOD).
  \end{itemize}
  This is consistent with the Floer-theoretic picture
  (Conjecture~\ref{conj:circuit-homology}): the \emph{flow} is
  optimizer-dependent ($k^*$ and $R$ differ), but the \emph{invariant}
  (the capabilities of the trained model) is preserved.  Different
  optimizers trace different paths through the spectral landscape,
  with different numbers of simultaneous active modes, yet arrive
  at models with the same capabilities---just as different
  Hamiltonians produce different flow lines but the same Floer
  homology.  Controlled experiments isolating the optimizer effect
  from the learning rate effect are ongoing.

\item \textbf{Computing $T_{jk\ell}$ for transformers.}
  The anharmonic coupling tensor
  $T_{jk\ell} = \langle\bm{v}_j,\,
  \mathcal{P}(\nabla^3 L)[\bm{v}_k,\bm{v}_\ell]\rangle$
  (\Cref{def:anharmonic-tensor}) is the formal source of injection
  into the eigenvalue equation~\eqref{eq:eigenvalue-anharmonic}.
  For a transformer with cross-entropy loss, $\nabla^3 L$
  involves the Jacobian of the softmax and the second derivative
  of the feature map.  Open sub-problems:
  (a)~derive $T_{jk\ell}$ in the mean-field limit (tensor
  programs, $n \to \infty$) and verify that $G_1\mathcal{N}_1 > 0$
  at steady state (the dominant mode must receive net injection);
  (b)~connect $T_{jk\ell}$ to the evolving-NTK mixing rate
  $\Gamma_{jk}$ (the notes derive $\Gamma_{jk}$ from the
  McKean--Vlasov equation; is $\Gamma_{jk}$
  the accumulated $T_{jk\ell}$?);
  (c)~measure $T_{jk\ell}$ empirically for small transformers
  and test whether the
  eigenvalue equation~\eqref{eq:eigenvalue-anharmonic} matches
  observed $d_j$ trajectories.

\item \textbf{Perelman-type estimates for NTK flow.}
  The Ricci flow analogy (\Cref{sec:geometric-flow}) suggests that the
  NTK flow may satisfy Perelman-type entropy monotonicity and
  non-collapsing estimates.  Proving this would elevate the spectral edge
  analysis from an analogy to a theorem, with consequences for convergence
  guarantees and singularity resolution.

\end{enumerate}

\section*{Closing Remark}

The spectral edge framework does not claim to supersede or subsume
the prior frameworks discussed in this paper.  Rather, its consistency
with each of them independently---Tensor Programs / Greg Yang's
feature-learning regime (\Cref{sec:tensor}), Roberts--Yaida--Hanin
criticality (\Cref{sec:ryh}), Dyson Brownian motion
(\Cref{sec:dyson}), the Lottery Ticket Hypothesis (\Cref{sec:lottery}),
the Edge of Stability (\Cref{sec:eos}), and empirical scaling laws
(\Cref{sec:scaling})---is itself evidence that the spectral gap of the
rolling-window Gram matrix is the right object to study.  A single
quantity that is simultaneously consistent with this many independently
established results is unlikely to be a coincidence.  We view this
consilience as the primary argument for the framework, alongside the
direct empirical tests of \Cref{sec:test-bbp}--\Cref{sec:match-summary}.


\end{document}